%% file: fairGM-main.tex
\documentclass[ejs]{imsart}

\RequirePackage{amsthm,amsmath,amsfonts,amssymb}
\RequirePackage[numbers]{natbib}
\RequirePackage[colorlinks,citecolor=blue,urlcolor=blue]{hyperref}
\RequirePackage{graphicx}
\usepackage{booktabs}
\usepackage[normalem]{ulem}

\usepackage{bbm}
\usepackage{enumerate}
\usepackage{multirow}
\usepackage{multicol}
\usepackage{array,booktabs}
\usepackage{algorithm}
\usepackage{algpseudocode}
\usepackage{enumitem}
\usepackage{epstopdf}
\usepackage{subcaption}
\usepackage{adjustbox}
\usepackage{fancyhdr}
\usepackage{geometry}

\newcommand{\mc}{\mathcal}
\newcommand{\m}[1]{{\bf{#1}}}
\newcommand{\wh}[1]{{\widehat{#1}}}

\newcommand{\ti}[1]{{\tilde{#1}}}
\renewcommand{\mc}[1]{\ensuremath{\mathcal{#1}}} 
\newcommand{\g}[1]{\mbox{\boldmath $#1$}}
\newcommand{\mb}[1]{{\mathbb{#1}}}
\DeclareMathOperator*{\argmin}{argmin}
\DeclareMathOperator*{\minimize}{minimize}
\DeclareMathOperator{\trace}{trace}
\DeclareMathOperator{\sign}{sign}

\DeclareMathOperator{\Var}{Var}
\startlocaldefs
\theoremstyle{plain}
\newtheorem{thm}{Theorem}

\newtheorem{exm}[thm]{Example}

\newtheorem{rem}[thm]{Remark}

\newtheorem{lem}[thm]{Lemma}

\theoremstyle{remark}

\endlocaldefs

\begin{document}
\begin{frontmatter}
\title{Fair Community Detection and Structure Learning in Heterogeneous Graphical Models}
\runtitle{Fair Graphical Models }

\begin{aug}
\author[A]{\fnms{Davoud}~\snm{Ataee Tarzanagh}\ead[label=e1]{tarzanagh@gmail.com,girasole@umich.edu,hero@eecs.umich.edu}\orcid{https://orcid.org/0000-0003-1267-3889}},
\author[A]{\fnms{Laura}~\snm{Balzano}\orcid{https://orcid.org/0000-0003-2914-123X}}
\and
\\
\author[A]{\fnms{Alfred}~\snm{O. Hero} \orcid{https://orcid.org/0000-0002-2531-9670}}
\address[A]{Department of Electrical Engineering and Computer Science, University of Michigan\printead[presep={,\ }]{e1}}
 \runauthor{D. Ataee Tarzanagh et al.}
\end{aug}

\begin{abstract}
Inference of community structure in probabilistic graphical models may not be consistent with fairness constraints when nodes have demographic attributes. Certain demographics may be over-represented in some detected communities and under-represented in others.  This paper defines a novel $\ell_1$-regularized pseudo-likelihood approach for fair graphical model selection. In particular, we assume there is some community or clustering structure in the true underlying graph, and we seek to learn a sparse undirected graph and its communities from the data such that demographic groups are fairly represented within the communities. In the case when the graph is known a priori, we provide a convex semidefinite programming approach for fair community detection. We establish the statistical consistency of the proposed method for both a Gaussian graphical model and an Ising model for, respectively, continuous and binary data, proving that our method can recover the graphs and their fair communities with high probability. 
\end{abstract}



\end{frontmatter}
\tableofcontents
\input{sections/sec_intro}
\input{sections/sec_fairGM}
\input{sections/sec_fairGMs}

\input{sections/sec_exp_simu}
\input{sections/sec_exp_real}

\input{sections/sec_conc}
\section*{Acknowledgment}
Davoud Ataee Tarzanagh and Laura Balzano were supported by NSF BIGDATA award \#1838179, ARO YIP award W911NF1910027, and NSF CAREER award CCF-1845076. Alfred O. Hero was supported by US Army Research Office grants \#W911NF-15-1-0479 and \#W911NF-19-1-0269, in addition to NSF award CCF-2246213.

\input{sections/sec_appendix}

\bibliographystyle{plain}
\bibliography{ref_fairgl}

\end{document}

%% file: sections/sec_intro.tex
\section{Introduction}\label{sec:intro}

Probabilistic graphical models have been applied in a wide range of machine learning problems to infer dependency relationships among random variables. Examples include gene expression~\cite{peng2009partial,wang2009learning}, social interaction networks~\cite{Tan14,tarzanagh2018estimation}, computer vision~\cite{hassner1981use,laferte2000discrete,manning1999foundations}, and recommender systems~\cite{kouki2015hyper,wang2015collaborative}.  Since in most applications the number of model parameters to be estimated far exceeds the available sample size, it is necessary to impose structure, such as sparsity or community structure, on the estimated parameters to make the problem well-posed. With the increasing application of structured graphical models and community detection algorithms in human-centric contexts \cite{tan2013data, song2011fast, glassman2014feature, burke2011recommender, pham2011clustering, das2014clustering}, there is a growing concern that, if left unchecked, they can lead to discriminatory outcomes for protected groups. For instance, the proportion of a minority group assigned to some community can be far from its underlying proportion, even if detection algorithms do not take the minority sensitive attribute into account in decision making~\cite{chierichetti2017fair}. Such an outcome may, in turn, lead to unfair treatment of minority groups.   

For example, in precision medicine, patient-patient similarity networks over a biomarker feature space can be used to cluster a cohort of patients and support treatment decisions on particular clusters \cite{parimbelli2018patient, lafit2019partial}. If the clusters learned by the algorithm are demographically imbalanced, this treatment assignment may unfairly exclude under-represented groups from effective treatments. In social networks, ensuring demographic parity can help prevent echo chambers, fostering diverse viewpoints within communities and promoting a healthier discourse~\cite{cinelli2021echo, aslan2021demographic}. In educational settings, fair clustering can ensure that study groups are diverse, providing equal opportunities for interaction and learning across different demographic groups~\cite{brusilovsky2016educational, holstein2019fairness}. These examples highlight the importance of integrating fairness considerations into community detection algorithms to mitigate biases.

To the best of our knowledge, the estimation of \textit{fair} structured graphical models has not previously been addressed. However, there is a vast body of literature on learning structured probabilistic graphical models. Typical approaches to impose structure in graphical models, such as $\ell_1$-regularization, encourage sparsity structure that is uniform throughout the network and may therefore not be the most suitable choice for many real world applications where data have clusters or communities, i.e., groups of graph nodes with similar connectivity patterns or stronger connections within the group than to the rest of the network. Graphical models with these properties are called heterogeneous.

It is known that if the goal is structured heterogeneous graph learning, structure or community inference and graph weight estimation should be done jointly. In fact, performing structure inference before weight estimation results in a sub-optimal procedure~\cite{marlin2009sparse}. To overcome this issue, some of the initial work focused on either inferring connectivity information or performing graph estimation in case the connectivity or community information is known~\textit{a priori} \cite{Danaher13,guo2011joint,gan2019bayesian,ma2016joint,lee2015joint}. Recent approaches consider the two tasks jointly and simutaneously perform connectivity inference, community detection, and weight estimation in structured graphical models ~\cite{kumar2020unified,hosseini2016learning,hao2018simultaneous,  tarzanagh2018estimation,kumar2019structured, gheche2020multilayer,pircalabelu2020community,cardoso2020algorithms,eisenach2020high}.  Our proposed fair graphical modeling framework adopts this joint inference, detection and estimation approach. 

In this paper, we develop a provably convergent penalized pseudo-likelihood method to induce fairness into clustered probabilistic graphical models. Our contributions are as follows:
\begin{itemize}
    \item \textbf{Fair Structure Learning in Graphical Models}: We formulate a novel version of probabilistic graphical modeling that incorporates fairness considerations. We assume an underlying community structure in our graph and aim to learn an undirected graph from the data such that demographic groups are fairly represented within the communities.
    \item \textbf{Fair Convex Community Detection}: We provide a semidefinite programming for detecting fair communities within a known graph. This approach retains the computational complexity of convex community detection methods \cite{amini2018semidefinite,cai2015robust,li2021convex} while enhancing fairness in community structure.
    \item \textbf{Unified Consistency Estimation Analysis}: Our rigorous analysis shows high-probability recovery of fair communities within an unknown graph. We show that community estimators are asymptotically consistent in high-dimensional settings for Gaussian and Ising graphical models, assuming standard regularity conditions.
    \item We present numerical results from both synthetic and real-world datasets, highlighting the importance of proportional clustering. Comparisons of our method with traditional graphical models show our algorithms' superior efficacy in estimating graphs and their communities, thus attaining higher fairness levels.
\end{itemize}

The remainder of the paper is organized as follows:
Section \ref{sec:genfram} gives a general framework for fair structure learning in graphs. Section~\ref{sec:fairgraphs}
gives a detailed statement of the proposed fair graphical models for continuous and binary datasets. In  Sections~\ref{sec:synth} and \ref{Sec:realdata}, we illustrate the proposed framework on a number of synthetic and real data sets, respectively. Section~\ref{sec:conc} provides some concluding remarks.

\paragraph{Notation.} 
For a set $\mc{S}$, $|\mc{S}|$ denotes its cardinality, and $\mc{S}^c$ its complement. The real and nonnegative real fields are denoted by $\mb{R}$ and $\mb{R}_{+}$, respectively. Lower- and upper-case bold letters, such as $\m{x}$ and $\m{X}$, represent vectors and matrices, respectively, with $x_i$ and $x_{ij}$ denoting their elements. If all coordinates of a vector $\m{x}$ are nonnegative, we write $\m{x} \geq 0$; similarly, $\m{X} \geq 0$ and $\m{X} > 0$ denote elementwise nonnegativity and positivity for matrices. For a symmetric matrix $\m{X} \in \mb{R}^{n\times n}$, we write $\m{X} \succ 0$ if $\m{X}$ is positive definite, and $\m{X} \succeq 0$ if it is positive semidefinite. 
We denote by $\m{I}_p$, $\m{J}_p$, and $\m{0}_p$ the $p\times p$ identity, all-ones, and all-zeros matrices, respectively. The notation $\m{X}^k$ represents the $k$-th matrix power of $\m{X}$. 
For any matrix $\m{X}$, $\Lambda_{i}(\m{X})$, $\Lambda_{\max}(\m{X})$, and $\Lambda_{\min}(\m{X})$ denote its $i$-th, maximum, and minimum singular values, respectively.  
Matrix and vector norms are defined as follows:
\begin{align*}
 \|\m{X}\|_{\infty} := \max_{ij}|x_{ij}|, \quad 
\|\m{X}\|_{1} := \sum_{ij}|x_{ij}|, \quad 
\|\m{X}\|_{2} := \Lambda_{\max}(\m{X}), \quad
\|\m{X}\|_{\text{F}} := \sqrt{\sum_{ij}|x_{ij}|^2}.   
\end{align*}
For a vector $\m{x}$, we use $\|\m{x}\|:= (\sum_i x_i^2)^{1/2}$ for its Euclidean (or $\ell_2$) norm.

%% file: sections/sec_fairGM.tex
\section{Fair Structure Learning in Graphical Models}\label{sec:genfram}

We introduce a fair graph learning method that simultaneously accounts for fair community detection and estimation of heterogeneous graphical models.  

Let $\m{Y}$ be an $n\times p$ matrix, with columns $\m{y}_{1},\ldots,\m{y}_{p}$. Each column \(\mathbf{y}_i\) (\(i \in \{1, 2, \ldots, p\}\)) is a vector of length \(n\), representing the observations for the $i$-th variable across all \(n\) samples. We associate to each column in $\m{Y}$ a node in a graph $\mc{G}=( \mc{V}, \mc{E})$, where $\mc{V}= \{1,2,\ldots,p \}$ is the vertex set and $\mc{E}\in \mc{V}\times\mc{V}$ is the edge set.  We consider a simple undirected graph, without self-loops, and whose edge set contains only distinct pairs. Graphs are conveniently represented by a $p \times p $ matrix, denoted by $\g{\Theta}$, whose nonzero entries correspond to edges in the graph. The precise definition of this usually depends on modeling assumptions,  properties of the desired graph, and application domain.

In order to obtain a sparse and interpretable graph estimate, many authors have considered the problem
\begin{align}~\label{eqn:obj:graph}
\begin{array}{ll}
\underset{\g{\Theta}}{\text{minimize}} &
\begin{array}{c}
\hspace{.2cm}  L (\g{\Theta};\m{Y}) + \rho_1 \|\g{\Theta}\|_{1,\textnormal{off}} 
\end{array}\\
\text{subj. to} & \begin{array}[t]{l}
\hspace{.2cm} \g{\Theta} \in \mc{M}.
\end{array}
\end{array}
\end{align}
Here, $L$ is a loss function; $\rho_1 \|\g{\Theta}\|_{1,\textnormal{off}}$ is the $\ell_1$-norm regularization applied to off-diagonal elements of $\g{\Theta}$ with parameter $\rho_1> 0$; and $\mc{M}$ is a convex constraint subset of $\mb{R}^{p \times p}$. For instance, in the case of a Gaussian graphical model,  we could take $L( \g{\Theta};\m{Y}) =-\log \det (\g{\Theta}) + \trace(\m{S}\g{\Theta})$,  where $\m{S}=n^{-1} \sum_{i=1}^n \m{y}_i \m{y}_i^\top$ and $\mc{M}$ is the set of $p \times p$ positive definite matrices. The solution to \eqref{eqn:obj:graph} can then be interpreted as a sparse estimate of the inverse covariance matrix \citep{banerjee2008model,Friedman07}.  Throughout, we assume that $L (\g{\Theta};\m{Y})$ and $\mc{M}$ are convex function and set, respectively.

\subsection{Model Framework}\label{sec:model}

We build our fair graph learning framework using \eqref{eqn:obj:graph} as a starting point.  
Let $\mc{V}$ denote the set of nodes. Suppose there exist $K$ disjoint communities of nodes denoted by $\mc{V}=\mc{C}_1 \cup \cdots \cup \mc{C}_K$ where $\mc{C}_k$ is the subset of nodes from $\mc{G}$ that belongs to the $k$--th community. 
An illustrative example of these community structures and demographic groups is shown in Figure~\ref{fig:SBM-example}.
For each candidate partition of $n$ nodes into $K$ communities, we associate it with a \emph{partition matrix} $\m{Q} \in [0, 1]^{p \times p}$, such that $q_{ij}=1/|\mc{C}_k|$ if and only if nodes $i$ and $j$ are assigned to the $k$th community. Let $\mc{Q}_{pK}$ be the set of all such partition matrices, and $ \bar{\m{Q}}$ the true partition matrix associated with the ground-truth clusters $\{\bar{\mc{C}}_k\}_{k=1}^K$. 

\begin{figure}[t]
    \begin{center} 
    \includegraphics[width=0.46\textwidth]{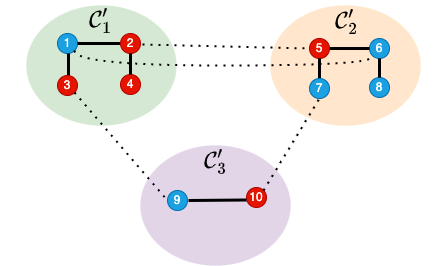}
    \hspace{.5cm}
    \includegraphics[width=0.46\textwidth]{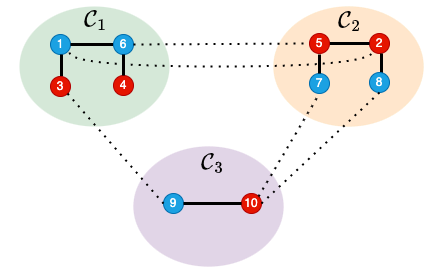} 
    \end{center} 
\caption{Examples of graphical models with $p = 10$ nodes and demographic groups $\mc{V}=\mc{D}_1\cup\mc{D}_2$, where $\mc{D}_1$ is the set of {\color{red}red} nodes and $\mc{D}_2$ is the set of {\color{cyan}cyan} nodes. The left and right figures show the graphical model renderings of the sparse precision matrices associated with $\m{Q}'$ and $\m{Q}$ (fair membership matrix) defined in Example~\ref{exam:fair:sbm}, respectively. In the fair graphical model (right) with disjoint communities $\mc{C}_1$, $\mc{C}_2$, and $\mc{C}_3$, we have $|\mc{D}_h\cap\mc{C}_k|/|\mc{C}_k| = 1/2 = |\mc{D}_h|/p$ for all $k\in\{1,2,3\}$ and $h\in\{1,2\}$, as a result, $\m{Q}$ satisfies demographic parity in every community—equivalently, $\m{R}(\m{I}-\m{J}_p/p)\m{Q}=\m{0}$. The right panel (membership with $\m{Q}$) therefore represents a fairer model than the left panel (membership with $\m{Q}'$), since $\m{Q}'$ violates this balance in at least one community, i.e., $\m{R}(\m{I}-\m{J}_p/p)\m{Q}'\neq \m{0}$. The dotted edges represent lower-probability connections across communities ($a=0.1$), while solid edges are higher-probability within-community connections ($5a$).}
\label{fig:SBM-example}
\end{figure}

Assume the set of nodes contains \( H \) demographic groups, such that \( \mathcal{V} = \mathcal{D}_1 \cup \cdots \cup \mathcal{D}_H \), with the possibility of overlapping group memberships. Let \( \mathbf{R} \in \{0, 1\}^{p \times p} \) be defined so that \( r_{ij} = 1 \) if and only if nodes \( i \) and \( j \) belong to the same demographic group, with the convention that \( r_{ii} = 1 \) for all \( i \). For example, in a graph with nodes \( \mathcal{V} = \{v_1, v_2, v_3\} \) and two demographic groups \( \mathcal{D}_1 = \{v_1, v_2\} \) and \( \mathcal{D}_2 = \{v_2, v_3\} \), the matrix \( \mathbf{R} \) is given by 
$
\mathbf{R} = [1, 1, 0 ;\ 1, 1, 1 ;\ 0, 1, 1] 
$
where \( r_{ij} = 1 \) indicates that nodes \( i \) and \( j \) share at least one demographic group. Here, node \( v_2 \) belongs to both groups, demonstrating an overlap in demographic memberships. 

Chierichetti et al. \cite{chierichetti2017fair} proposed a fair clustering model that requires the representation in each cluster to preserve the global fraction of each demographic group~$\mc{D}_h$, i.e., 
\begin{equation}\label{eq:fairnotion}
 \frac{|\mc{D}_h\cap \mc{C}_k|}{|\mc{C}_k|}= \frac{|\mc{D}_h|}{p}~~ \textnormal{for all}~~k \in [K]. 
\end{equation} 
Let 
\begin{equation}\label{eqn:def:AB}
\m{A}_1:= \m{R}(\m{I} - \m{J}_p / p), \quad  \textnormal{and} \quad \m{B}_1:= \text{diag}(\g{\epsilon}) \m{J}_{p}.    
\end{equation}
 for some $\g{\epsilon}>0$ that controls how close we are to exact demographic parity. Under this setting, we introduce a general optimization framework for fair structured graph learning via a trace regularization and a fairness constraint on the partition matrix $\m{Q}$ as follows:
\begin{align}\label{loss:fair}
\begin{array}{ll}
\underset{\g{\Theta},~\g{Q}}{\text{minimize}} &
\begin{array}{c}
\hspace{.5cm}  L (\g{\Theta};\m{Y}) + \rho_1 \|\g{\Theta}\|_{1,\textnormal{off}} +  \rho_2 \trace\big(\m{Q} G(\g{\Theta}) \big)   
\end{array}\\
\text{subj. to} & \begin{array}[t]{l}
\hspace{.2cm} \g{\Theta} \in \mc{M},~~  \m{A}_1 \m{Q} \leq \m{B}_1,~~\text{and}~~\m{Q} \in \cup_{K} \mc{Q}_{pK}.
\end{array}
\end{array}
\end{align}
Here,  $G(\g{\Theta}): \mc{M} \rightarrow \mc{M}$ is a function of $\g{\Theta}$ (introduced in Sections~\ref{sec:fconcord} and \ref{sec:fblasso}).  

Figure~\ref{fig:SBM-example} depicts an illustrative example of two models: an unbalanced demographic representation ${\mathbf Q}^{'}$ (Left) and a balanced  representation ${\mathbf Q}$ (Right). In Problem \eqref{loss:fair}, the term $ \rho_1 \|\g{\Theta}\|_{1,\textnormal{off}}$ shrinks small entries of $\g{\Theta}$ to $0$, thus enforcing sparsity in $\g{\Theta}$ and consequently in $\mc{G}$. This term controls the presence of edges between any two nodes irrespective of the community they belong to, with higher values for $\rho_1$ forcing sparser estimators. The polyhedral constraint is the fairness constraint, enforcing that every community contains the $\g{\epsilon}$-approximate proportion of elements from each demographic group $\mc{D}_h$, $h \in [H]$, matching the overall proportion. The term $ \rho_2 \trace\left(\m{Q} G(\g{\Theta}) \right)$ enforces community structure in a similarity graph, $G(\g{\Theta})$. It should be mentioned that such trace regularization has been previously used in different
forms \citep{cai2015robust,amini2018semidefinite,hosseini2016learning,pircalabelu2020community,eisenach2020high} when estimating communities of networks. The intuition/interpretation behind the specific choice of the function $G(\g{\Theta})$ is to connect the trace term and the loss function $L (\g{\Theta};\m{Y})$ in the different graphical models; see discussions in Sections~\ref{sec:fairgraphs} and \ref{sec:fblasso}.  Specifically, In Equations \eqref{eqn:fgglasso}, \eqref{eqn:fglasso}, and \eqref{eqn:fair:ising}, we specify \(G(\g{\Theta}) = (n/2) \g{\Theta}^2\), \(G(\g{\Theta}) = (n/2) \g{\Theta}\), and \(G(\g{\Theta}) = (n/2) \g{\Theta}^2\) to obtain the fair CCONCORD \citep{khare2015convex}, a fair variant of graphical Lasso \cite{Friedman07}, and the fair Ising graphical model \cite{ravikumar2010high}, respectively. This approach allows us to present a unified framework while clearly distinguishing between the different methods later in  Sections~\ref{sec:fairgraphs} and \ref{sec:fblasso}.

Problem \eqref{loss:fair} is in general NP-hard due to its constraint on $\m{Q}$. However, it can be relaxed to a computationally feasible problem. To do so, we exploit algebraic properties of a community matrix $\m{Q}$. By definition, we see that $ \m{Q} $ must have the form $\m{Q}=\g{\Psi}\g{\Gamma} \g{\Psi}^{\intercal}$, where $\g{\Gamma}$ is block diagonal with size ${p_k} \times {p_k}$ blocks on the diagonal with blue all entries equal to  $1/|\mc{C}_k|$ associated with $k$-th community, $\g{\Psi}$ is some permutation matrix, and the number of communities $K$ is unknown. The set of all matrices ${\m{Q}} $ of this form is non-convex. The key observation is that any such $\m{Q}$ satisfies several convex constraints such as (i) all entries of $\m{Q}$ are nonnegative, (ii) all diagonal entries of $\m{Q}$ are $1/|\mc{C}_k|$, and (iii) $\m{Q}$ is positive semi-definite \citep{cai2015robust,amini2018semidefinite,li2021convex}. Without loss of generality, we assume that the permutation matrix corresponding to the ground truth communities is the identity, i.e., $\g{\Psi} = \m{I}$. Now, let 
\begin{equation*}
\m{A}:=[\m{A}_1 ; \m{I}_p],~~~\text{and} ~~\m{B}:=[\m{B}_1; \m{J}_p]. 
\end{equation*}
Thus, we propose the following relaxation: 
\begin{subequations}\label{eqn:loss:relaxfair}
\begin{align}~\label{eqn:loss:relaxfair:obj}
\begin{array}{ll}
\underset{\g{\Theta},~\m{\g{Q}}}{\text{minimize}} &
\begin{array}{c}
\hspace{.5cm} L (\g{\Theta};\m{Y}) + \rho_1 \|\g{\Theta}\|_{1,\textnormal{off}} + \rho_2 \trace\big(\m{Q} G(\g{\Theta}) \big)   
\end{array}\\
\text{subj. to} & 
\begin{array}[t]{l}
\hspace{.5cm} \g{\Theta} \in \mc{M},~~  \textnormal{and}~~\m{Q} \in \mc{N},
\end{array}
\end{array}
\end{align}
where 
\begin{align}\label{eqn:const:v}
\mc{N} &=\left\{\m{Q}\in \mb{R}^{p \times p}:~\m{Q} \succeq \m{0}, ~\m{0} \leq \m{A} \m{Q} \leq \m{B},~q_{ii}=1~\textnormal{for}~1 \leq i \leq p \right\}.
\end{align}
\end{subequations}
The solution of \eqref{eqn:loss:relaxfair} jointly learns the fair community matrix $\m{Q}$ and the network estimate $\g{\Theta}$. We highlight the following attractive properties of the formulation \eqref{eqn:loss:relaxfair}: (i) the communities are allowed to have significantly different sizes; (ii) the number of communities $ K$ may grow as $p$ increases; (iii) the knowledge of $K$ is not required for fair community detection, and (iv) the objective function \eqref{eqn:loss:relaxfair:obj} is convex in $\g{\Theta}$ given $\m{Q}$ and conversely.

In the framework \eqref{eq:fairnotion}, we prioritize demographic parity as the fairness criterion, which ensures that each demographic group is proportionately represented within each community. This choice is particularly suitable for applications in social networks \cite{aslan2021demographic}, education \cite{brusilovsky2016educational}, and healthcare \cite{pfohl2019creating}, where equitable representation of demographic groups can mitigate biases and promote inclusivity \cite{holstein2019fairness}. For comparison, other fairness notions such as equal opportunity and individual fairness could be considered. Equal opportunity focuses on ensuring that individuals from all demographic groups have equal chances of receiving positive outcomes \cite{hardt2016equality}, which is more relevant in scenarios like hiring or admissions where decisions are binary. Individual fairness, on the other hand, requires similar individuals to be treated similarly \cite{dwork2012fairness}, and is most applicable in personalized recommendation systems where fairness at the individual level is critical. However, demographic parity is chosen in our context because it directly addresses the proportional representation within communities, aligning well with the goals of fair clustering in graphical models.

\begin{rem}
Demographic attributes can be associated with nodes, samples, or both in a given data matrix \(\m{Y} \in \mathbb{R}^{n \times p}\). For example, in movie recommender systems, if nodes represent movies, demographic information may include whether a movie is old or new \cite{zhu2018fairness} (see Section~\ref{sec:exp:recom}). In the same datasets, users can also have demographic information, with potential biases related to users' gender or age \cite{wang2023survey}.  In this work, our methodology is designed for cases where demographic grouping is relevant to the nodes, which aligns with many real-world scenarios where demographic attributes are intrinsic to the entities represented by the nodes, such as users in social networks or items in recommender systems. Our approach could potentially be extended to cases where demographic groups are associated with samples, or both nodes and samples, by modifying \eqref{loss:fair} or incorporating concepts from fair supervised learning \cite{donini2018empirical}.
\end{rem}

\subsection{Algorithm}\label{sec:alg}

In order to solve (\ref{eqn:loss:relaxfair}), we use an \emph{alternating direction method of multipliers} (ADMM) algorithm \citep{Boyd11}. ADMM is an attractive algorithm for this problem, as it allows us to decouple some of the terms in (\ref{eqn:loss:relaxfair}) that are difficult to optimize jointly.  In order to develop an ADMM algorithm for (\ref{eqn:loss:relaxfair}) with guaranteed convergence, we reformulate it as follows:
\begin{align}\label{Eq:reformulate}
\nonumber
\minimize_{\g{Q},\g{\Theta},\g{\Omega}}~~~~&
 L(\g{\Theta};\m{Y}) + \rho_1 \|\g{\Omega}\|_{1,\textnormal{off}}+\rho_2 \trace\big(\m{Q} G(\g{\Omega}) \big)  \\
\text{subj. to} ~~~~ &\g{\Theta}=\g{\Omega},~~\g{\Theta} \in \mc{M},~~  \textnormal{and}~~\m{Q} \in \mc{N}.
\end{align}

The \textit{scaled} augmented Lagrangian function for (\ref{Eq:reformulate}) takes the form 
\begin{equation}\label{eqn:lagfun}
\begin{split}
  \Upsilon_{\gamma} (\g{\Theta},\m{Q},\g{\Omega},\m{W}) &:= L(\g{\Theta};\m{Y}) + \rho_1 \|\g{\Omega}\|_{1,\textnormal{off}} \\
& + \rho_2 \trace\big(\m{Q} G(\g{\Omega}) \big)+\frac{\gamma}{2}\left\|\g{\Theta}-\g{\Omega}+\m{W}\right\|^2_\textnormal{F},\\
\end{split}
\end{equation}
where $\g{\Theta} \in \mc{M}$, $\g{\Omega}$, and $\m{Q} \in \mc{N}$ are the primal variables;  ${\bf W}$ is the dual variable; and $\gamma>0$ is a dual parameter. We note that the scaled augmented Lagrangian can be derived from the usual Lagrangian by adding a quadratic term and completing the square~\citep[Section~3.1.1]{Boyd11}. 

\noindent  The proposed ADMM algorithm requires the following updates:
\begin{subequations}
\begin{eqnarray}
\m{Q}^{(t+1)}&\leftarrow & \underset{ \m{Q} \in \mc{N}}{\text{argmin}}~\Upsilon_{\gamma} \left(\m{Q},\g{\Omega}^{(t)},\g{\Theta}^{(t)},\m{W}^{(t)}\right), 
\label{eq:up:parti}\\
\g{\Omega}^{(t+1)}&\leftarrow & \underset{ \g{\Omega}}{\text{argmin}}~\Upsilon_{\gamma} \left(\m{Q}^{(t+1)},\g{\Omega},\g{\Theta}^{(t)},\m{W}^{(t)}\right), 
\label{eq:up:dtheta}\\
\g{\Theta}^{(t+1)} & \leftarrow &  \underset{\g{\Theta} \in\mc{M}}{\text{argmin  }}   \Upsilon_{\gamma} \left(\m{Q}^{(t+1)},\g{\Omega}^{(t+1)},\g{\Theta},\m{W}^{(t)}\right), 
\label{eq:up:theta}\\
\m{W}^{(t+1)} & \leftarrow & \m{W}^{(t)}+\g{\Theta}^{(t+1)}-\g{\Omega}^{(t+1)}. \label{eq:up:w}
\end{eqnarray}
\end{subequations}
\noindent  A general algorithm for solving (\ref{eqn:loss:relaxfair})  is provided in Algorithm \ref{Alg:general}.  Note that the update for $\g{\Theta}$, $\m{Q}$,  and $\g{\Omega}$ depends on the form of the functions $L$ and $G$, and is addressed in Sections~\ref{sec:fconcord} and \ref{sec:fblasso}. The $\m{Q}$ sub-problem in \ref{step:q} can be solved via a variety of convex optimization methods such as  CVX \citep{grant2014cvx} and ADMM~\citep{cai2015robust,amini2018semidefinite}. In the following sections, we consider special cases of (\ref{eqn:loss:relaxfair}) that lead to the estimation of Gaussian graphical model and an Ising model for, respectively, continuous and binary data.
\begin{algorithm}[t]
\small
\caption{Fair Graph Learning via Alternating Direction Method of Multipliers}
\label{Alg:general}
\textit{Initialize}  primal variables $\g{\Theta}^{(0)},\m{Q}^{(0)}$,  and $\g{\Omega}^{(0)}$; dual variable $\m{W}^{(0)}$; and positive constants $\gamma,\nu$. \\
\textit{Iterate} 
until the stopping criterion $\max \left\{\frac{\| {\g{\Theta}}^{(t+1)}- {\g{\Theta}}^{(t)} \|_{{\textnormal{F}}}^2}{\| {\g{\Theta}}^{(t)}\|_{\textnormal{F}}^2}, \frac{\| {\mathbf{Q}}^{(t+1)}- {\mathbf{Q}}^{(t)} \|_{\textnormal{F}}^2}{\| {\mathbf{Q}}^{(t)}\|_{\textnormal{F}}^2}  \right\} \le \nu$ is met, where ${\g{\Theta}}^{(t+1)}$ and ${\m{Q}}^{(t+1)}$ are the value of $\g{\Theta}$ and $\m{Q}$, respectively, obtained at the $t$-th iteration: 
\begin{enumerate}[label={\textnormal{\textit{S}\arabic*.}}]
\item \label{step:q} $\m{Q}^{(t+1)} \leftarrow  \underset{\m{Q} \in \mc{N}}{\textnormal{argmin}}~ \trace\big( \m{Q} G(\g{\Omega}^{(t)}) \big)$.
\item \label{step:dtheta}   $ \g{\Omega}^{(t+1)} \leftarrow  \underset{\g{\Omega} }{\textnormal{argmin}}~\rho_2 \trace\big( \m{Q}^{(t+1)}G(\g{\Omega}) \big)+ \rho_1 \|\g{\Omega}\|_1+ \frac{\gamma}{2} \left\|\g{\Theta}^{(t)}-\g{\Omega}+\m{W}^{(t)} \right\|_{\textnormal{F}}^2.$
\item \label{step:theta} $\g{\Theta}^{(t+1)} \leftarrow \underset{\g{\Theta} \in\mc{M}}{\textnormal{argmin}}~L(\g{\Theta};\m{Y})+ \frac{\gamma}{2} \left\|\g{\Theta}-\g{\Omega}^{(t+1)}+\m{W}^{(t)}\right\|_{\textnormal{F}}^2$.
\item \label{step:w} $\m{W}^{(t+1)} \leftarrow \m{W}^{(t)} + \g{\Theta}^{(t+1)}-\g{\Omega}^{(t+1)}$. 
\end{enumerate}
\end{algorithm}

We have the following global convergence result for~Algorithm~\ref{Alg:general}.

\begin{thm}\label{thm:bcd}
Algorithm~\ref{Alg:general} converges globally for any sufficiently large\footnote{The lower bound is given in \citep[Lemma 9]{wang2019global}.} $\gamma$, i.e., starting from any  $(\g{\Theta}^{(0)}, \m{Q}^{(0)},\g{\Omega}^{(0)},\m{W}^{(0)})$, it generates  $(\g{\Theta}^{(t)}, \m{Q}^{(t)}, \g{\Omega}^{(t)}, \m{W}^{(t)})$ that converges to a stationary point of \eqref{eqn:lagfun}.
\end{thm}

The proof of Theorem~\ref{thm:bcd} follows from the results in \cite{wang2019global}, with adaptations to \eqref{Eq:reformulate}, and is provided in Appendix~\ref{sec:proof:thm:bcd}.

In Algorithm~\ref{Alg:general}, Step~\ref{step:q} dominates the computational complexity in each iteration of ADMM \citep{cai2015robust,amini2018semidefinite}. In fact, an exact implementation of this optimization subproblem requires a full SVD, whose computational complexity is $O(p^3)$. When $p$ is as large as hundreds of thousands, the full SVD becomes computationally impractical. To facilitate implementation, one may apply an iterative approximation method in each iteration of ADMM, where the full SVD is replaced by a partial SVD. Although this type of method may converge to a local minimizer, since the SDP implementation can be viewed as a preprocessing step before $K$-means clustering, such a local minimum approximation can still be helpful. It is worth mentioning that when the number of communities $K$ is known, the computational complexity of ADMM can be much smaller than $O(p^3)$; see Remark~\ref{rem:compl} for further discussion.

\subsection{Two-Stage Graph Estimation and Fair Community Detection}
In this section, we introduce a variant of \eqref{eqn:loss:relaxfair} and propose a semidefinite programming framework for fair community detection. A key aspect of our approach is the interdependence between the graph structure $\g{\Theta}$ and the community partition matrix $\m{Q}$. Unlike standard methods that estimate these components separately, our formulation jointly models the interactions between community structures and graphical relationships, similar to prior work~\cite{pircalabelu2020community, hao2018simultaneous, kumar2020unified, hosseini2016learning}. However, our approach uniquely incorporates fairness constraints to ensure equitable clustering, distinguishing it from these methods; see Figure~\ref{fig:sub:chart1}.

An alternative to our proposed joint approach is a two-step sequential strategy (see Figure~\ref{fig:sub:chart2}), where the graph structure $\g{\Theta}$ is first estimated independently of community detection, and clustering is then performed based on the estimated adjacency structure \cite{pircalabelu2020community}. While this sequential approach simplifies the estimation process, it introduces several limitations:
\begin{itemize}
\item It directly relies on the empirical covariance, which is rank-deficient when the number of samples is small relative to the dimension.
    \item It does not incorporate feedback from community structures when estimating $\g{\Theta}$, leading to potential biases in graph inference.
    \item The adjacency matrix $\hat{\m{A}}$ is treated as an observed variable without accounting for estimation error, making the inferred graph structure less robust.
\end{itemize}
These limitations highlight the necessity of an integrated framework that jointly models community structures and graphical dependencies, as supported by prior work~\cite{pircalabelu2020community, hao2018simultaneous, kumar2020unified, hosseini2016learning}. The following two-step procedure provides a baseline for sequential estimation that will be compared with our proposed joint estimation framework introduced below.
\begin{figure}[ht]
\centering
\begin{subfigure}[b]{0.39\textwidth}
    \includegraphics[width=\textwidth]{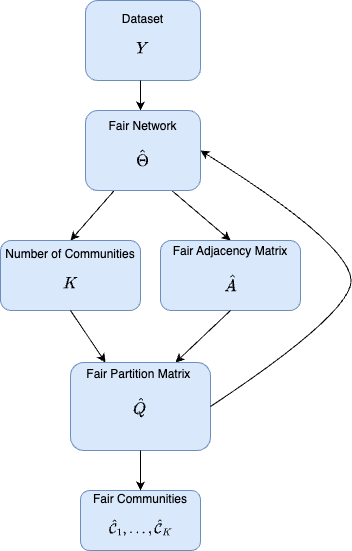}
    \caption{Joint procedure}
    \label{fig:sub:chart1}
\end{subfigure}
\hspace{1cm}
\begin{subfigure}[b]{0.41\textwidth}
    \includegraphics[width=\textwidth]{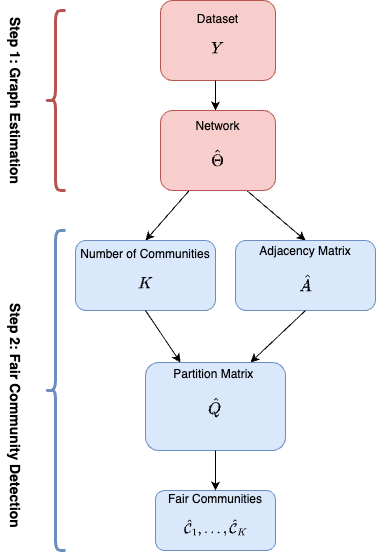}
    \caption{Two-step procedure}
    \label{fig:sub:chart2}
\end{subfigure}
\caption{Comparison of (a) our joint estimation framework with (b) a traditional two-step approach.}
\label{fig:chart}
\end{figure}

\begin{enumerate}[label={\textnormal{\textit{T}\arabic*.}}]
\item \label{step:known:theta} Solve the sparse graphical model estimation problem:
\begin{align}~\label{eqn:obj:know:graph}
\hat{\g{\Theta}}\leftarrow   \underset{}{\text{argmin  }}  
\hspace{.2cm}  L (\g{\Theta};\m{Y}) + \rho_1 \|\g{\Theta}\|_{1,\textnormal{off}}, ~\text{ subj. to }~\g{\Theta} \in \mc{M}. 
\end{align}
\item \label{step:known:q}
Estimate the community partition matrix by solving:
\begin{equation}\label{eqn:obj:know:q}
\min_{\m{Q}}  ~ \text{trace}(\hat{\m{E}}\m{Q})~\text{ subj. to }~\m{Q} \in \mc{N},
\end{equation}
where  
\begin{align*}
\hat{\m{E}}(\Theta) 
&= -(\m{I} - \hat{\m{D}})^{1/2} \, \hat{\m{A}}(\g{\Theta}) \, (\m{I} - \hat{\m{D}})^{1/2} 
+ \hat{\m{D}}^{1/2} \, \big(\m{I} - \hat{\m{A}}(\g{\Theta})\big) \, \hat{\m{D}}^{1/2}.
\end{align*}
Here, $\hat{\m{A}}(\g{\Theta})$ is the adjacency matrix with entries $\hat{a}_{ij} = 1$ if $\hat{\theta}_{ij} \neq 0$ and $0$ otherwise, and $\hat{\m{D}}$ is the diagonal matrix whose diagonal entries are the degrees of the nodes in $\hat{\m{A}}(\g{\Theta})$. 

\end{enumerate}

Compared to prior approaches, the joint framework integrates fairness constraints directly into the graph estimation process, rather than treating them as a separate post-processing step. Specifically, while two-step methods estimate $\g{\Theta}$ first and then infer communities independently, the joint formulation leverages the community structure to inform the estimation of the inverse covariance matrix. This integration ensures that fairness considerations are incorporated during graph estimation, rather than being applied afterward.
Furthermore, the semidefinite programming formulation in \eqref{eqn:obj:know:q} adapts prior work on robust community detection~\cite{cai2015robust} by explicitly incorporating fairness constraints, thereby ensuring demographic parity in learned community structures.
The key novelty of this joint approach lies in its simultaneous consideration of community structure and graph topology. The formulation directly integrates community dependencies into the graph estimation process, rather than relying on a sequential two-step method that lacks feedback between $\g{\Theta}$ and $\m{Q}$. By ensuring that community-aware constraints inform the estimation of $\boldsymbol{\Theta}$, the approach yields more robust and fairness-aware graph representations.

\subsection{Related Work}\label{sec:related_work}
To the best of our knowledge, the fair graphical model proposed here is the first model that can jointly learn fair communities simultaneously with the structure of a conditional dependence network. Related work falls into two categories: graphical model estimation and fairness. 

\paragraph{Estimation of graphical models.} There is a substantial body of literature on methods for estimating network structures from high-dimensional data, motivated by important biomedical and social science applications \citep{Liljeros01,Robins07,guo2011joint,Danaher13,Friedman07,Tan14,Guo15,tarzanagh2018estimation}. 
Since in most applications the number of model parameters to be estimated far exceeds the available sample size, the assumption of sparsity is made and imposed through regularization of the learned graph. An $\ell_1$ penalty on the parameters encoding the network edges is the most common choice \citep{Friedman07,meinshausen2006high,ElKaroui08,Cai11,Xue12,khare2015convex,peng2009partial}. This approach encourages sparse uniform network structures that may not be the most suitable choice for real world applications, that may not be uniformly sparse. As argued in \citep{Danaher13,guo2011joint,tarzanagh2018estimation} many networks exhibit different structures at different scales. An example includes a densely connected subgraph or community in the social networks literature. Such structures in social interaction networks may correspond to groups of people sharing common interests or being co-located \citep{tarzanagh2018estimation}, while in biological systems to groups of proteins responsible for regulating or synthesizing chemical products and in precision medicine the communities may be patients with common disease suceptibilities. An important part of the literature therefore deals with the estimation of hidden communities of nodes, by which it is meant that certain nodes are linked more often to other nodes which are similar, rather than to dissimilar nodes. This way the nodes from communities that are more homogeneous within the community than between communities where there is a larger degree of heterogeneity. Some of the initial work focused on either inferring connectivity information~\citep{marlin2009sparse} or performing graph estimation in case the connectivity or community information is known~\textit{a priori} \citep{Danaher13,guo2011joint,gan2019bayesian,ma2016joint,lee2015joint}, but not both tasks simultaneously. Recent developments consider the two tasks jointly and estimate the structured graphical models arising from heterogeneous observations~\citep{kumar2020unified,hosseini2016learning,hao2018simultaneous, tarzanagh2018estimation,kumar2019structured, gheche2020multilayer,pircalabelu2020community,cardoso2020algorithms,eisenach2020high}. 

\paragraph{Fairness.} There is a growing body of work on fairness in ML. Much of the research is on fair supervised methods; see, \cite{chouldechova2018frontiers,barocas-hardt-narayanan,donini2018empirical} and references therein. Our paper adds to the literature on fair methods for unsupervised learning tasks \citep{chierichetti2017fair,celis2017ranking,samadi2018price,tantipongpipat2019multi,oneto2020fairness,caton2020fairness,kleindessner2019guarantees}. 
We discuss the work on fairness most closely related to our paper. \cite{chierichetti2017fair} provides approximation algorithms that incorporate this fairness notion into $K$-center as well as $K$-median clustering. \cite{kleindessner2019guarantees} extend this to K-means and provide a provable fair spectral clustering method; they implement $K$-means on the subspace spanned by the smallest \textit{fair} eigenvectors of Laplacian matrix. Unlike these works, which assume that the graph structure and/or the number of communities is given in advance, an appealing feature of our method is to learn fair community structure while estimating heterogeneous graphical models.

\paragraph{Convex Community Detection.}
Convex optimization approaches for community detection can be traced back to the computer science and mathematical programming literature in the study of the planted partition problem; see, e.g., \cite{mathieu2010correlation,chen2014clustering,Ames2013}. For community detection under statistical models, various theoretical properties of convex optimization methods have been studied in depth recently; a partial list includes strong consistency\footnote{ In the graph context, an estimator is \emph{strongly consistent} if it identifies the true community structure with increasing certainty as the network size grows. It exhibits \emph{weak consistency} if the proportion of incorrectly labeled nodes diminishes to zero. For \emph{non-trivial recovery}, it performs better than random guessing in classifying nodes.} with a growing number of communities \cite{CSX2014,ChenXu14,cai2015robust}, sharp threshold under sparse networks for strong consistency \cite{Abbe14, Bandeira15, ABBK}, weak consistency \cite{guedon2016community, fei2018exponential}, non-trivial recovery \cite{montanari2016semidefinite,mei2017solving}, robustness against outlier nodes \cite{cai2015robust}, consistency under degree-corrected models \cite{CLX2015}, and consistency under weak associativity \cite{amini2018semidefinite}. Motivated by this line of work, our work provides the first convex programming approach for detecting fair communities within a known graph. This approach retains the computational complexity of convex community detection methods \cite{amini2018semidefinite,cai2015robust,li2021convex} while enhancing fairness in community structure.

\paragraph{Deep Graph Clustering Methods.}
Recent advances in deep learning have produced powerful graph neural network (GNN) based clustering methods \cite{tsitsulin2023graph,bianchi2020spectral,chien2021adaptive}. These methods leverage neural architectures to learn node embeddings that facilitate clustering on known graph structures. However, our framework differs fundamentally in several aspects. First, GNN methods are applied to graph-structured relational data, 
    whereas the methods proposed in this paper are applied to multivariate attributional data where the graph is hidden in the statistical model as a sparsity pattern in the precision matrix $\boldsymbol{\Theta}$ of the data.  We address the distinct problem of simultaneously estimating both the graph structure in the precision matrix and making fair community assignments based on the multivariate data. Second, while GNN approaches often achieve strong empirical performance on large graphs, they typically lack theoretical guarantees for graph recovery and fairness satisfaction. Our optimization framework provides finite-sample guarantees (Theorems~\ref{thm:three:fglasso} and~\ref{thm:bin}) and ensures global optimality (Theorem~\ref{thm:bcd}). Third, our estimated precision matrix has direct probabilistic interpretation as conditional dependencies in a Gaussian graphical model, enabling interpretable insights for domain scientists. Incorporating fairness into GNN architectures remains an active research area \cite{dong2022fairness,agarwal2021towards}, with most existing approaches treating fairness as a soft penalty rather than a hard constraint.




%% file: sections/sec_fairGMs.tex
\section{The Fair Community Detection and  Graphical Models}\label{sec:fairgraphs}

In the following subsections, we consider two special cases of~\eqref{eqn:loss:relaxfair} that lead to estimation of graphical models for continuous and binary data.

\subsection{Fair Pseudo-Likelihood Graphical Model}\label{sec:fconcord}
Suppose $\m{y}_\kappa = (y^1_\kappa,\ldots,y^p_\kappa)$ are i.i.d. observations from $N(\g{0}, \g{\Sigma})$, for $\kappa = 1,\ldots , n$. Denote the sample of the $i$th variable as $\m{y}^i =  (y^i_1,\ldots,y^i_n)^\top \in \mb{R}^p$. 
Let $\omega_{ij}:=-\theta_{ij}/\theta_{ii}$, for all $j \neq i$. We note that the set of nonzero coefficients of $\omega_{ij}$ is the same as the set of nonzero entries in the row vector of $\theta_{ij}~ (i\neq j)$, which defines the set of neighbors of node $\theta_{ij}$. Using an $\ell_1$-penalized regression, \cite{meinshausen2006high} estimates the zeros in $\g{\Theta}$ by fitting separate Lasso regressions for each variable $\m{y}^i$ given the other variables as follows
\begin{equation}
F_i(\g{\Theta};\m{Y})= \| {\bf y}^i - \sum_{j \neq i} \omega_{ij} {\bf y}^j
\|^2 + \rho_1 \sum_{1 \leq i < j \leq p} \left| \omega_{ij}\right|, ~~\textnormal{where}~~ \omega_{ij}=-\theta_{ij}/\theta_{ii}.
\end{equation}

These individual Lasso fits give neighborhoods that link each variable to others. \cite{peng2009partial} improve this neighborhood selection method by taking the natural symmetry in the problem into account (i.e., $\theta_{ij} =\theta_{ij}$ ), and propose the following joint objective function (called SPACE):
\begin{align}\label{eqn:space}
\nonumber
 F(\g{\Theta};\m{Y})&= \frac{1}{2} \sum_{i=1}^p \big( -n\log \theta_{ii} + w_i \| {\bf y}^i - \sum_{j \neq i} \dot{\omega}_{ij}\sqrt{\frac{\theta_{jj}}{\theta_{ii}}} {\bf y}^j
\|^2\big) + \rho \sum_{1 \leq i < j \leq p} \left| \dot{\omega}_{ij}\right|\\
    &= \frac{1}{2}
    \sum_{i=1}^p \big(-n\log \theta_{ii} +  w_i \| {\bf y}^i + \sum_{j \neq i} 
      \frac{\theta_{ij}}{\theta_{ii}} {\bf y}^j \|^2 \big)+ \rho_1 \sum_{1 \leq i < j \leq p} \left| \dot{\omega}_{ij} \right|,
\end{align}
where $\{w_i\}_{i=1}^p$ are nonnegative weights  and  $\dot{\omega}_{ij}=-\frac{\theta_{ij}}{\sqrt{\theta_{ii}\theta_{jj}}}$ denotes the partial correlation between the $i$th and $j$th variables for $1\leq i \neq j \leq p$. Note that $\dot{\omega}_{ij}=\dot{\omega}_{ji}$ for $i\neq
j$. 

It is shown in \citep{khare2015convex} that the above expression is not convex.  
Setting $w_i=  \theta_{ii}^2$ and putting the $\ell_1$-penalty term on the partial covariances $ \theta_{ij}$ instead of on the partial correlations $ \dot{\omega}_{ij}$, they obtain a convex pseudo-likelihood approach with good model selection properties called CONCORD. Their objective takes the form  
\begin{align} 
 F(\g{\Theta};\m{Y})&= \sum_{i=1}^p \big( - n\log \theta_{ii} +  \frac{1}{2}\|
  \theta_{ii} {\bf y}^i + \sum_{j \neq i} \theta_{ij} {\bf y}^j \|_2^2\big) + \rho_1 \sum_{1 \leq i < j \leq p} |\theta_{ij}|. \label{eq:concord}
\end{align}
Note that the matrix version of the CONCORD objective can be obtained by setting $L (\g{\Theta}; \m{Y}) = n/2[-\log |\text{diag}(\g{\Theta})^2 | + \text{trace}(\m{S}\g{\Theta}^2)]$ in \eqref{eqn:obj:graph}. Importantly, as discussed in \citep[Section~3.2]{khare2015convex}, CONCORD does not enforce positive definiteness of ${\g{\Theta}}$ during optimization, focusing instead on identifying the sparsity pattern; a positive definite estimate can subsequently be obtained via standard procedures such as maximum likelihood estimation with known zeros.

Our proposed fair graphical model formulation, called  FCONCORD, is a fair version of CONCORD from \eqref{eq:concord}. In particular, letting $G(\g{\Theta})= \frac{n}{2}\g{\Theta}^2$ and
$$\mc{M} = \left\{ \g{\Theta} \in  \mb{R}^{p \times p}:~\theta_{ij} =\theta_{ji}, \text{ and } \theta_{ii} > 0, \mbox{ for every } 1  \leq i,j \leq p \right\},
$$ in \eqref{eqn:loss:relaxfair}, our problem takes the form
\begin{align}\label{eqn:fgglasso}
\nonumber
\minimize_{\g{\Theta},~\g{Q}}~~&\frac{n}{2} \big[-\log |\text{diag}(\g{\Theta})^2|+ \trace\big( (\m{S}+ \rho_2 \m{Q} ) \g{\Theta}^2\big) \big]+\rho_1 \|\g{\Theta}\|_{1,\textnormal{off}} \\
\text{subj. to} ~~~~~~& \g{\Theta} \in \mc{M}~~\text{and}~~\m{Q} \in \mc{N}.
\end{align}
Here, $\mc{M}$ and $\mc{N}$ are the graph adjacency and fairness constraints, respectively.

\begin{rem}
When $\rho_2=0$, i.e., without a fairness constant and the second trace term, 
the objective in \eqref{eqn:fgglasso} reduces to the objective of the CONCORD estimator, and is similar to those of SPACE \citep{peng2009partial}, SPLICE \citep{rocha2008path}, and SYMLASSO \citep{friedman2010applications}. Our framework is a generalization of these methods to fair graph learning and community detection, when the demographic group representation holds.
\end{rem}

Problem~\eqref{eqn:fgglasso} can be solved using Algorithm~\ref{Alg:general}. The update for  $\g{\Omega}$ and  $\g{\Theta}$ in \ref{step:dtheta} and \ref{step:theta} can be derived by minimizing
\begin{subequations}
\begin{align} 
\Upsilon_{1,\gamma} ( \g{\Omega}) &:= \frac{n  \rho_2 }{2} \trace\big(\m{Q} \g{\Omega}^2\big) + \rho_1 \|\g{\Omega}\|_{1,\textnormal{off}} +\frac{\gamma }{2}\left\|\g{\Theta}-\g{\Omega}+\m{W}\right\|_F^2,\label{eqn:aug:sub:dtheta}\\
\Upsilon_{2,\gamma} ( \g{\Theta}) &:= \frac{n}{2} \big[-\log |\text{diag}(\g{\Theta})^2| + \trace\big(\m{S}\g{\Theta}^2\big) \big] +\frac{\gamma}{2}\left\|\g{\Theta}-\g{\Omega}+\m{W}\right\|_F^2,\label{eqn:aug:sub:theta}
\end{align}
\end{subequations}
with respect to  $\g{\Omega}$ and $\g{\Theta}$, respectively.

For $1 \leq i \leq j \leq p$, define the matrix function $ \m{T}_{ij}: \mc{M} \rightarrow \mc{M}$ by
\begin{subequations}
\begin{align}\label{eq:ipm section} 
\m{T}_{ij} (\g{\Omega}) &\leftarrow \argmin_{\ti{\g{\Omega}}} \big\{ \Upsilon_{1,\gamma}(\ti{\g{\Omega}}) : ~\ti{\omega}_{kl} = \omega_{kl} \; \forall (k,l) \neq (i,j) \big\}, \\
\m{T}_{ij} (\g{\Theta}) & \leftarrow \argmin_{\ti{\g{\Theta}}} \big\{ \Upsilon_{2,\gamma}(\ti{\g{\Theta}}) : ~\ti{\theta}_{kl} = \theta_{kl} \; \forall (k,l) \neq (i,j) \big\} .
\end{align}
\end{subequations}
For each \((i,j)\), \(\m{T}_{ij} (\g{\Omega})\) and \(\m{T}_{ij}(\g{\Theta})\) update the \((i,j)\)-th entry with the minimizer of \eqref{eqn:aug:sub:dtheta} and \eqref{eqn:aug:sub:theta} with respect to \(\omega_{ij}\) and \(\theta_{ij}\), respectively, holding all other variables constant. Recall that \(s_{ij}\) denotes the \((i,j)\)-th entry of the sample covariance matrix \(\m{S}\). Given $\m{T}_{ij} (\g{\Omega})$ and $\m{T}_{ij}(\g{\Theta})$, the update for  $\g{\Omega}$ and  $\g{\Theta}$ in \ref{step:dtheta} and \ref{step:theta} can be obtained by a similar coordinate-wise descent algorithm proposed in \cite{peng2009partial,khare2015convex}. Closed form updates for $\m{T}_{ij} (\g{\Omega})$ and $\m{T}_{ij}(\g{\Theta})$ are provided in Lemma~\ref{lem:fcon:update}.

\begin{lem}\label{lem:fcon:update}
Set $\gamma_n :=\gamma n$. For all $1\leq i \leq p$, let
\begin{align*}
&a_{i}:=s_{ii} +\gamma_n, \qquad \qquad \quad  b_{i}:=\sum_{j \neq i} \theta_{ij} s_{ij}+ \gamma_n( w_{ii} -\omega_{ii}),\\
&c_{i}:=\rho_2 q_{ii}+ \gamma_n, \qquad \qquad d_{i}:=  \rho_2  \sum_{j \neq i} q_{ij} \omega_{ij}+ \gamma_n( w_{ii}+\theta_{ii}).
\end{align*}
Further, for $1\leq i < j \leq p$, let
\begin{small}
\begin{align*}
&a_{ij} := (s_{ii} + s_{jj}) + \gamma_n, \quad \quad  \quad b_{ij} :=  \big(\sum_{j' \neq j} \theta_{ij'}  s_{jj'} + \sum_{i' \neq i} \theta_{i'j} s_{ii'}\big)+ \gamma_n(w_{ij}-\omega_{ij}), \\
&c_{ij} := \rho_2(q_{ii} + q_{jj}) +\gamma_n, \qquad d_{ij} := \rho_2 \big(\sum_{j' \neq j} \omega_{ij'} q_{jj'} +  \sum_{i' \neq i} \omega_{i'j} q_{ii'}\big)+ \gamma_n(w_{ij}+\theta_{ij}). 
\end{align*}
\end{small} 
Then, for all $1\leq i \leq p$:
\begin{subequations}
\begin{align} 
~~~~~~\big(\m{T}_{ii} (\g{\Omega}) \big)_{ii} &= \frac{-d_i}{c_i} , \qquad \qquad \qquad \big(\m{T}_{ii} (\g{\Theta}) \big)_{ii}=  \frac{1}{2a_i} (-b_i+\sqrt{b_i^2+4a_i}),\label{eq:fggl:offdiag}
\end{align}
and for all $1\leq i < j \leq p$:
\begin{align}
\big(\m{T}_{ij} (\g{\Omega}) \big)_{ij} &= S \big(-\frac{d_{ij}}{c_{ij}}, \frac{\rho_1}{\gamma_n} \big), \quad ~~~~~ \big(\m{T}_{ij} (\g{\Theta}) \big)_{ij} = -\frac{b_{ij}}{a_{ij}}, \label{eq:fggl:diag}
\end{align}
\end{subequations}
where $S(\alpha,\beta) := \textnormal{sign}(\alpha)\max( |\alpha|-\beta, 0)$. 
\end{lem}

Note that, alternatively, one can minimize the augmented Lagrangian function \eqref{eqn:lagfun} and subproblem~\eqref{eqn:aug:sub:dtheta} over $\Upsilon_{1,\gamma} (\g{\Omega})-\log|\text{diag}(\g{\Omega})^2|$. The additional term ensures that $\omega_{ii} >0$ at each iteration of the ADMM algorithm. In this case, the coordinate-descent update for $\big(\m{T}_{ii} (\g{\Omega}) \big)_{ii}$ in \eqref{eq:fggl:offdiag} becomes $$\left(\m{T}_{ii} (\g{\Omega}) \right)_{ii} =  \frac{1}{2c_i} \left(-d_i+\sqrt{d_i^2+4c_i}\right).$$

\begin{rem}\label{rem:compl}
In the case when $K$ and $H$ are known, the computational complexity of Step~\ref{step:q} in Algorithm~\ref{Alg:general}, corresponding to updating the estimators, is of the same order as that of the CONCORD~\cite{khare2015convex}, SPACE~\cite{peng2009partial}, and SYMLASSO~\cite{friedman2010applications} estimators. In fact, computing the fair clustering matrix $\m{Q}$ requires $O((p-H+1)^2K)$ operations. On the other hand, it follows from \citep[Lemma~5]{khare2015convex} that the $\g{\Theta}$ updates can be performed with complexity $\min\!\left(O(np^2), O(p^3)\right)$. This shows that when the number of communities is known, the computational cost of each iteration of FCONCORD is
\[
\max\!\left(\min\!\left(O(np^2), O(p^3)\right),\, O((p-H+1)^2K)\right),
\]
which is comparable to the aforementioned estimators.
\end{rem}

Note that Eq.~\eqref{loss:fair} provides a unifying framework that includes a fair variant of the graphical Lasso (GLASSO) \cite{Friedman07}, SPACE~\citep{peng2009partial}, SPLICE \citep{rocha2008path}, and SYMLASSO \citep{friedman2010applications}. In particular, letting $L( \g{\Theta};\m{Y}) =n/2[-\log \det (\g{\Theta}) + \trace(\m{S}\g{\Theta})]$, $G(\g{\Theta})={(n/2)}\g{\Theta}$, and $\mc{M} = \{\g{\Theta}: \g{\Theta} \succ 0~\text{ and}~\g{\Theta}=\g{\Theta}^\top\}$, we obtain
\begin{align}\label{eqn:fglasso}
\nonumber
\minimize_{\g{\Theta},~\g{Q}}~~& \frac{n}{2} \big[-\log \det\g{\Theta}+ \trace\big( (\m{S}+ \rho_2 \m{Q} ) \g{\Theta}\big) \big]+\rho_1 \|\g{\Theta}\|_{1,\textnormal{off}} \\
\text{subj. to} ~~~~~~& \g{\Theta} \in \mc{M}~~\text{and}~~\m{Q} \in   \mc{N},
\end{align}
which can be considered as a fair variant of cluster-based GLASSO~\citep{kumar2020unified,hosseini2016learning,hao2018simultaneous, gheche2020multilayer,pircalabelu2020community,cardoso2020algorithms,eisenach2020high}.   

\subsubsection{Large Sample Properties of FCONCORD}\label{sect:est}

We show that under suitable conditions, the FCONCORD estimator achieves both model selection consistency and estimation consistency. 

As in other studies \citep{khare2015convex,peng2009partial}, for the convergence analysis we assume that the diagonal of the graph matrix $\g{\Theta}$ and partition matrix $\m{Q}$ are known. 
Let $\g{\theta}^o =(\theta_{ij})_{1 \leq i < j \leq p}$ and $\m{q}^o =(q_{ij})_{1 \leq i < j \leq p}$ denote the vector of off-diagonal entries of $\g{\Theta}$ and $\m{Q}$, respectively. Let $\g{\theta}^d$ and $\m{q}^d$ denote the vector of diagonal entries of $\g{\Theta}$ and $\m{Q}$, respectively.  Let $\bar{\g{\theta}}^o, \bar{\g{\theta}}^d, \bar{\m{q}}^o$, and $\bar{\m{q}}^d$ denote the true value of $ {\g{\theta}}^o, {\g{\theta}}^d, \m{q}^o$, and $\m{q}^d$, respectively. Let $\mc{B}$ denote the set of non-zero entries in the vector $\bar{\g{\theta}}^o$ and define
\begin{align}\label{eqn:rate:quant}
 q &:= |\mc{B}|,~~~~\Psi(p,H,K) := (p-H+1) ((p-H+1)/K-1).
\end{align}
In our consistency analysis, we let the regularization parameters $\rho_1=\rho_{1n}$ and $\rho_2=\rho_{2n}$ vary with $n$.  

The following standard assumptions are required:
\\
\noindent   {\bf Assumption A}
\begin{enumerate}[label={\textbf{(A\arabic*})}]
\item \label{assu:subgauss}  
The random vectors $\m{y}_{1},\dots, \m{y}_{n}$ are \emph{i.i.d.} sub-Gaussian for every $n \geq 1$, i.e., there exists $\tau>0$ such that $\|\m{u}^\top\m{y}_i\|_{\psi_2}\leq \tau \sqrt{\mb{E} (\m{u}^\top\m{y}_i)^2}$, $\forall \m{u}\in\mb{R}^p$. Here, $\|\m{y}\|_{\psi_2}= \sup_{t\ge 1}(\mb{E}|\m{y}|^t)^{\frac{1}{t}}/\sqrt{t}$. 
\item \label{assu:beig} 
There exist constants $\tau_1, \tau_2 \in (0,\infty)$ such that 
$$
\tau_1< \Lambda_{\min}(\bar{\g{\Theta}}) \leq \Lambda_{\max}(\bar{\g{\Theta}}) < \tau_2.
$$
\item \label{assu:beigq}  
There exists a constant $ \tau_3 \in (0,\infty)$ such that   
$$0\leq \Lambda_{\min}(\bar{\m{Q}} ) \leq \Lambda_{\max}(\bar{\m{Q}}) < \tau_3.$$
\item  \label{assu:comm}
For any $K, H \in [p]$, we have $K \leq p-H+1$.
\item \label{assu:incoh} 
There exists a constant $\delta  \in (0,1] $ such that, for any $(i,j) \in \mc{B}^c$
\begin{align*}
\left | \bar{\m{H}}_{ij,\mc{B}}\bar{\m{H}}_{\mc{B}, \mc{B}} ^{-1}\sign(\bar{\g{\theta}}^o_{\mc{B}}) \right | \leq (1-\delta), 
\end{align*}
where for $1 \leq i,j,t,s \leq p$ satisfying $i < j$ and $t < s$, 
\begin{align}\label{eqn:secder}
\bar{\m{H}}_{ij,ts}:= \mb{E}_{\bar{\g{\theta}}^o} \Big(\frac{\partial^2 L(\bar{\g{\theta}}^d, {\g{\theta}}^o;\m{Y})}{\partial \theta_{ij} \partial \theta_{ts}} \Big|_{\g{\theta}^o=\bar{\g{\theta}}^o} \Big).
\end{align}
\end{enumerate}

Assumptions~\ref{assu:beig}--\ref{assu:beigq} guarantee that the eigenvalues of the true graph matrix $\bar{\g{\Theta}}$ and those of the true membership matrix $\bar{\m{Q}}$ are well-behaved. Assumption~\ref{assu:comm} links how $H, K$ and $p$ can grow with $n$. Note that $K$ is limited in order for fairness constraints to be meaningful; if $K>p-H+1$ then there can be no community with $H$ nodes among which we enforce fairness. Assumption~\ref{assu:incoh} corresponds to the incoherence condition in \cite{meinshausen2006high}, which plays an important role in proving model
selection consistency of $\ell_1$ penalization problems. \cite{zhao2006model} show that such a condition is almost necessary and sufficient for model selection consistency in Lasso regression, and they provide some examples when this condition is satisfied. Note that Assumptions~\ref{assu:subgauss}, \ref{assu:beig}, and \ref{assu:incoh} are identical to Assumptions (C0)--(C2) in \cite{peng2009partial}. Further, it follows from \cite{peng2009partial} that under Assumption~\ref{assu:incoh} for any $(i,j) \in \mc{B}^c$,
\begin{equation}\label{eq:mm}
\| \bar{\m{H}}_{ij,\mc{B}} \bar{\m{H}}_{\mc{B},\mc{B}}^{-1}\|\leq M(\bar{\g{\theta}}^o)
\end{equation}
for some finite constant $M(\bar{\g{\theta}}^o)$.

Next, inspired by \cite{peng2009partial,khare2015convex}, we prove estimation consistency for the nodewise FCONCORD.

\begin{thm}\label{thm:three:fglasso}
Suppose Assumptions~\ref{assu:subgauss}--\ref{assu:incoh} are satisfied. Assume further that
$p = O(n^{\alpha})$ for some ${\alpha} > 0$, $\rho_{1n} =O(\sqrt{\log p/n})$, $n > O(q \log (p))$ as $n \rightarrow \infty$, $ \rho_{2n} =O(\sqrt{\log(p-H+1)/n})$, $ \rho_{2n} \leq \delta\rho_{1n}/((1+M(\bar{\g{\theta}}^o))\tau_2 \tau_3)$, and $\g{\epsilon}=\g{0}$  in \eqref{eqn:def:AB}. Then, there exist finite constants $C(\bar{\g{\theta}}^o)$ and $D(\bar{\m{q}}^o)$, such that for any $\eta>0$, the following events hold with probability at least $1 - O(\exp(-\eta\log p))$:
 \begin{itemize}
\item There exists a minimizer $(\wh{\g{\theta}}^o, \wh{\m{q}}^o)$ of \eqref{eqn:fgglasso} such that 
\begin{equation*}
\hspace{-.25cm}\max\big(\| \widehat{\g{\theta}}^o- \bar{\g{\theta}}^o \|, \|\wh{\m{q}}^o- \bar{\m{q}}^0\| \big) \\
\leq \max \big( C(\bar{\g{\theta}}^o) \rho_{1n}\sqrt{q}, D(\bar{\m{q}}^o) \rho_{2n} \sqrt{\Psi(p,H,K)}\big),
\end{equation*}    
where $q$ and $\Psi(p,H,K)$ are defined in \eqref{eqn:rate:quant}.
\item   If  $\min_{(i,j) \in \mc{B}} \bar{\theta}_{ij} \geq  2 C(\bar{\g{\theta}}^o) \rho_{1n}\sqrt{q}$,  then $\wh{\g{\theta}}^o_{\mc{B}^c} =0$. 
  \end{itemize}
\end{thm}

Theorem \ref{thm:three:fglasso} provides sufficient conditions on the quadruple $(n, p,H, K)$ and the model parameters for the FCONCORD to succeed in consistently estimating the neighborhood of every node in the graph  and communities simultaneously. Notice that if $H=1$ (no fairness) and $K=p$ (no clustering) we recover the results of \cite{khare2015convex,peng2009partial}.
%

\subsection{Fair Ising Graphical Model}\label{sec:fblasso}

In the previous section, we studied the fair estimation of graphical models for continuous data. Next, we focus on estimating an Ising Markov random field~\citep{Ising25}, suitable for binary or categorical data. Let $\m{y} = (y_1, \ldots, y_p)^\top \in \{0, 1\}^p$ denote a binary random vector. The Ising model specifies the probability mass function 
\begin{equation}\label{eqn:ising:prob}
p(\m{y}) = \frac{1}{\mc{W}(\g{\Theta})} \exp \Big( \sum_{j=1}^{p} \theta_{jj} y_j + \sum_{1\leq j< j'\leq p} \theta _{jj'} y_j y_{j'}\Big). 
\end{equation}
Here, $\mc{W}(\boldsymbol{\Theta})$ is the partition function, which ensures that 
the probability mass function in~\eqref{eqn:ising:prob} sums to one; $\boldsymbol{\Theta}$ is a $p \times p$ symmetric matrix that specifies the graph structure, where $\theta_{jj'} = 0$ implies that the $j$th and $j'$th variables are conditionally independent given the remaining ones.

Several sparse estimation procedures for this model have been proposed. \cite{lee06} considered maximizing an $\ell_1$-penalized log-likelihood for this model. Due to the difficulty in computing the log-likelihood with the expensive partition function, alternative approaches have been considered. For instance, \cite{ravikumar2010high} proposed a neighborhood selection approach which involves solving $p$ logistic regressions separately (one for each node in the network), which leads to an estimated parameter matrix that is in general not symmetric. In contrast, others have considered maximizing an $\ell_1$-penalized pseudo-likelihood with a symmetric constraint on $\g{\Theta}$ \citep{guo2011joint,Guo15,Tan14,tarzanagh2018estimation}. Under the probability model above,  the negative $\log$-pseudo-likelihood for $n$ observations takes the form
\begin{equation}\label{eqn:ising:loss}
L(\g{\Theta};\m{Y})= -\sum_{j=1}^{p}\sum_{j'=1}^{p} \theta_{jj'} s_{jj'}+ \frac{1}{n} \sum_{\kappa=1}^{n}\sum_{j=1}^{p} \log \big(1 + \exp (\theta_{jj}+ \sum_{j'\neq j}\theta_{jj'} y_{\kappa j'})\big).
\end{equation}

We propose to additionally impose the fairness constraints on $\g{\Theta}$ in (\ref{eqn:ising:loss}) in order to obtain a sparse binary network with fair clustering.  Let $\mc{M} =\{ \g{\Theta} \in  \mb{R}^{p \times p}:~\g{\Theta}=\g{\Theta}^\top \}$, $G(\g{\Theta})=  {(n/2)}\g{\Theta}^2$ and choose $\iota_n \in (0,1)$ such that $n\iota_n \rightarrow 0$ as $n \rightarrow \infty$. Under this setting, we consider the following criterion
\begin{align}\label{eqn:fair:ising}
\nonumber
\minimize_{\g{\Theta}, \g{Q}}~~~~~~& \sum_{j=1}^p \sum_{j'=1}^p -n\theta_{jj'} s_{jj'}+\frac{\rho_2n}{2} \trace\big((\m{Q}+\iota_n\m{I})\g{\Theta}^2\big)\\
\nonumber
& +  \sum_{\kappa=1}^n \sum_{j=1}^p \log \big( 1+ \exp(\theta_{jj}+ \sum_{j'\ne j} \theta_{jj'} y_{\kappa j'}) \big)+ \rho_1 \sum_{1 \leq i < j \leq p} |\theta_{ij}|,\\
\text{subj. to} ~~~~~~~~~ & \g{\Theta} \in \mc{M}~~\text{and}~~\m{Q} \in   \mc{N}.
\end{align}
Here, $\mc{M}$ and $\mc{N}$ are the graph and fairness constraints, respectively. Note that $\iota_n$ is used to provide a theoretical guarantee of convergence to a local minimum and throughout all numerical experiments is set to zero; see, Lemma~\ref{lem:disc:ising} in Appendix for further details. 

We refer to the solution to \eqref{eqn:fair:ising} as the \emph{Fair Binary Network} (FBN).  An interesting connection can be drawn between our technique and a fair variant of Ising block model discussed in \cite{bert16}, which is a perturbation of the mean field approximation of the Ising model known as the Curie-Weiss model: the sites are partitioned into two blocks of equal size and the interaction between those within the same block is stronger than across blocks, to account for more order within each block. 

An ADMM algorithm for solving \eqref{eqn:fair:ising} is given in Algorithm~\ref{Alg:general}. The update for $\g{\Omega}$ in \ref{step:dtheta} can be obtained from \eqref{eqn:aug:sub:dtheta} by replacing $\g{\Omega}^2$ with $\g{\Omega}$. We solve the update for $\g{\Theta}$ in \ref{step:theta} using a relaxed variant of Barzilai-Borwein method~\citep{barzilai1988two}. The details are given in \citep[Algorithm~2]{tarzanagh2018estimation}.

\subsubsection{Large Sample Properties of FBN}\label{sect:est:bin}

In this section, we present the model selection consistency property for the Ising model. The spirit of the proof is similar to \cite{ravikumar2010high}, but their model does not include the membership matrix $\m{Q}$ nor fairness constraints.

Similar to Section~\ref{sect:est}, let $\g{\theta}^o =(\theta_{ij})_{1 \leq i < j \leq p}$ and $\m{q}^o =(q_{ij})_{1 \leq i < j \leq p}$ denote the vector of off-diagonal entries of $\g{\Theta}$ and $\m{Q}$, respectively. Let $\g{\theta}^d$ and  $\m{q}^d$ denote the vector of diagonal entries of $\g{\Theta}$ and $\m{Q}$, respectively.  Let $ \bar{\g{\theta}}^o, \bar{\g{\theta}}^d, \bar{\m{q}}^o, \bar{\m{q}}^d$, $\bar{\g{\Theta}}$ and $\bar{\m{Q}}$ denote the true value of $ {\g{\theta}}^o, {\g{\theta}}^d, \m{q}^o, \m{q}^d$, $\g{\Theta}$ and $\m{Q}$, respectively. Let $\mc{B}$ denote the set of non-zero entries in the vector $\bar{\theta}^o$, and let $q = |\mc{B}|$. Denote the log-likelihood for the $\kappa$-th observation by
\begin{equation}\label{eqn:ising:ithloss}
L_{\kappa}(\g{\theta}^d, {\g{\theta}}^o; \m{Y}) = -\sum_{j=1}^{p} y_{ij} \big(\theta_{jj}  + \sum_{j\neq j'}\theta_{jj'} y_{\kappa j'}\big)+ \log \big(1 + \exp( \g{\theta}_{jj} \sum_{j\neq j'}\theta_{jj'} y_{\kappa j'})\big).
\end{equation}

The population Fisher information matrix of $L_{\kappa}$ at $(\bar{\g{\theta}}^d, \bar{\g{\theta}}^o)$ and its sample counterpart can be jointly expressed as
\[
\bar{\m{H}} = \mathbb{E}\!\left[\nabla^2 L_{\kappa}(\bar{\g{\theta}}^d, \bar{\g{\theta}}^o; \m{Y})\right], \quad 
\bar{\m{H}}^{n} = \frac{1}{n}\sum_{\kappa=1}^{n} \nabla^2 L_{\kappa}(\bar{\g{\theta}}^d, \bar{\g{\theta}}^o; \m{Y}),
\]
respectively.

 Let 
\begin{align*}
v_{{ \kappa} j}&= \dot{v}_{{ \kappa} j }(1-\dot{v}_{{ \kappa} j}),~~\textnormal{where}~~\dot{v}_{{ \kappa} j}=\frac{\exp (\theta_{jj}+ \sum_{j'\neq j}\theta_{jj'}y_{{ \kappa} j'})}{1 + \exp (\theta_{jj}+\sum_{j'\neq j}\theta_{jj'}y_{{ \kappa} j'})},~~\textnormal{and}\\
\widetilde{\m{y}}_j&=\left(\sqrt{v_{1j}} y_{1j}-\dot{\m{y}}_j,\ldots,\sqrt{v_{nj}}y_{nj}-\dot{y}_j\right)^\top,~~\textnormal{where}~~\dot{\m{y}}_j= 1/n\sum_{i=1}^n\sqrt{v_{ { \kappa} j}} y_{{ \kappa} j}.
\end{align*}
We use  $\tilde{\m{X}}=(\tilde{\m{X}}_{(1,2)},\cdots,\tilde{\m{X}}_{(p-1,p)})$ to denote an $np$ by $\binom{p}{2}$ matrix, with
$$
\tilde{\m{X}}_{(j,j^{'})}=\Big(\m{0}_n, ..., \m{0}_n,\underbrace{\widetilde{\m{y}}_j}_{j-\textnormal{th block}}, \m{0}_n, ...,
\m{0}_n, \underbrace{\widetilde{\m{y}}_{j^{'}}}_{j^{'}-\textnormal{th block}},  \m{0}_n,
\ldots, \m{0}_n\Big)^\top,
$$
where $\m{0}_n$ is an $n$-dimensional column vector of zeros. Let  $\tilde{\m{X}}^{(i,j)}$ be the $[(j-1)n + i]$-th row of $\tilde{\m{X}}$ and $\tilde{\m{X}}^{(i)} = (\tilde{\m{X}}^{(i,1)},\ldots, \tilde{\m{X}}^{(i,p)})$. Let $\m{T}= \mb{E}(\tilde{\m{X}}^{(i)} (\tilde{\m{X}}^{(i)})^\top)$ and $\m{T}^n=1/n\sum_{\kappa=1}^n\tilde{\m{X}}^{(\kappa)} (\tilde{\m{X}}^{(\kappa)})^\top$ as its sample counterpart.

Our results rely on Assumptions~\ref{assu:beig}--\ref{assu:comm} and the following regularity conditions:
\\
\noindent   {\bf Assumption B}
\begin{enumerate}[label={\textbf{(B\arabic*})}]
\item\label{assu:eigp:bin:pop} 
There exist constants $ \tau_2, \tau_4, \tau_5 \in (0,\infty)$ such that $\Lambda_{\max}(\bar{\g{\Theta}}) < \tau_2$, and  
\begin{equation*}
\Lambda_{\min}( \bar{\m{H}}_{\mc{B}\mc{B}}) \geq \tau_4~~~\text{and}~~~\Lambda_{\max}(\m{T}) \leq \tau_5.     
\end{equation*}

\item \label{assu:incp:bin:pop} There exists a constant $\delta\in (0,1]$, such that 
\begin{equation}
  \|\bar{\m{H}}_{\mc{B}^c\mc{B}} \left(\bar{\m{H}}_{\mc{B}\mc{B}}\right)^{-1}\|_{\infty} \leq (1 - \delta).  
\end{equation}
\end{enumerate}

It should be mentioned that Assumptions~\ref{assu:eigp:bin:pop}-\ref{assu:incp:bin:pop} are similar to those listed in~\cite{ravikumar2010high}. Under these assumptions and \ref{assu:beig}--\ref{assu:comm}, we have the following result:
\begin{thm}\label{thm:bin}
Suppose Assumptions~\ref{assu:beig}--\ref{assu:comm} and  \ref{assu:eigp:bin:pop}--\ref{assu:incp:bin:pop} are satisfied. Assume further that $\rho_{1n} =O(\sqrt{\log p/n})$, $n>O(q^3\log p)$ as $n \rightarrow \infty$, $ \rho_{2n} =O(\sqrt{\log(p-H+1)/n})$, $\rho_{2n} \leq \delta\rho_{1n}/(4(2-\delta) \tau_2 \tau_3)$, and $\g{\epsilon}=\g{0}$ { in \eqref{eqn:def:AB}}. Then, there exist finite constants $\check{C}(\bar{\g{\theta}}^o)$, $\check{D}(\bar{\m{q}}^o)$, and $\eta$ such that the following events hold with probability at least $1 - O(\exp(-\eta \rho_{1n}^2n))$:
 \begin{itemize}
\item There exists a local minimizer $(\wh{\g{\theta}}^o, \wh{\m{q}}^o)$ of \eqref{eqn:fair:ising} such that
\begin{align}\label{eqn:finalb:bin}
\nonumber 
&\max\left(\| \widehat{\g{\theta}}^o- \bar{\g{\theta}}^o \|, \|\wh{\m{q}}^o- \bar{\m{q}}^0\| \right) \\
& \qquad \leq \max \left( \check{C}(\bar{\g{\theta}}^o) \rho_{1n} \sqrt{q},  \check{D}(\bar{\m{q}}^o) \rho_{2n} \sqrt{\Psi(p,H,K)}\right),
\end{align}    
where $q$ and $\Psi(p,H,K)$ are defined in \eqref{eqn:rate:quant}.
\item  If $\min_{(i,j) \in \mc{B}} \bar{\theta}_{ij}^o \geq 2 \check{C}(\bar{\g{\theta}}^o) \sqrt{q }\rho_{1n}$, then $\wh{\g{\theta}}^o_{\mc{B}^c} =0$. 
  \end{itemize}
\end{thm}

Theorem \ref{thm:bin} gives sufficient conditions on the quadruple $(n, p,H, K)$ and the model parameters for the FBN to succeed in consistently estimating the neighborhood of every node in the graph  and communities simultaneously. In the case when $H=1$ (no fairness) and $K=p$ (no clustering), we recover the results of \cite{ravikumar2010high}.

\subsection{Consistency of Fair Community Detection in Graphical Models}\label{sec:cons:clus}

In this section, we establish that the proposed algorithm can recover the fair ground-truth community structure within a graph. Let \(\wh{\m{V}}\) and \(\bar{\m{V}}\) denote the orthonormal eigenvector matrices corresponding to the \(K\) largest eigenvalues of \(\wh{\m{Q}}\) and \(\bar{\m{Q}}\), respectively. According to \cite[Lemma 2.1]{lei2015consistency}, if two rows in the matrix \(\bar{\m{V}}\) are identical, their corresponding nodes belong to the same community. Therefore, our objective is to demonstrate that, after applying an orthogonal transformation, the rows of \(\wh{\m{V}}\) are close to those of \(\bar{\m{V}}\), enabling the use of K-means clustering on the rows of \(\wh{\m{V}}\) to identify community memberships. Specifically, we adopt a K-means clustering formulation inspired by \citep{lei2015consistency}, defined as:
\begin{equation}\label{eq:k-means} 
(\wh{\m{U}},\wh{\m{O}}) = \argmin_{\m{U},\m{O}} \|\m{U} \m{O} - \wh{\m{V}}\|_{\textnormal{F}}^2, ~~ \text{subject to}~~ \m{U} \in\mathbb{M}_{p,K}, ~~\m{O}\in \mb{R}^{K\times K},
\end{equation}
where \(\mathbb{M}_{p,K}\) is the set of \(p \times K\) matrices with rows containing a single \(1\) to indicate the fair community assignment of a node, while all other entries in the row are set to \(0\), as a node belongs to only one community.

Solving the optimization problem in \eqref{eq:k-means} to global optimality is known to be NP-hard \citep{aloise2009np}. However, polynomial-time algorithms \citep{kumar2004simple} exist to compute approximate solutions, providing a pair \((\wh{\m{U}},\wh{\m{O}}) \in \mathbb{M}_{p,K} \times \mb{R}^{K\times K}\) such that:
\begin{equation}\label{eq:k-means:relati} 
\|\wh{\m{U}} \wh{\m{O}} - \wh{\m{V}}\|_{\textnormal{F}}^2 \leq (1+\xi) \argmin_{(\m{U},\m{O}) \in \mathbb{M}_{p,K} \times \mb{R}^{K\times K} } \|\m{U} \m{O} - \wh{\m{V}}\|_{\textnormal{F}}^2,
\end{equation}
where \(\xi > 0\) is a small approximation factor.

Next, following the methodology of \citep[Theorem 3.1]{lei2015consistency}, we analyze the errors incurred during \((1+\xi)\)-approximate K-means clustering on the rows of \(\widehat{\m{V}}\) for estimating the community memberships. Let \(\mc{E}_k\) denote the set of nodes misclassified from the \(k\)-th community. Define \(\bar{\mc{C}} = \cup_{k\in[K]} ({\mc{C}_k}\backslash  {\mc{E}_k})\) as the set of nodes correctly classified across all communities, and let \(\bar{\m{V}}_{\bar{\mc{C}}}\) denote the submatrix of \(\bar{\m{V}}\) containing only the rows indexed by \(\bar{\mc{C}}\). The following theorem establishes bounds on the sizes of misclassified nodes for each community under specific conditions involving \(n\), \(p\), \(H\), and \(K\).

\begin{thm}\label{thm:consist:commu}
Let \(\wh{\m{U}}\) be the output of \((1+\xi)\)-approximate K-means clustering given in \eqref{eq:k-means:relati}. If 
\[
(2 + \xi) \Psi(p,H,K) \sqrt{\frac{K}{n}} < \pi
\]
for some constant \(\pi > 0\), then there exist subsets \(\mc{E}_k \subset \mc{C}_k\) for \(k=1,\ldots,K\), and a permutation matrix \(\g{\Phi}\) such that \(\wh{\m{V}}_{\bar{\mc{C}}} \g{\Phi}=\bar{\m{V}}_{\bar{\mc{C}}}\), and
\[
\sum_{k=1}^K \frac{|\mc{E}_k|}{|\mc{C}_k|} \leq \frac{(2 + \xi) \Psi(p,H,K) \sqrt{\frac{K}{n}}}{\pi}
\]
with probability tending to 1.
\end{thm}

%% file: sections/sec_exp_simu.tex
\section{Simulation Study}\label{sec:synth}
\subsection{Tuning Parameter Selection}\label{Sec:tuning parameter}

We consider a \emph{Bayesian information criterion}~(BIC)-type quantity for tuning parameter selection in \eqref{loss:fair}. Recall from Section~\ref{sec:model} that objective function~\eqref{loss:fair} decomposes the parameter of interest into $(\g{\Theta}, \m{Q})$ and places $\ell_1$ and trace penalties on $\g{\Theta}$ and $\m{Q}$, respectively. Specifically, for the graphical Lasso, i.e., problem in~\eqref{loss:fair} with $\rho_2 =0$, \cite{yuan2006model} proposed to select the tuning parameter $\rho_1$  such that $\hat{\g{\Theta}}$ minimizes the following quantity:
\[
n \left(-\log \det (\hat{\g{\Theta}}) +  \text{trace}(\g{S}\hat{\g{\Theta}})\right) + \log (n) \cdot |\hat{\g{\Theta}}|.
\]
Here, $|\hat{\g{\Theta}}|$ is the cardinality of $\hat{\g{\Theta}}$, i.e., the number of unique non-zeros in $\hat{\g{\Theta}}$. 
Note that $\rho_1$ controls the sparsity of the inferred graph and that the dependency of the criterion comes through the estimator $\hat{\g{\Theta}}$.
 
Using a similar idea, we consider minimizing the following BIC-type criteria for selecting the set of tuning parameters $(\rho_1, \rho_2)$ for \eqref{loss:fair}:
\begin{align}\label{eqn:biq}
\nonumber 
\mathrm{BIC} (\hat{\g{\Theta}},\hat{\m{Q}}) := \sum_{k=1}^K & n_k\left(-\log |\hat{\g{\Theta}}_k^2| +\text{trace}\left(\left(\m{S}_k+ c\hat{\m{Q}}_k\right)\hat{\g{\Theta}}_k^2\right) \right)\\
&+ \log (n_k) \cdot |\hat{\g{\Theta}}_k|,    
\end{align}
where $\hat{\g{\Theta}}_{k}$ is the $k$-th estimated inverse covariance matrix.
   
\noindent The constant $c$ controls the balance between graph estimation and fair community detection. A small $c$ emphasizes  generic graphical model estimation, while a larger $c$   enhances the influence of fair community detection. Since the optimal value of $c$ may depend on the scale of the sample covariance matrix, particularly when the data are not centered or scaled, we employ a grid search strategy over $c \in \{0.1,0.25,0.45,0.65,0.85\}$ jointly with $(\rho_1,\rho_2)$ in the simulation study. In the empirical real-data experiments below, results are reported with $c=0.25$ to ensure consistency and comparability.

Throughout, other parameters of the algorithm are set to $\gamma=1e-2$ and $\nu=1e-4$.

We define several measures of performance that will be used to numerically compare the various methods. To assess the clustering performance, we compute the clustering error (CE) criteria and Ratio Cut (RCut)~\citep{von2007tutorial}. CE calculates the distance between an estimated community assignment $\hat{z}_i$ and the true assignment $z_i$ of the $i$th node:
\begin{equation*}
\textnormal{CE}:=\frac{1}{p}\left\vert\left\{(i,j): \mathbf{1}(\hat{z}_{i}=\hat{z}_j)\neq \mathbf{1}(z_{i}=z_j),i<j \right\}\right\vert.
\end{equation*}
For a clustering $\mc{V}=\mc{C}_1 {\cup}\ldots {\cup}\mc{C}_K$, we have
\begin{align}\label{def_ratio_cut}
\text{RCut}:=\sum_{k=1}^K \frac{\text{Cut}(\mc{C}_k,\mc{V}\setminus \mc{C}_k)}{|\mc{C}_k|}, ~~~\text{Cut}(\mc{C}_k,\mc{V}\setminus \mc{C}_k):=\sum_{i\in \mc{C}_k, j\in \mc{V}\setminus \mc{C}_k}\theta_{ij}.
\end{align}
The F1 score for community detection measures the harmonic mean of precision and recall for identifying true community memberships:
\begin{subequations}\label{metric:f1}
\begin{align}
\textnormal{Precision} &:= \frac{\sum_{i<j} \mathbf{1}(\hat{z}_i = \hat{z}_j \text{ and } z_i = z_j)}{\sum_{i<j} \mathbf{1}(\hat{z}_i = \hat{z}_j)}, \\[0.5em]
\textnormal{Recall} &:= \frac{\sum_{i<j} \mathbf{1}(\hat{z}_i = \hat{z}_j \text{ and } z_i = z_j)}{\sum_{i<j} \mathbf{1}(z_i = z_j)}, \\[0.5em]
\textnormal{F1} &:= \frac{2 \cdot \textnormal{Precision} \cdot \textnormal{Recall}}{\textnormal{Precision} + \textnormal{Recall}}.
\end{align}
\end{subequations}

To measure the  graph and parameter estimation quality, we calculate the proportion of correctly estimated edges (PCEE) \cite{tarzanagh2018estimation} and sum of squared error (SSE):
\begin{align}\label{metric:PCEE}
\textnormal{PCEE} &:= \frac{\sum_{j'< j} \mathbf{1}_{\{ |\hat{\theta}_{jj'}| > 10^{-5} \text{ and } |\theta_{jj'}| \ne 0 \}}}{\sum_{j' <j} \mathbf{1}_{\{ |\theta_{jj'}| \ne 0 \}}}, \\
\label{metric:SSE}
\textnormal{SSE} &:=   \sum_{j' < j} (\hat{\theta}_{jj'} - \theta_{jj'})^2.
\end{align}
Finally, we use \textit{balance} as a fairness metric to reflect the distribution of fair clustering \cite{chierichetti2017fair}. Let $\mc{N}_i = \{j : r_{ij} = 1\}$ be the set of neighbors of node $i$ in $\m{R}$. For a set of communities $\{\mc{C}_k\}_{k=1}^K$, the balance coefficient is defined as
\begin{equation}\label{eqn:balance}
\textnormal{Balance}:=\frac{1}{p} \sum_{i = 1}^p \tau_i  ~~~\textnormal{where}~~~\tau_i=\min_{ k, \ell \in [K]} \frac{|\mc{C}_{k}\cap \mc{N}_i|}{|\mc{C}_{\ell} \cap \mc{N}_i |}.
\end{equation}
Note that $0\leq \tau_i \leq 1$, and a large $\tau_i$ indicates that node $i$ has adequate representation in all communities. In particular, balance is used to quantify how well the selected edges can eliminate discrimination---the selected edges are considered fairer if they can lead to a balanced community structure that preserves proportions of protected attributes.

\subsection{Data Generation}\label{GGM:datagenerate}
In order to demonstrate the performance of the proposed algorithms, we create several synthetic datasets based on a special random graph with community and group structures. Then the baseline and proposed algorithms are used to recover graphs (i.e., graph-based models) from the artificially generated data. To create a
dataset, we first construct a graph, then its {  associated precision matrix, $\g{\Theta}$}, is used to generate independent data samples from the distribution $N(0, \g{\Theta}^{\dagger})$ where $\dagger$ denotes pseudoinverse. A graph (i.e., $ \g{\Theta}$) is constructed in two steps.

\textbf{In the first step}, we determine the graph structure based on the random modular graph also known as stochastic block model (SBM)~\cite{holland1983stochastic,lei2015consistency}.  The stochastic block model~\cite{holland1983stochastic} is a generative model for random graphs with planted blocks (ground-truth clustering). It is widely used to generate synthetic networks containing communities, subsets of nodes characterized by being connected with one another with particular edge densities~\cite{lei2015consistency}. In an SBM, each of $p$ vertices is assigned to one of $K$ blocks/clusters to prescribe a clustering, and edges are placed between vertex pairs with probabilities dependent only on the block membership of the vertices. SBM takes
the following parameters: 
\begin{itemize}
\item   The number $p$ of vertices;
\item A partition of the vertex set $\mc{V}=\{1, \ldots, p\}$ into communities  $\mc{C}_1,\ldots, \mc{C}_K$; and
\item A symmetric matrix $\mb{P} \in \mb{R}^{p \times p}$ of edge probabilities.  
\end{itemize}
The edge set is then sampled at random as follows: any two vertices $s \in \mc{C}_i$ and $ u \in \mc{C}_j$ are connected by an edge with probability $\mb{P}_{ij}$. More precisely, the SBM takes, as input, a function $\pi_c: [p] \rightarrow [K]$ that assigns each vertex $i \in \mc{V}$ to one of the $K$ clusters. Then, independently, for all node pairs $(i, j)$ such that $i > j$, $\mb{P}(a_{ij} = 1) = b_{\pi_c(i) \pi_c(j)}$, where $\m{B} \in [0, 1]^{K \times K}$ is a symmetric matrix. Each $b_{k\ell}$ specifies the probability of a connection between two nodes that belong to clusters $\mc{C}_k$ and $\mc{C}_\ell$, respectively. A commonly used variant of SBM assumes $b_{kk} = \xi_2$ and $b_{k\ell} = \xi_1$ for all $k, \ell \in [K]$ such that $k \neq \ell$: 
\begin{equation*}
        \mb{P}(a_{ij} = 1) = \begin{cases}
          \zeta_2 & \text{if } \pi_c(i) = \pi_c(j), \\
          \zeta_1 & \text{if } \pi_c(i) \neq \pi_c(j).
        \end{cases}
\end{equation*}

To take protected groups into account, we use a modified SBM \cite{kleindessner2019guarantees}. Let $\pi_d: [p] \rightarrow [H]$ be a function that assigns each vertex $i \in \mc{V}$ to one of the $H$ protected groups. We consider a variant of SBM with the following probabilities: 
    \begin{equation}\label{eq:sbm_specification}
        \mb{P}(a_{ij} = 1) = \begin{cases}
          \zeta_4 & \text{if } \pi_c(i) = \pi_c(j)~\text{ and }~  \pi_d(i) = \pi_d(j) , \\
          \zeta_3 & \text{if } \pi_c(i) \neq \pi_c(j)~\text{ and }~\pi_d(i) = \pi_d(j), \\
          \zeta_2 & \text{if } \pi_c(i) = \pi_c(j)~\text{ and }~\pi_d(i) \neq \pi_d(j),  \\
          \zeta_1 & \text{if } \pi_c(i) \neq \pi_c(j)~\text{ and }~\pi_d(i) \neq \pi_d(j).
        \end{cases}
      \end{equation}
Here, $1 \geq \zeta_{i+1} \geq \zeta_i \geq 0$ are probabilities used for sampling edges. In our implementation, we set  $\zeta_{i} = 0.1i$ for all $i=1, \ldots, 4$. We note that when vertices $i$ and $j$ belong to the same community, they have a higher probability of connection between them for a fixed value of $\pi_d$; see,  \cite{kleindessner2019guarantees} for further discussions.

\textbf{In the second step}, the graph weights (i.e., node and edge weights) are randomly selected based on a uniform distribution from the interval $[0.1, 3]$ and the { associated  precision matrix $\g{\Theta}$} is constructed. Finally, given the graph precision matrix $\g{\Theta}$, we generate the data matrix $\m{Y}$  according to $\m{y}_1,\ldots,\m{y}_n \stackrel{\mathrm{i.i.d.}} \sim N(\g{0}, \g{\Theta}^{\dagger})$. 

\begin{exm}\label{exam:fair:sbm}
Let $p = 10$,  $H = 2$,  $\mc{D}_1=\{1,6,7,8,9\}$, and $\mc{D}_2=\{2,3,4,5,10\}$. This gives the group membership matrix as follows:
\begin{align*}
\m{R}= \begin{bmatrix}
1 & 0 & 0 & 0 & 0 & 1 & 1 & 1&1&0 \\ 
0 & 1 & 1 & 1 & 1 & 0 & 0 & 0 &0&1 \\ 
0 & 1 & 1 & 1 & 1 & 0 & 0 & 0 &0&1\\  
0 & 1 & 1 & 1 & 1 & 0 & 0 & 0 &0&1\\ 
0 & 1 & 1 & 1 & 1 & 0 & 0 & 0 &0&1\\ 
1 & 0 & 0 & 0 & 0 & 1 & 1 & 1 &1&0\\  
1 & 0 & 0 & 0 & 0 & 1 & 1 & 1 &1&0\\
1 & 0 & 0 & 0 & 0 & 1 & 1 & 1 & 1 & 0 \\
1 & 0 & 0& 0 & 0 & 1 & 1 & 1 & 1 & 0\\
0 & 1 & 1 & 1 & 1 & 0 & 0 & 0 & 0 & 1
    \end{bmatrix} 
    .
\end{align*}
Set $K= 3$. Define $\mc{C}_1'=\{1,2,3,4\}$, $\mc{C}_2'=\{5,6,7,8\}$, $\mc{C}_3'=\{9,10\}$,   $\mc{C}_1=\{1,3,4,6\}$, $\mc{C}_2=\{2,5,7,8\}$, $\mc{C}_3'=\{9,10\}$.     $\m{Q}'$ and   $\m{Q}$  provide the membership matrices associated with clusterings $\mc{C}_1' \cup  \mc{C}_2' \cup \mc{C}_3'$ and $\mc{C}_1 \cup  \mc{C}_2 \cup \mc{C}_3$, respectively, as follow:
\begin{align*}
    \m{Q}'&= \begin{bmatrix}
\frac{1}{4} & \frac{1}{4} & \frac{1}{4} & \frac{1}{4} & 0 & 0 & 0 & 0 &0&0\\
\frac{1}{4} & \frac{1}{4} & \frac{1}{4} & \frac{1}{4} & 0 & 0 & 0 & 0 &0&0\\
\frac{1}{4} & \frac{1}{4} & \frac{1}{4} & \frac{1}{4} & 0 & 0 & 0 & 0 &0&0 \\
\frac{1}{4} & \frac{1}{4} & \frac{1}{4} & \frac{1}{4} & 0 & 0 & 0 & 0 &0&0\\
0 & 0 & 0 & 0 & \frac{1}{4} & \frac{1}{4} & \frac{1}{4} & \frac{1}{4} &0&0 \\
0 & 0 & 0 & 0 & \frac{1}{4} & \frac{1}{4} & \frac{1}{4} & \frac{1}{4} &0&0\\
0 & 0 & 0 & 0 & \frac{1}{4} & \frac{1}{4} & \frac{1}{4} & \frac{1}{4} &0&0\\
0 & 0 & 0 & 0 & \frac{1}{4} & \frac{1}{4}& \frac{1}{4} & \frac{1}{4} &0&0\\
0 & 0 & 0 & 0 & 0 & 0& 0 & 0 &\frac{1}{2}&\frac{1}{2}\\
0 & 0 & 0 & 0 & 0 & 0& 0 & 0 &\frac{1}{2}&\frac{1}{2}\\
    \end{bmatrix},
    \\
~~~~\m{Q}&= \begin{bmatrix}
\frac{1}{4} & 0 & \frac{1}{4} & \frac{1}{4} & 0 & \frac{1}{4} & 0 & 0 & 0 & 0 \\
0 & \frac{1}{4} & 0 & 0 & \frac{1}{4} & 0 & \frac{1}{4} & 1 & 0 & 0 \\
\frac{1}{4} & 0 & \frac{1}{4} & \frac{1}{4} & 0 & \frac{1}{4} & 0 & 0 & 0 & 0 \\
\frac{1}{4} & 0 & \frac{1}{4} & \frac{1}{4} & 0 & \frac{1}{4} & 0 & 0 & 0 & 0 \\
0 & \frac{1}{4} & 0 & 0 & \frac{1}{4} & 0 & \frac{1}{4} & \frac{1}{4} & 0 & 0 \\
\frac{1}{4} & 0 & \frac{1}{4} & \frac{1}{4} & 0 & \frac{1}{4} & 0 & 0 & 0 & 0 \\
0 & \frac{1}{4} & 0 & 0 & \frac{1}{4} & 0 & \frac{1}{4} & \frac{1}{4} & 0 & 0 \\
0 & \frac{1}{4} & 0 & 0 & \frac{1}{4} & 0 & \frac{1}{4} & \frac{1}{4} & 0 & 0 \\
0 & 0 & 0 & 0 & 0 & 0 & 0 & 0 & \frac{1}{2} & \frac{1}{2} \\
0 & 0 & 0 & 0 & 0 & 0 & 0 & 0 & \frac{1}{2} & \frac{1}{2}
    \end{bmatrix}.
     \qquad  \qquad  \qquad \qquad 
\end{align*}
For each $ h \in \{1,2\}$, we get:
\begin{align*}
 \frac{|\mc{D}_h\cap \mc{C}_k'|}{|\mc{C}_k'|} &\neq \frac{|\mc{D}_h|}{p} =\frac{1}{2}~~ \text{for some}~~k \in \{1, 2,3\}~~~\text{and}~~~\m{R}(\m{I} - \m{J}_p / p)\m{Q}' \neq \m{0}, \\
 \frac{|\mc{D}_h\cap \mc{C}_k|}{|\mc{C}_k|}&=  \frac{|\mc{D}_h|}{p}= \frac{1}{2}~~ \text{for all}~~k \in \{1, 2,3\}~~~\text{and}~~~\m{R}(\m{I} - \m{J}_p / p)\m{Q} = \m{0}.
\end{align*}
Hence, $\m{Q}$ is a fair partition matrix (clustering). 
Figures~\ref{fig:SBM-example} (left) and (right) show the graphical models generated with probabilities \eqref{eq:sbm_specification} and using the membership matrices $\m{Q}'$ and $\m{Q}$, respectively.
\end{exm}

\subsection{Comparison to Community Detection Methods in the Known Graph Setting }\label{sec:algs:comm:set}
In this section, we consider the graph structure to be known. Specifically, we assume that the graph matrix $\g{\Theta}$ is constructed following the procedure outlined in the first step of data generation, as detailed in Section \ref{GGM:datagenerate}.

To evaluate our methods, we compare them against the following baselines in the context of convex community detection:
\begin{enumerate}[label=\hspace{1cm}CD-\Roman*., wide, labelindent=-26pt, itemsep=0pt]
 \item \label{cd:1} Two-stage approach for which we (ii) apply a community detection approach \cite{amini2018semidefinite} to compute partition matrix $\hat{\m{Q}}$, and (ii) employ a K-means clustering to obtain clusters.
 \item \label{cd:2} Two-stage approach for which we (ii) apply a community detection approach \cite{cai2015robust} to compute partition matrix $\hat{\m{Q}}$, and (ii) employ a K-means clustering to obtain clusters.
 \end{enumerate}
FCD. Two-stage approach for which we (ii) apply the fair community detection approach in \eqref{eqn:obj:know:q} to compute partition matrix $\hat{\m{Q}}$, and (ii) employ a K-means clustering to obtain clusters.
 
\begin{figure}[t]
    \begin{center} 
    \includegraphics[width=0.49\textwidth]{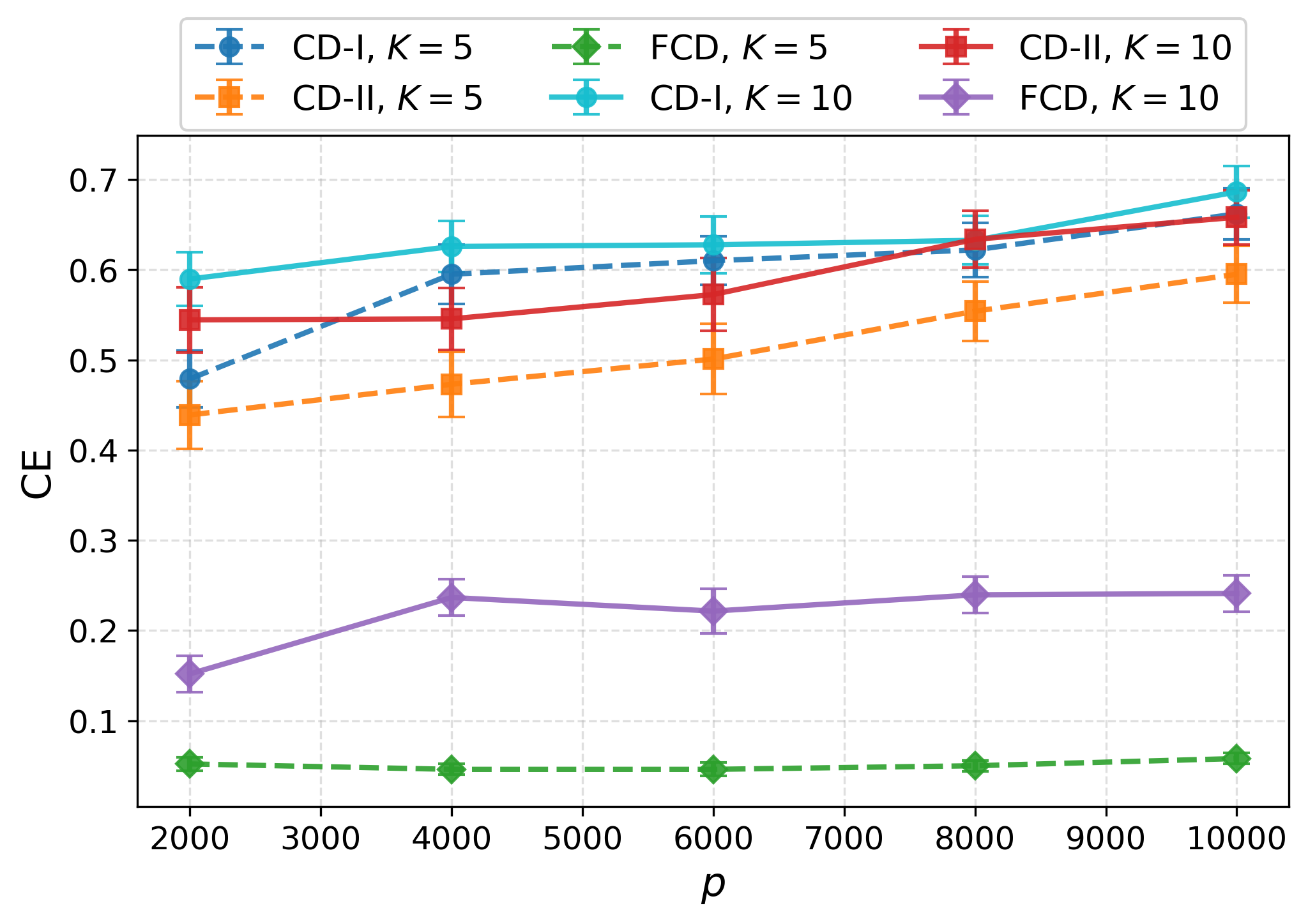}
    \includegraphics[width=0.49\textwidth]{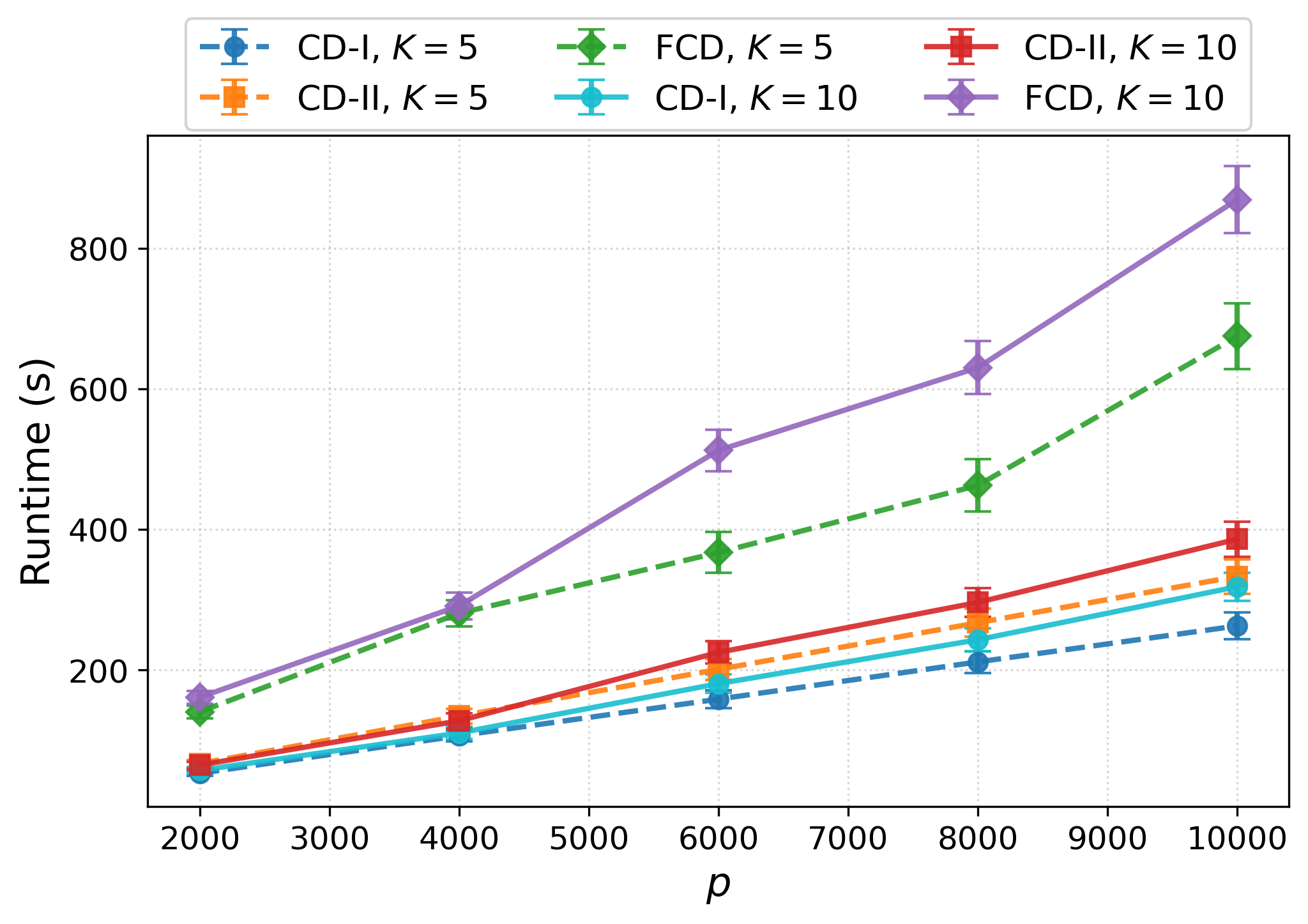} 
    \end{center} 
\caption{CE and runtime (in seconds) of CD-I~\cite{amini2018semidefinite}, CD-II~\cite{cai2015robust}, and the proposed FCD on a stochastic block model (SBM) with $H = 5$ and $K \in \{5, 10\}$.}
\label{fig:SBM-cd-example}
\end{figure}

Figure \ref{fig:SBM-cd-example} presents the CE and runtime (in seconds) obtained by repeating the procedure 100 times for the CD-I \cite{amini2018semidefinite}, CD-II \cite{cai2015robust}, and the proposed FCD methods on an SBM with $H = 5$ and $K \in \{5, 10\}$. We observe that the FCD method demonstrates significantly lower clustering error compared to both CD-I and CD-II. Furthermore, CD-I and CD-II yield comparable results in terms of clustering error, with CD-II showing marginally better performance. Regarding computational efficiency, the runtime of all three methods scales linearly with $p$. Notably, despite incorporating an additional fairness constraint, the proposed FCD maintains comparable runtime to CD-I and CD-II for smaller values of $p$, demonstrating that the fairness constraint does not significantly compromise computational efficiency.

\subsection{Comparison to Graphical Lasso and Neighbourhood Selection Methods in the unknown Graph Setting }\label{sec:algs:set}
We consider four setups for comparing our methods with community-based graphical models (GM):
\begin{enumerate}[label=\Roman*.] 
 \item \label{eqn:3sta} Three-stage approach for which we (i) use a GM to estimate precision matrix $\hat{\g{\Theta}}$, (ii) apply a community detection approach \cite{cai2015robust} to compute partition matrix $\hat{\m{Q}}$, and (iii) employ a K-means clustering to obtain clusters.
  \item  \label{eqn:2sta} Two-stage approach for which we (i) use \eqref{eqn:loss:relaxfair} without fairness constraint to simultaneously estimate precision and partition matrices and (ii) employ a K-means clustering to obtain clusters.
 \end{enumerate}
 \begin{enumerate}[label=F\Roman*.]  
  \item \label{eqn:3staf} Three-stage approach for which we (i) use a GM to estimate precision matrix $\hat{\g{\Theta}}$, (ii) apply a community detection approach~\cite{cai2015robust}  to compute partition matrix $\hat{\m{Q}}$, and (iii) employ a fair K-means clustering~\cite{chierichetti2017fair} to obtain clusters.
 \item  \label{eqn:2staf} Two-stage approach for which we (i) use \eqref{eqn:loss:relaxfair} to simultaneously estimate precision and partition matrices and (ii) employ the K-means clustering to obtain clusters.
\end{enumerate}
The main goal of Setups~\ref{eqn:3sta} and \ref{eqn:2sta} is to compare the community detection errors without fairness constraints under different settings of $L$ and $G$ functions. 

We consider three type of GMs in Setups~\ref{eqn:3sta}--\ref{eqn:2staf}:
\begin{enumerate}[label=\Alph*.]
\item \label{ref:t:glasso} A graphical Lasso-type method \cite{Friedman07} implemented using  the objective with 
\begin{align*}
L( \g{\Theta};\m{Y}) =n/2[-\log \det (\g{\Theta}) + \trace(\m{S}\g{\Theta})] \quad  \textnormal{and} \quad G(\g{\Theta})=\g{\Theta}.     
\end{align*}
\item \label{ref:t:concord} A neighborhood selection-type method \cite{khare2015convex} implemented using $L (\g{\Theta}; \m{Y}) = n/2[-\log |\text{diag}(\g{\Theta}^2) | + \text{trace}(\m{S}\g{\Theta}^2)]$ and $G(\g{\Theta})=\g{\Theta}^2$. 
\item \label{ref:t:hbn} A neighborhood selection-type method \cite{ravikumar2010high} implemented using $L (\g{\Theta}; \m{Y}) =-\sum_{j=1}^{p}\sum_{j'=1}^{p} \theta_{jj'} s_{jj'}+ 1/n\sum_{i=1}^{n}\sum_{j=1}^{p} \log \big(1 + \exp (\theta_{jj}+ \sum_{j'\neq j}\theta_{jj'}y_{ij'})\big)$ and  $G(\g{\Theta})=\g{\Theta}$. 
\end{enumerate}
%

In tables below, the method column contains labels of the form ``GM-Type-Setup'' to refer to the GM type and the setup above. For example, GM-\ref{ref:t:glasso}\ref{eqn:2sta} refers to a graphical Lasso used in the first step of the two stage approach in \ref{eqn:2sta}
It is worth mentioning that GM-\ref{ref:t:glasso}\ref{eqn:2sta} and GM-\ref{ref:t:concord}\ref{eqn:3staf} can be seen as variants of the cluster-based GLASSO \cite{pircalabelu2020community,kumar2020unified,hosseini2016learning} and fair K-means applied to spectral clustering~\cite{kleindessner2019guarantees}, respectively.  Note that unlike \cite{kleindessner2019guarantees}, which assume that the graph structure and the number of communities are given in advance, GM-\ref{ref:t:concord}\ref{eqn:3staf} learns fair community structure while estimating heterogeneous GMs.

\begin{table}[!ht]
\renewcommand{\arraystretch}{1}
\centering
\resizebox{\textwidth}{!}{
\begin{tabular}{||l|l|c|c|c|c||}
\hline
\textbf{\# Sample}  & \textbf{Method} & \textbf{CE} & \textbf{PCEE} & \textbf{F1} &  \textbf{SSE} \\
\hline\hline
$n=300$ 
& GM-\ref{ref:t:concord}\ref{eqn:3sta}   & 0.458(0.006) & 0.771(0.004) & 0.42(0.02) & 15.4(0.3) \\
& GM-\ref{ref:t:concord}\ref{eqn:2sta}   & 0.421(0.005) & 0.789(0.004) & 0.47(0.02) & 14.1(0.3) \\
& GM-\ref{ref:t:concord}\ref{eqn:3staf}  & 0.193(0.003) & 0.771(0.004) & 0.79(0.02) & 10.3(0.2) \\
& GM-\ref{ref:t:concord}\ref{eqn:2staf}\,(\textbf{FCONCORD}) & \textbf{0.057(0.003)} & \textbf{0.833(0.004)} & \textbf{0.93(0.01)} & \textbf{9.1(0.2)} \\
\hline\hline
$n=450$ 
& GM-\ref{ref:t:concord}\ref{eqn:3sta}   & 0.418(0.006) & 0.811(0.004) & 0.48(0.02) & 13.2(0.3) \\
& GM-\ref{ref:t:concord}\ref{eqn:2sta}   & 0.413(0.005) & 0.861(0.004) & 0.49(0.02) & 12.5(0.3) \\
& GM-\ref{ref:t:concord}\ref{eqn:3staf}  & 0.159(0.006) & 0.811(0.004) & 0.82(0.02) & 9.2(0.2) \\
& GM-\ref{ref:t:concord}\ref{eqn:2staf}\,(\textbf{FCONCORD}) & \textbf{0.009(0.003)} & \textbf{0.889(0.005)} & \textbf{0.98(0.01)} & \textbf{8.1(0.1)} \\
\hline\hline
$n=1000$ 
& GM-\ref{ref:t:concord}\ref{eqn:3sta}   & 0.281(0.004) & 0.852(0.004) & 0.63(0.02) & 10.1(0.2) \\
& GM-\ref{ref:t:concord}\ref{eqn:2sta}   & 0.258(0.005) & 0.869(0.004) & 0.67(0.02) & 8.8(0.2) \\
& GM-\ref{ref:t:concord}\ref{eqn:3staf}  & 0.118(0.004) & 0.885(0.004) & 0.87(0.01) & 7.1(0.1) \\
& GM-\ref{ref:t:concord}\ref{eqn:2staf}\,(\textbf{FCONCORD}) & \textbf{0.039(0.003)} & \textbf{0.902(0.004)} & \textbf{0.96(0.01)} & \textbf{6.2(0.1)} \\
\hline
\end{tabular}
}
\caption{Simulation results of neighborhood selection-type GMs on SBM network. The proposed FCONCORD outperforms the other methods.  The results are for $p=600$, $H=3$, and $K=2$ averaged over 100 repetitions.} 
\label{tab:conti:1}
\end{table}

\begin{table}[!ht]
\renewcommand{\arraystretch}{1}
\centering
\resizebox{\textwidth}{!}{
\begin{tabular}{||l|l|c|c|c|c||}
\hline
\textbf{\# Sample} & \textbf{Method} & \textbf{CE} & \textbf{PCEE} & \textbf{F1} &   \textbf{SSE}\\
\hline\hline
$n=300$ 
& GM-\ref{ref:t:glasso}\ref{eqn:3sta}   & 0.426(0.006) & 0.729(0.005) & 0.44(0.02) & 16.6(0.3) \\
& GM-\ref{ref:t:glasso}\ref{eqn:2sta}   & 0.411(0.006) & 0.759(0.005) & 0.46(0.02) & 15.3(0.3) \\
& GM-\ref{ref:t:glasso}\ref{eqn:3staf}  & 0.178(0.003) & 0.729(0.005) & 0.80(0.02) & 11.4(0.2) \\
& GM-\ref{ref:t:glasso}\ref{eqn:2staf}\,(\textbf{FGLASSO}) & \textbf{0.093(0.003)} & \textbf{0.813(0.018)} & \textbf{0.89(0.02)} & \textbf{10.1(0.2)} \\
\hline\hline
$n=450$ 
& GM-\ref{ref:t:glasso}\ref{eqn:3sta}   & 0.465(0.005) & 0.790(0.004) & 0.41(0.02) & 14.4(0.3) \\
& GM-\ref{ref:t:glasso}\ref{eqn:2sta}   & 0.401(0.005) & 0.801(0.004) & 0.49(0.02) & 13.2(0.3) \\
& GM-\ref{ref:t:glasso}\ref{eqn:3staf}  & 0.159(0.003) & 0.790(0.004) & 0.82(0.02) & 9.8(0.2) \\
& GM-\ref{ref:t:glasso}\ref{eqn:2staf}\,(\textbf{FGLASSO}) & \textbf{0.046(0.003)} & \textbf{0.861(0.005)} & \textbf{0.94(0.01)} & \textbf{8.6(0.1)} \\
\hline\hline
$n=1000$ 
& GM-\ref{ref:t:glasso}\ref{eqn:3sta}   & 0.301(0.005) & 0.838(0.004) & 0.61(0.02) & 11.0(0.2) \\
& GM-\ref{ref:t:glasso}\ref{eqn:2sta}   & 0.284(0.005) & 0.855(0.004) & 0.64(0.02) & 9.7(0.2) \\
& GM-\ref{ref:t:glasso}\ref{eqn:3staf}  & 0.141(0.004) & 0.870(0.004) & 0.85(0.01) & 7.3(0.1) \\
& GM-\ref{ref:t:glasso}\ref{eqn:2staf}\,(\textbf{FGLASSO}) & \textbf{0.034(0.003)} & \textbf{0.898(0.005)} & \textbf{0.96(0.01)} & \textbf{6.4(0.1)} \\
\hline
\end{tabular}
}
\caption{Simulation results of graphical Lasso-type algorithms on SBM network. The proposed FGLASSO outperforms the other methods. The results are for $p=600$, $H=3$, and $K=2$ averaged over 100 repetitions.} 
\label{tab:conti:2}
\end{table}
\begin{table}[!ht]
\renewcommand{\arraystretch}{1}
\centering
\resizebox{\textwidth}{!}{
\begin{tabular}{||l|l|c|c|c|c||}
\hline
\textbf{Sample Size} & \textbf{Method} & \textbf{CE} & \textbf{PCEE} & \textbf{F1} & \textbf{SSE} \\
\hline\hline
$n=200$
& GM-\ref{ref:t:hbn}\ref{eqn:3sta}   & 0.484(0.005) & 0.599(0.004) & 0.38(0.02) & 18.4(0.3) \\
& GM-\ref{ref:t:hbn}\ref{eqn:2sta}   & 0.456(0.003) & 0.649(0.004) & 0.42(0.02) & 17.1(0.3) \\ 
& GM-\ref{ref:t:hbn}\ref{eqn:3staf}  & 0.215(0.003) & 0.599(0.004) & 0.76(0.02) & 12.9(0.2) \\ 
& GM-\ref{ref:t:hbn}\ref{eqn:2staf}\,(\textbf{FBN}) & \textbf{0.113(0.004)} & \textbf{0.736(0.003)} & \textbf{0.87(0.02)} & \textbf{8.0(0.2)} \\
\hline\hline
$n=400$
& GM-\ref{ref:t:hbn}\ref{eqn:3sta}   & 0.429(0.004) & 0.635(0.004) & 0.45(0.02) & 15.9(0.3) \\
& GM-\ref{ref:t:hbn}\ref{eqn:2sta}   & 0.456(0.004) & 0.679(0.004) & 0.42(0.02) & 14.4(0.3) \\ 
& GM-\ref{ref:t:hbn}\ref{eqn:3staf}  & 0.215(0.003) & 0.635(0.004) & 0.76(0.02) & 10.8(0.2) \\ 
& GM-\ref{ref:t:hbn}\ref{eqn:2staf}\,(\textbf{FBN}) & \textbf{0.100(0.003)} & \textbf{0.794(0.003)} & \textbf{0.88(0.02)} & \textbf{6.6(0.1)} \\
\hline\hline
$n=1000$
& GM-\ref{ref:t:hbn}\ref{eqn:3sta}   & 0.308(0.004) & 0.708(0.004) & 0.60(0.02) & 11.5(0.2) \\
& GM-\ref{ref:t:hbn}\ref{eqn:2sta}   & 0.284(0.004) & 0.745(0.004) & 0.64(0.02) & 10.0(0.2) \\ 
& GM-\ref{ref:t:hbn}\ref{eqn:3staf}  & 0.141(0.003) & 0.759(0.004) & 0.84(0.01) & 7.4(0.1) \\ 
& GM-\ref{ref:t:hbn}\ref{eqn:2staf}\,(\textbf{FBN}) & \textbf{0.059(0.003)} & \textbf{0.820(0.004)} & \textbf{0.93(0.01)} & \textbf{4.1(0.1)} \\
\hline
\end{tabular}
}
\caption{Simulation results of binary neighborhood selection-type GMs on SIBM network. The proposed FBN outperforms the other methods. The results are for $p=100$, $H=3$, and $K=2$ averaged over 100 repetitions.} 
\label{tab:isi:1}
\end{table}

To obtain the standard errors (shown in parentheses) in columns CE, PCEE, F1, and SSE, we repeated each experiment 100 times. Tables~\ref{tab:conti:1} and \ref{tab:conti:2} are for SBM with $p=600$. As shown in Tables~\ref{tab:conti:1} and \ref{tab:conti:2}, GM-\ref{ref:t:glasso}\ref{eqn:3sta}, GM-\ref{ref:t:glasso}\ref{eqn:2sta}, GM-\ref{ref:t:concord}\ref{eqn:3sta}, and GM-\ref{ref:t:concord}\ref{eqn:2sta} have the largest clustering error and the lowest proportion of correctly estimated edges. GM-\ref{ref:t:glasso}\ref{eqn:2sta} and GM-\ref{ref:t:concord}\ref{eqn:2sta} improve the performance of GM-\ref{ref:t:glasso}\ref{eqn:3sta} and GM-\ref{ref:t:concord}\ref{eqn:3sta} in the precision matrix estimates. However, they still incur a relatively large clustering error since they ignore the similarity across different community matrices and employ a standard K-means clustering. In contrast, our  proposed  FCONCORD and FGLASSO algorithms achieve the best clustering accuracy and estimation accuracy for both scenarios. This is due to our joint clustering and estimation strategy as well as the consideration of the fairness of precision matrices across clusters. This experiment shows that a satisfactory fair community detection algorithm is critical to achieve accurate estimations of heterogeneous and fair GMs, and alternatively good estimation of GMs can also improve the fair community detection performance. This explains the success of our joint method in terms of both fair clustering and GM estimation. 

Next, we consider a natural composition of SBM and the Ising model called Stochastic Ising Block Model (SIBM) \cite{bert16,ye2020exact} for more details. In SIBM, we take SBM similar to \eqref{eq:sbm_specification}, where $p$ vertices are divided into clusters and subgroups and the edges are connected independently with probability $\{\xi_i\}_{i=1}^4$. Then, we use the graph $\mc{G}$ generated by the SBM as the underlying graph of the Ising model and draw $n$ i.i.d. samples from it. The objective is to exactly recover the fair clusters in SIBM from the samples generated by the Ising model, without having access to the graph $\mc{G}$.

Table~\ref{tab:isi:1} reports the averaged clustering errors and the proportion of correctly estimated edges for SIBM with $p=100$. The standard GM-\ref{ref:t:hbn}\ref{eqn:3sta} method has the largest clustering error due to its ignorance of the network structure in the precision matrices. GM-\ref{ref:t:hbn}\ref{eqn:3staf} improves the clustering performance of the GM-\ref{ref:t:hbn}\ref{eqn:3sta} by using the method of \cite{ravikumar2010high} in the precision matrix estimation and the robust community detection approach~\cite{cai2015robust} for computing partition matrix $\hat{\m{Q}}$. GM-\ref{ref:t:hbn}\ref{eqn:2staf} is able to achieve the best clustering performance due to the procedure of joint fair clustering and heterogeneous GMs estimation.

%% file: sections/sec_exp_real.tex
\section{Real Data Application }\label{Sec:realdata}

{
\subsection{Application to Social Networks}\label{sec:exp:socialnet}

\begin{figure}[t]
    \begin{center} 
        \includegraphics[width=0.49\textwidth]{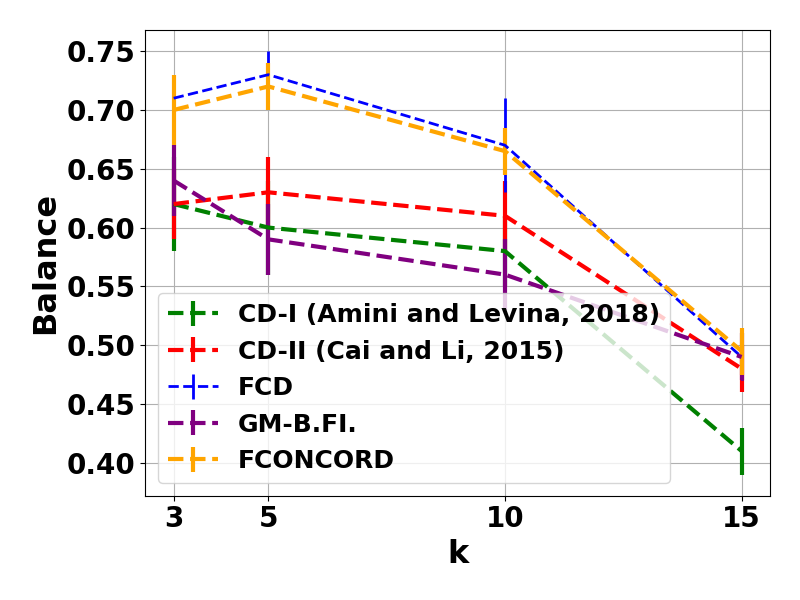}
        \includegraphics[width=0.49\textwidth]{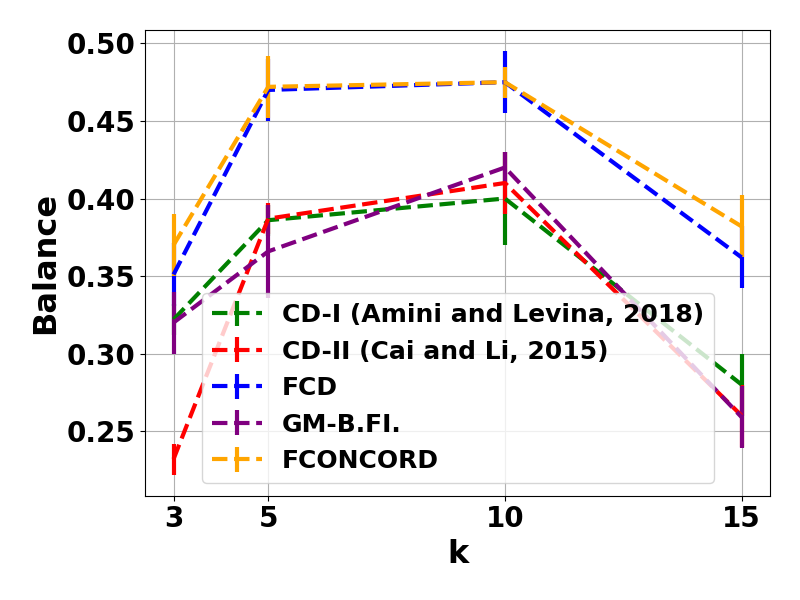}\\
        \includegraphics[width=0.49\textwidth]{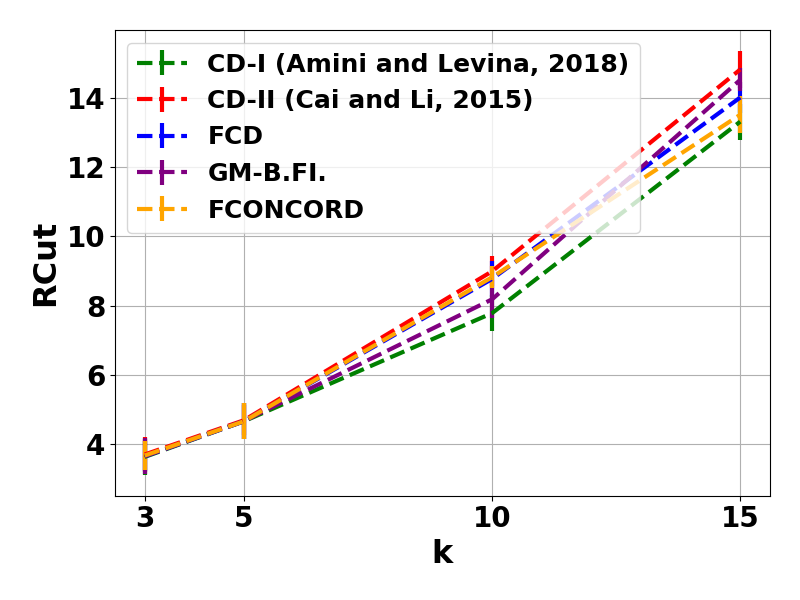}
        \includegraphics[width=0.49\textwidth]{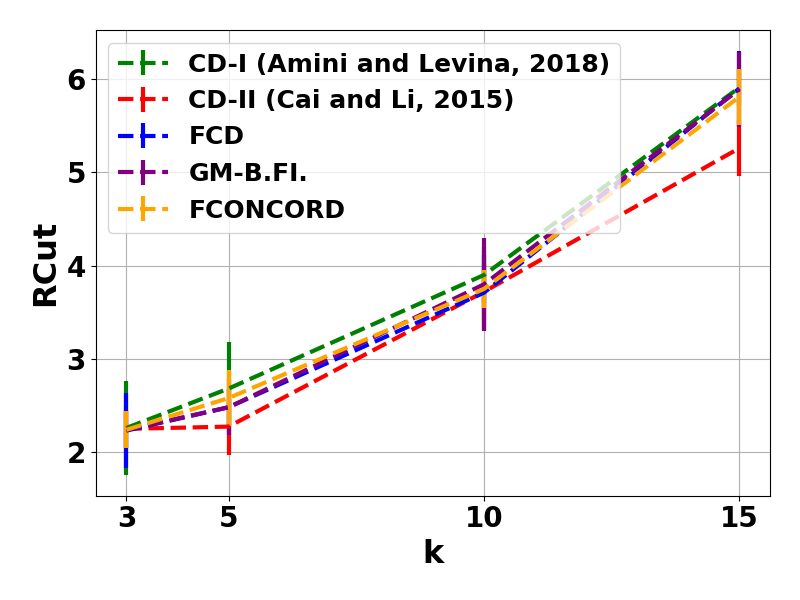}
    \end{center} 
    \caption{ { Balance and RCut of CD-I, CD-II \cite{cai2015robust}, FCD, GM-B.FI, and FCONCORD on Friendship (left) and DrugNet (right) Networks as a function of the number $k$ of clusters.}}
    \label{fig:balance-and-rcut-friendanddrug}
\end{figure}

In this section, we test the performance of our model on two commonly used benchmark datasets of community detection in social networks. The first dataset is a high school friendship network\footnote{\url{http://www.sociopatterns.org/datasets/}}. The dataset corresponds to the contacts and friendship relations between students in nine classes at a high school in Marseilles, France, over five days in December 2013. The dataset gives the contacts of the students of nine classes during five days in December 2013, as measured by the SocioPatterns infrastructure. It contains a tab-separated list representing the active contacts during 20-second intervals of the data collection. Each line has the form $[i, j, C_i, C_j]$, where $i$ and $j$ are the anonymous IDs of the persons in contact, $C_i$ and $C_j$ are their classes, and the interval during which this contact was active is $[t - 20s, t]$. Time is measured in seconds and expressed in UNIX ctime. The gender was considered the sensitive attribute. After data preprocessing, we obtain 127 students (nodes) split into male and female groups.

The second dataset, DrugNet, is a network encoding acquaintanceship between drug users in Hartford~\cite{weeks2002social}. After data preprocessing, we obtain 193 vertices. We use ethnicity as a sensitive attribute and split the vertices into three groups: African Americans, Latinos, and others. Note that the two datasets contain ground-truth graphs and no observed signals. One of the primary advantages of our GM model is that we can group observed data without real graph structures. Thus, we generate data similar to Section~\ref{GGM:datagenerate} based on the ground-truth networks~$\g{\Theta}$. We then use our model to group vertices via the observed data. For comparison, we apply CD-I~\cite{amini2018semidefinite}, CD-II~\cite{cai2015robust}, and FCD to the real networks to cluster vertices. Our goal is to demonstrate that our model, FCD, can achieve competitive fair clustering performance. Additionally, we show that FCONCORD can produce similar results even in the absence of real graphs. The quality of the fair clustering is evaluated using the "Balance" metric defined in~\eqref{eqn:balance}. We use RCut defined in~\eqref{def_ratio_cut} as the clustering separation accuracy metric. We let $N = 1000$. As displayed in Figure~\ref{fig:balance-and-rcut-friendanddrug}, for the two datasets, our model achieves almost the same RCut as CD-I~\cite{amini2018semidefinite} and CD-II~\cite{cai2015robust}—which are based on the ground-truth networks—even though we do not know the underlying graphs. However, compared to the state-of-the-art models CD-I and CD-II, our model can improve Balance by 20\% on average over $K$ on two benchmark datasets, meaning that our method can improve fairness in clustering at a moderate cost of RCut.

In this experiment, we assume the network matrix~$\Theta$ is given a priori. In the experiments of Figure~\ref{fig:balance-and-rcut-friendanddrug}, we evaluate the performance of standard semidefinite programming as in~\cite{cai2015robust} versus our fair versions on real network data.

The quality of a clustering is measured through its "Balance" defined in~\eqref{eqn:balance} and runtime. Figure~\ref{fig:balance-and-rcut-friendanddrug} shows the results as a function of the number of clusters~$k$ for two high school friendship networks~\cite{mastrandrea2015contact}. Vertices correspond to students and are split into two groups, males and females ($H=2$). The friendship network has 127 vertices, and an edge between two students indicates that one of them reported friendship with the other. The network corresponds to the largest connected component of the originally unconnected graph.

\begin{figure}[t]
    \begin{center} 
        \includegraphics[width=0.49\textwidth]{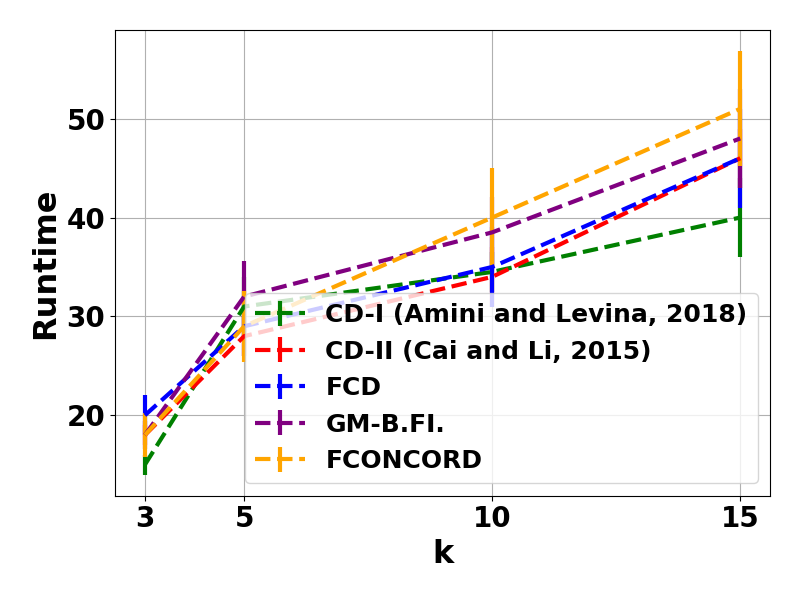}
        \includegraphics[width=0.49\textwidth]{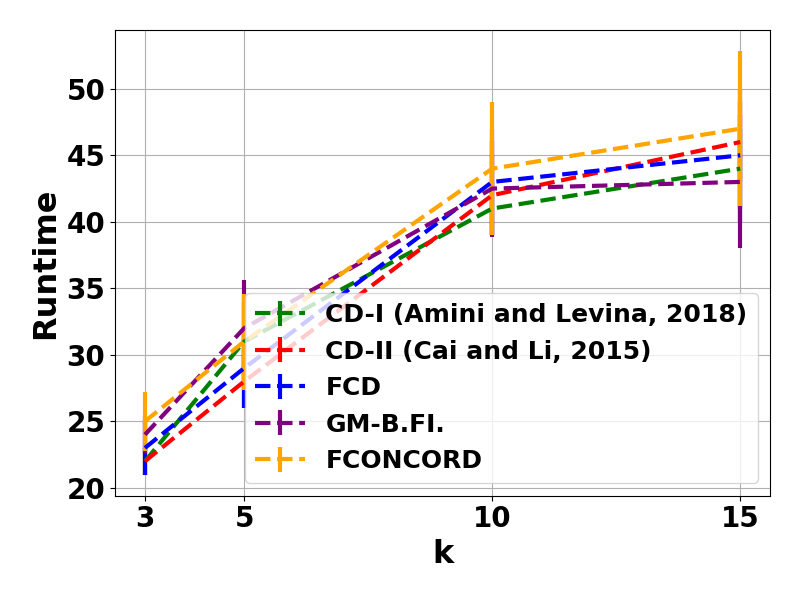}
    \end{center} 
    \caption{ { Runtime (in seconds) of CD-I, CD-II \cite{cai2015robust}, FCD, GM-B.FI, and FCONCORD on Friendship (left) and DrugNet (right) Networks as a function of the number $k$ of clusters.}}
    \label{fig:runtime:friend-and-drug}
\end{figure}

From Figure \ref{fig:runtime:friend-and-drug}, we can see that the runtime (in seconds) of CD-I, CD-II \cite{cai2015robust}, FCD, GM-B.FI, and FCONCORD on Friendship Network as a function of the number $k$ of clusters does not significantly increase despite FCONCORD having fairness constraints and needing to estimate graph and community structures.  In addition, despite incorporating an additional fairness constraint, the runtime of FCD is not significantly worse than  CD-I \cite{amini2018semidefinite} and CD-II \cite{cai2015robust}.

}

\subsection{Application to Natural Language Processing}\label{sec:wordembed}
Word embeddings are a form of word representation in an $p$-dimensional space. Word embeddings serve as a dictionary of sorts for computer programs that uses word meaning. The rationale behind word embeddings is that words with similar semantic meanings tend to have representations that are close together in the vector space than dissimilar words. This makes word embeddings useful by allowing them to represent the relationship between words in mathematical terms, making them an important and widely used component in many Natural Language Processing (NLP) models. However, despite its expansive applications, word embeddings have been found to exhibit strong, ethically questionable biases such as a bias against a specific race, gender, religion, and sexual preferences among others~\citep{bolukbasi2016man,caliskan2017semantics,petreski2022word}.

GMs are widely used in word embedding applications where the goal is to understand the relationships among the words/terms and their community structures~\citep{tan2014learning,harakawa2022trend}. In the following, we focus on the issues and the improvement of the GM approaches we have towards detecting and dealing with bias in word embeddings. To do so, we evaluated the bias based on data obtained from the preprocessed English dataset of the Small World of Words \citep{de2019small}. The goal of the analysis is to understand the relationships among the \textit{cue} terms, identify cue words that are hubs, and correct the bias in their community structure.  This dataset contains the responses of 83,864 participants in a word association task to 12,292 cue words. Each cue word was judged by exactly 100 participants; see Table~\ref{tab:sww}.

In our implementation, cue words and the responses from participants are considered random variables and the samples for the estimation of GMs, respectively. We apply a tokenizer to transform text data into numerical values (using \texttt{Word2Vec} algorithm). We select the 250 cue words for our numerical evaluations as follows: we first select 10 cue words uniformly at random and then select 25  words with the highest similarity to each cue word following \citep{de2019small}. We manually assign gender for each cue term. For example, the cue word \textquotedblleft woman\textquotedblright~(a node in GM) is female (group $\#1$), the cue word \textquotedblleft man\textquotedblright~is male (group $\#2$), and the cue word \textquotedblleft doctor\textquotedblright~can be both female and man (groups $\#1$ and 2). Please refer to Table~\ref{tab:sww} for further descriptions.

\begin{table}[t]
\centering
\begin{adjustbox}{max width=1\columnwidth}
\begin{tabular}{llllllllllll}
\hline
Participant & Age & Gender & Native & Country     & \textbf{Cue}    & R1    & R2        & R3        & \textbf{Cue} \\
ID &  &  & Language &        &  &       &         &         & \textbf{Group} ($h$) \\
\hline
 376           & 38  & Fe     & Australia      & Australia         & \textbf{strong} & box   & man       & weak      & \textbf{1 and 2}   \\
1062          & 19  & Ma     & Australia      & Australia         & \textbf{sick}   & ill   & hospital  & disease   & \textbf{1 and 2}   \\
4295          & 77  & Ma     & US  & US         & \textbf{doctor} & nurse & professor & physician & \textbf{1 and 2}   \\
4136          & 44  & Ma     & UK & Spain              & \textbf{mother} & woman & hair      & emotions  & \textbf{1}         \\
4214          & 63  & Fe     & US  & US         & \textbf{man}    & woman & boy       & child     & \textbf{2} \\        
\hline
\end{tabular}
\end{adjustbox}
\caption{ The data collected as part of the Small World of Words dataset. Each participant gave 3 responses (R1, R2, R3) to each cue word that they were presented with. The last column shows our manual group assignments. 
}\label{tab:sww}
\end{table}
%

\begin{table}[!ht]
\renewcommand{\arraystretch}{1.2}
\centering
\begin{tabular}{||l|l|c|c||}
\hline
 & Method & RCut & Balance \\ 
\hline\hline
\multirow{2}{*}{$K = 5$} 
    & GM-\ref{ref:t:concord}\ref{eqn:2sta} & 9.4 (0.1) & 0.324 (0.005) \\
    & \textbf{GM-\ref{ref:t:concord}\ref{eqn:2staf} (FCONCORD)} & \textbf{8.7 (0.1)} & \textbf{0.439 (0.005)} \\
\hline
\multirow{2}{*}{$K = 10$} 
    & GM-\ref{ref:t:concord}\ref{eqn:2sta} & 15.1 (0.1) & 0.507 (0.005) \\
    & \textbf{GM-\ref{ref:t:concord}\ref{eqn:2staf} (FCONCORD)} & \textbf{14.9 (0.1)} & \textbf{0.619 (0.005)} \\
\hline
\end{tabular}
\caption{The ratio cut and balance of various methods on the Word Embedding dataset. }
\label{tab:worddoc:real}
\end{table}

\begin{figure}[t]
\begin{center}
\includegraphics[scale=0.08]{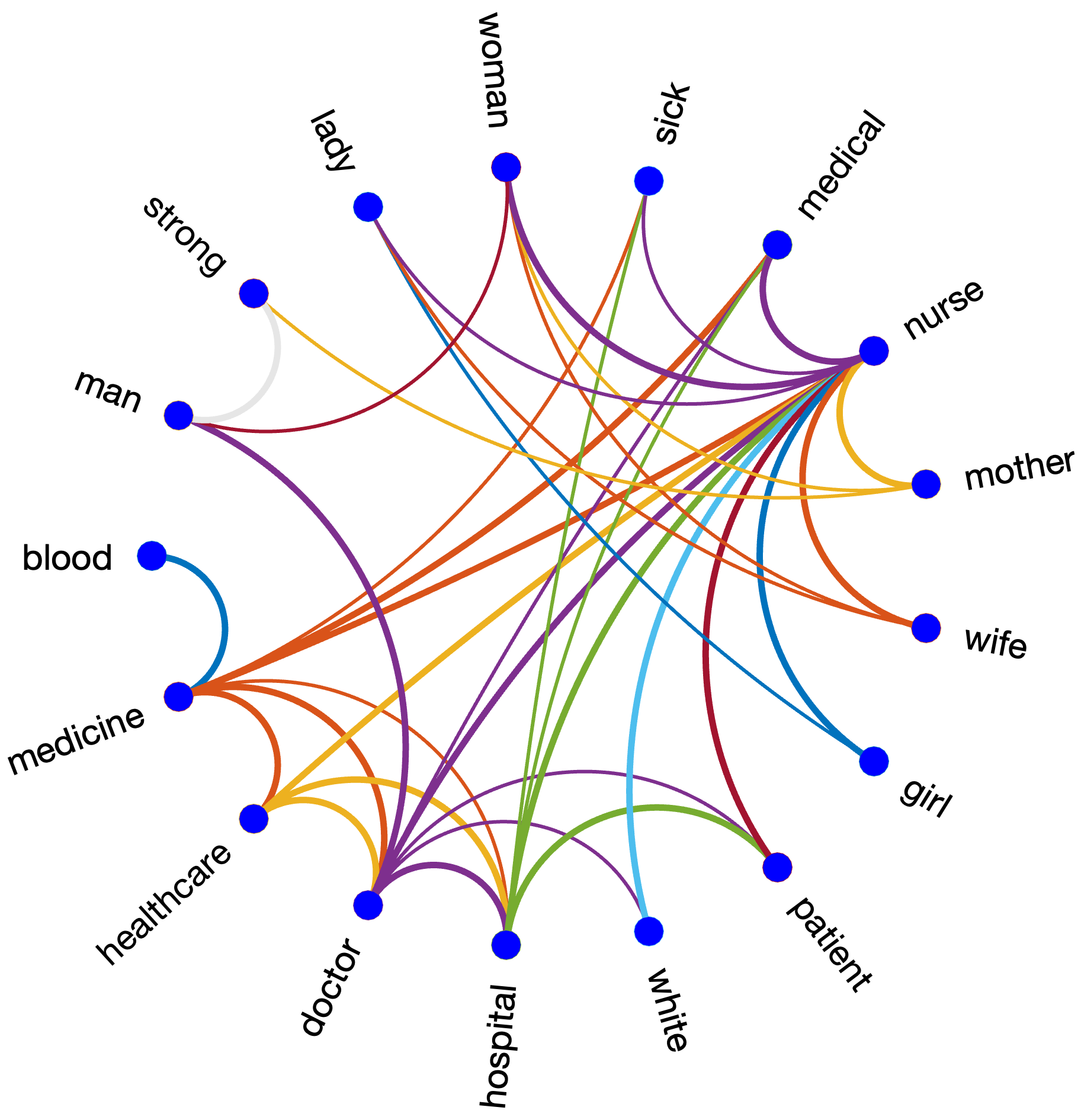}
\includegraphics[scale=0.08]{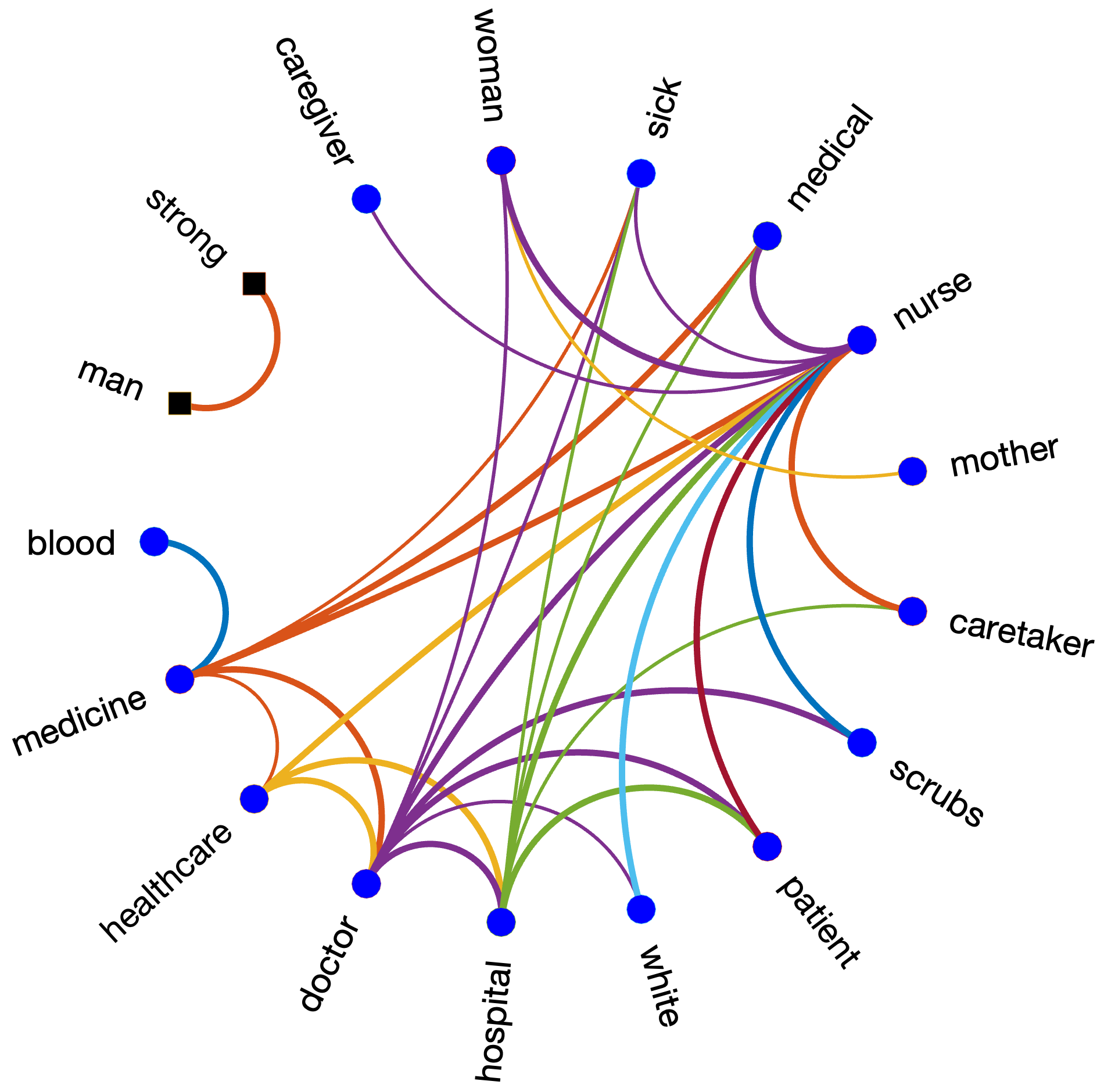}
 \end{center}
  \caption{Subgraphs of the precision matrices estimated by GM-\ref{ref:t:concord}\ref{eqn:2sta}(left) and GM-\ref{ref:t:concord}\ref{eqn:2staf}(right). Nodes represent the cue labeled by their word. Circle markers denote words within a healthcare community, and square markers denote isolated words. The width of a link is proportional to the magnitude of the corresponding partial correlations. GM-\ref{ref:t:concord}\ref{eqn:2staf}(FCONCORD) enhances the neutrality from gender bias and improves the related word connection within the healthcare community.
  }
\label{fig:net:word}
\end{figure}%
Despite the adoption of different graph learning algorithms for this dataset, the existing methods often do not have fairness considerations. Although one may manually remove the protected attributes in the selected features to avoid discrimination, a number of non-protected attributes that are highly correlated with the protected attributes may still be selected by these methods and result in discrimination. To address this issue, we next implement the proposed FCONCORD which uses a joint $\ell_1$--regularized regressions for each node (cue word) of a graph associated with the feature covariance matrix with tuning parameters selected using  \eqref{eqn:biq}. 

Table~\ref{tab:worddoc:real} compares various algorithms in terms of Ratio Cut (RCut) and Balance. Not only FCONCORD is able to achieve a higher balance, but it also does so with a better RCut as compared to GM-\ref{ref:t:concord}\ref{eqn:2sta}.  In addition, the estimated communities for two sub-graphs of cue words are also shown in Figure~\ref{fig:net:word}.  From both networks, we can see that the estimated communities mainly consist of healthcare communities. The estimated network also shows two \textit{hub} nodes: \textquotedblleft doctor \textquotedblright and  \textquotedblleft nurse \textquotedblright. For instance, the fact that the nurse is a hub indicates that many term occurrences (in the medical community) are explained by the occurrence of the word nurse. These results provide an intuitive explanation of the relationships among the terms in the cue word communities. We also note that FCONCORD enhances the neutrality from gender bias by removing the connection between biased edges such as (\textit{lady}, \textit{nurse}), (\textit{girl}, \textit{nurse}), and (\textit{doctor}, \textit{man}).   

\subsection{Application to Fair Community Detection in Recommender System}\label{sec:exp:recom}

Recommender systems (RS) model user-item interactions
to provide personalized item recommendations that will suit the user’s taste. Broadly speaking, two types of methods are used in such systems—content based and collaborative filtering. Content based approaches model interactions through user and item covariates. Collaborative filtering (CF), on the other hand, refers to a set of techniques that model user-item interactions based on user’s past response. 

A popular class of methods in RS is based on clustering users and items~\citep{ungar1998clustering,o1999clustering,sarwar2001item,schafer2007collaborative}. Indeed, it is more natural to model the users and the items using clusters (communities), where each cluster includes a set of like-minded users or the subset of items that they are interested in. The overall procedure of this method, called cluster CF (CCF), contains two main steps. First, it finds clusters of users and/or items, where each cluster includes a group of like-minded users or a set of items that these users are particularly interested in. Second, in each cluster, it applies traditional CF methods to learn users’ preferences over the items within this cluster. Despite efficiency and scalability of these methods, in many human-centric applications, using CCF in its original form can result in unfavorable and even harmful clustering and prediction outcomes towards some demographic groups in the data. 

It is shown in \cite{schafer2007collaborative,mnih2008probabilistic} that using item-item similarities based on “who-rated-what'' information is strongly correlated with how users explicitly rate items. Hence, using this information as user covariates helps in improving predictions for explicit ratings. Further, one can derive an item graph where edge weights represent movie similarities that are based on global “who-rated-what'' matrix~\citep{kouki2015hyper,wang2015collaborative,agarwal2011modeling,mazumder2011flexible}. Imposing sparsity on such a graph and finding its~\textit{fair} communities is attractive since it is intuitive that an item is generally related to only a few other items. This can be achieved through our fair GMs. Such a graph gives a fair neighborhood structure that can also help better predict explicit ratings. In addition to providing useful information to predict ratings, we note that using who-rated-what information also provides information to study the fair relationships among items based on user ratings.

The goal of our analysis is to understand the balance and prediction accuracy of fair GMs on RS datasets as well as the relationships among the items in these datasets. We compare the performance of our fair GMs implemented in the framework of standard CCF and its fair K-means variant. In particular, we consider the following algorithms:
\begin{itemize}
\item FGLASSO~(FCONCORD)+CF: A two-stage approach for which we first use FGLASSO (FCONCORD) to obtain the fair clusters and then apply traditional CF to learn users’ preferences over the items within each cluster. We set $\rho_1= 1$, $\rho_2 = 0.05$, $\gamma=0.01$, and $\epsilon=1e-3$ in our implementations.
\item CCF (Fair CCF): A two-stage approach for which we first use K-means (fair K-means \cite{chierichetti2017fair}) clustering to obtain the clusters and then apply CF to learn users’ preferences within each cluster~\citep{ungar1998clustering}.
\end{itemize}

\subsubsection{Music Data}\label{sec:exp:music}

Music RSs are designed to give personalized recommendations of songs, playlists, or artists to a user, thereby reflecting and further complementing individual users’ specific music preferences. Although accuracy metrics have been widely applied to evaluate recommendations in music RS literature, evaluating a user’s music utility from other impact-oriented perspectives, including their potential for discrimination, is still a novel evaluation practice in the music RS literature~\citep{epps2020artist,chen2020bias,shakespeare2020exploring}. Next, we center our attention on artists' gender bias for which we want to estimate if standard music RSs may exacerbate its impact. 
\begin{figure}[t]
\begin{center}
\includegraphics[scale=0.35]{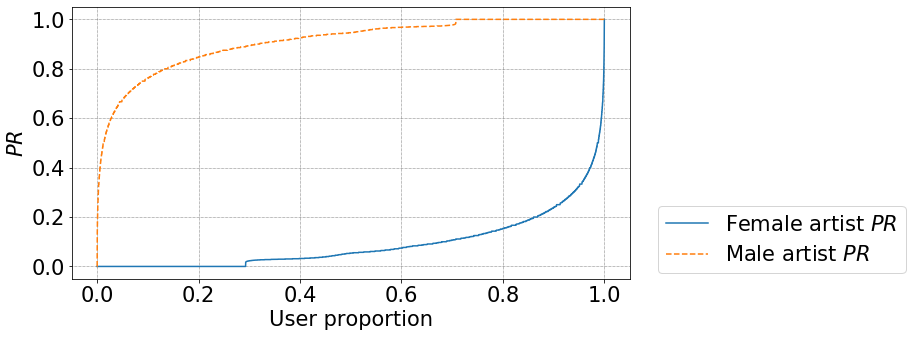}
 \end{center}
  \caption{Input Preference Ratio (PR) distributions of LFM-360K dataset.}\label{pr:lsf}
\end{figure}%

To illustrate the impact of artists gender bias in RSs, we use the freely available \textit{binary} LFM-360K music dataset\footnote{\url{http://www.last.fm}}. The LFM-360K consists of approximately 360,000 users listening histories from \texttt{Last.fm} collected during Fall 2008. We generate recommendations for a sample of all users for which gender can be identified. We limit the size of this sample to be 10\% randomly chosen of all male and female users in the whole dataset due to computational constraints. Let $\mc{U}$ be the set of $n$ users, $\mc{I}$ be the set of $p$ items and $\m{Y}$ be the $n \times p$ input matrix, where $y_{ui}=1$ if user $u$ has selected item $i$, and zero otherwise. Given the matrix $\m{Y}$, the input preference ratio (PR) for user group $\mc{D}$ on item category $\mc{C}$ is the fraction of liked items by group $\mc{D}$ in category $\mc{C}$,  defined as the following~\citep{shakespeare2020exploring}:
\begin{equation}\label{eqn:pr}
\textnormal{PR}(\mc{D},\mc{C}):= \frac{\sum_{u\in \mc{D}} \sum_{i\in \mc{C}} y_{ui}}{\sum_{u\in \mc{D}} \sum_{i\in \mc{I}} y_{ui}}.
\end{equation}

Figure~\ref{pr:lsf} represents the distributions of users’ input PR towards male and female artist groups. It shows that only around 20\% of users have a PR towards male artists lower than 0.8. On the contrary, 80\% of users have a PR lower than 0.2 towards female artists. This shows that commonly deployed state of the art CF algorithms may act to further increase or decrease artist gender bias in user-artist RS.  

\begin{figure}[t]
\begin{center}
\includegraphics[scale=0.3]{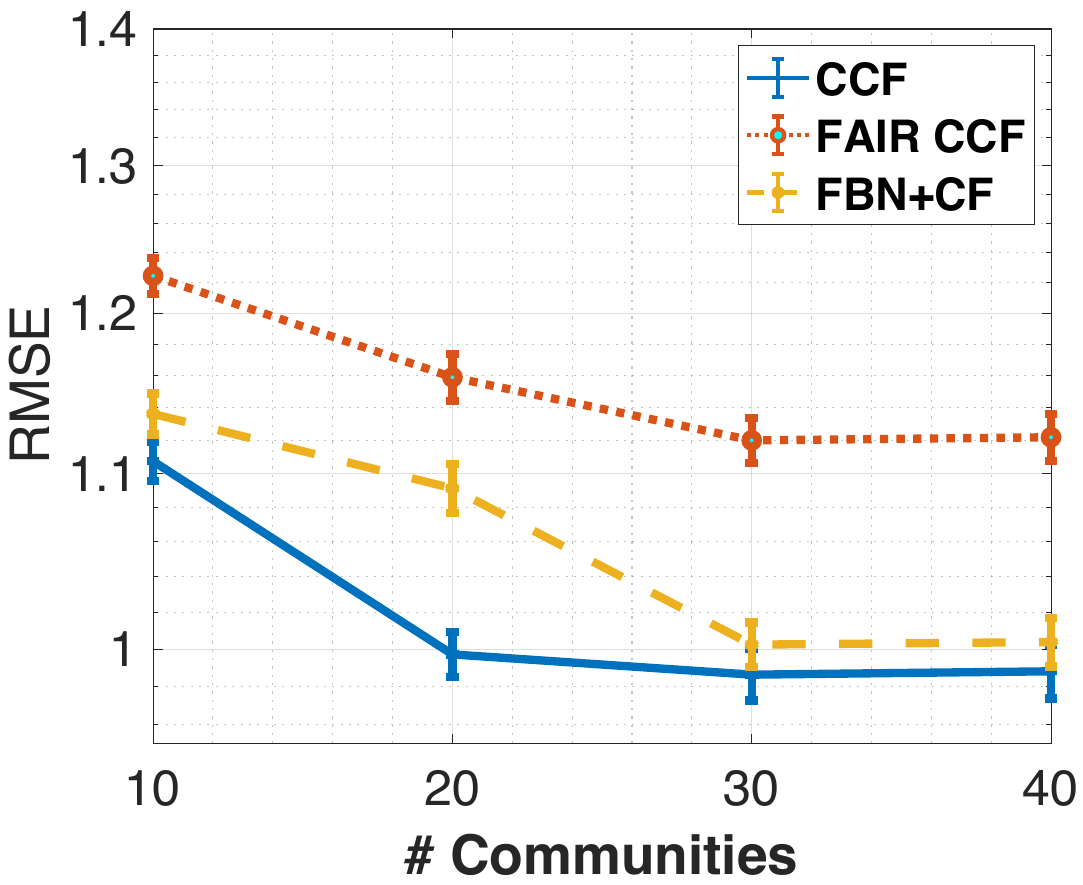}
~~~~~~~~
\includegraphics[scale=0.3]{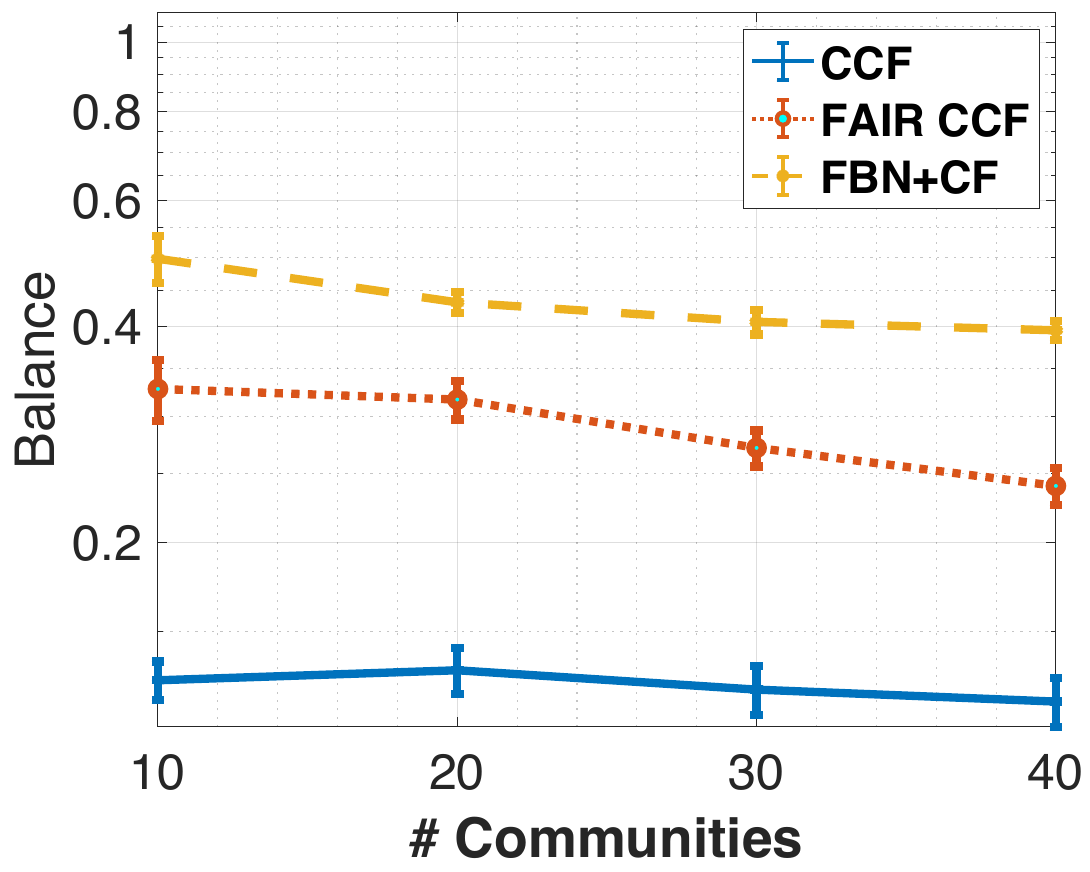}
 \end{center}
  \caption{RMSE (left) and Balance (right) of standard CCF, Fair CCF, and FBN+CF on LFM-360K music data.}\label{bal:rmese:lsf}
\end{figure}%

Next, we study the balance and prediction accuracy of fair GMs on music RSs. Figure~\ref{bal:rmese:lsf} indicates that the proposed FBN+CF has the best performance in terms of balance and root mean squared error (RMSE), where RMSE is defined as
\[
\text{RMSE} = \sqrt{\frac{1}{n}\sum_{i=1}^{n}(\hat{y}_i - y_i)^2},
\]
with $y_i$ denoting the true rating and $\hat{y}_i$ the predicted rating for user–item pair $i$. As expected, the baseline with no notion of fairness--CCF--results in the best overall precision. Of the two fairness-aware approaches, the fair K-means-based approach--Fair CCF--performs considerably below FBN+CF. This suggests that recommendation quality can be preserved, but leaves open the question of whether we can improve fairness.

Hence, we turn to the impact on the fairness of the three approaches. Figure~\ref{bal:rmese:lsf}(right) presents the balance. We can see that fairness-aware approaches--Fair CCF and FBN+CF--have a strong impact on the balance in comparison with standard CCF. And for RMSE, we see that FBN+CF achieves a much better rating difference in comparison with Fair CCF, indicating that we can induce aggregate statistics that are fair between the two sides of the sensitive attribute (male vs. female).

\subsubsection{MovieLens Data}\label{sec:exp:movie:RS}

We use the MovieLens 10K dataset\footnote{\url{http://www.grouplens.org}}.
Following previous works~\citep{koren2009collaborative,kamishima2012enhancement,chen2020bias}, we use  \textit{year} of the movie as a sensitive attribute and consider movies before 1991 as old movies. Those more recent are considered new movies. \cite{koren2009collaborative} showed that the older movies have a tendency to be rated higher, perhaps because only masterpieces have survived. When adopting~\textit{year} as a sensitive attribute, we show that our fair graph-based RS enhances the neutrality from this masterpiece bias.  
The clustering balance and RMSE have been used to evaluate different modeling methods on this dataset. Since reducing RMSE is the goal, statistical models assume the response (ratings) to be Gaussian for this data~\citep{kouki2015hyper,wang2015collaborative,agarwal2011modeling}. 

\begin{figure}[t]
\begin{center}
\includegraphics[scale=0.3]{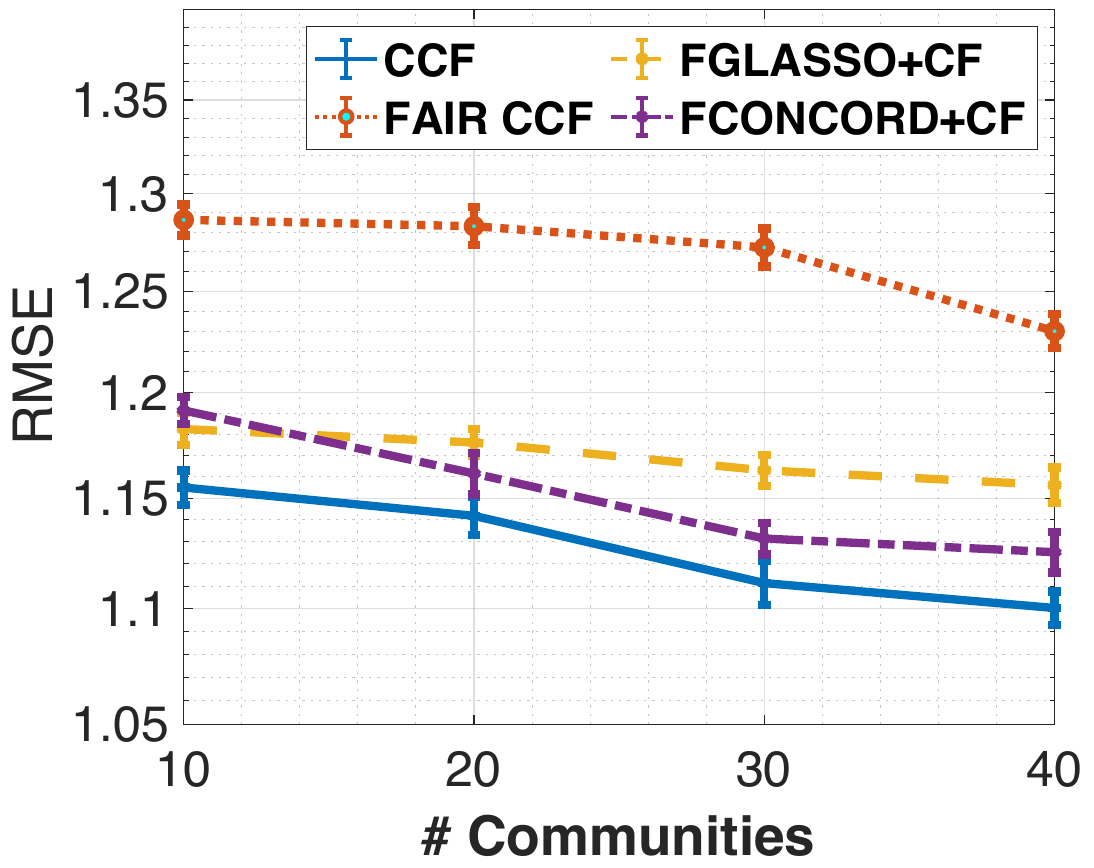}
~~~~~~~~
\includegraphics[scale=0.3]{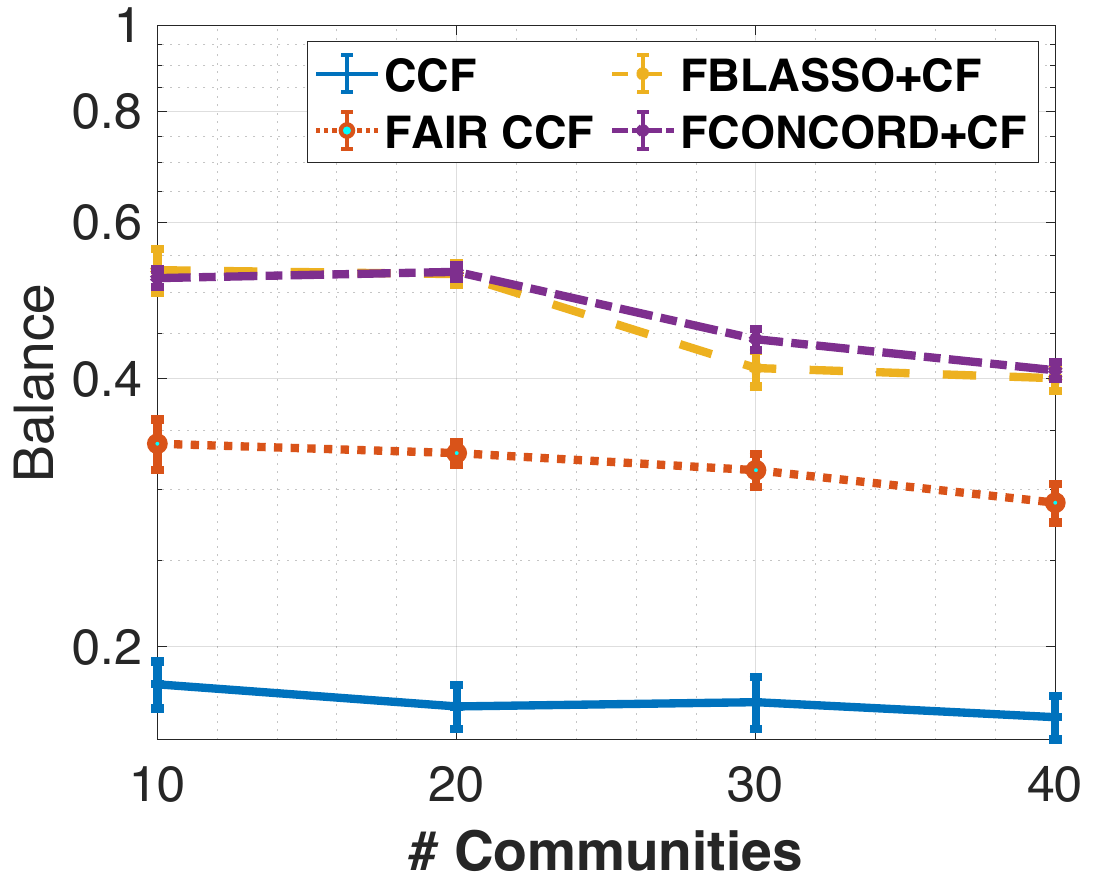}
 \end{center}
  \caption{RMSE (left) and Balance (right) of standard CCF, Fair CCF, FGLASSO+CF,  and FCONCORD+CF on MovieLens 10K data. 
   }\label{bal:rmese:mov}
\end{figure}%

Experimental results are shown in Figure~\ref{bal:rmese:mov}. As expected, the baseline with no notion of fairness--CCF--results in the best overall RMSEs, with our two approaches (FGLASSO+CF and FCONCORD+CF) providing performance fairly close to CCF. Figure~\ref{bal:rmese:mov} (right) shows that compared to fair CCF,  FGLASSO+CF and FCONCORD+CF significantly improve the clustering balance. Hence, our fair graph-based RSs successfully enhanced the neutrality without seriously sacrificing the prediction accuracy.
\begin{table}[!ht]
\centering
\begin{adjustbox}{max width=0.99\columnwidth}
\begin{tabular}{llr}
\hline
\hline
\multicolumn{2}{c}{GM-\ref{ref:t:concord}\ref{eqn:2sta}} \\
\cline{1-2}
\multicolumn{2}{c}{The pair of movies}  & Partial correlation \\
\hline
The Godfather (1972)       & The Godfather: Part II (1974)    & 0.592     \\
Grumpy Old Men (1993) & Grumpier Old Men (1995)  & 0.514      \\
Patriot Games (1992)        & Clear and Present Danger (1994)   & 0.484     \\
The Wrong Trousers (1993)    & A Close Shave (1995)  & 0.448      \\
Toy Story (1995)  & Toy Story 2 (1999)    &  0.431       \\
Star Wars: Episode IV--A New Hope (1977) &Star Wars: Episode V--The Empire Strikes Back (1980) & 0.415\\
\hline
\multicolumn{2}{c}{GM-\ref{ref:t:concord}\ref{eqn:2staf} (FCONCORD)} \\
\cline{1-2}
\multicolumn{2}{c}{The pair of movies}  & Partial correlation \\
\hline
The Godfather (1972)       & The Godfather: Part II (1974)    & 0.534   \\
Grumpy Old Men (1993) & Grumpier Old Men (1995)  & 0.520      \\
Austin Powers: International Man of Mystery (1997) & Austin Powers: The Spy Who Shagged Me (1999)& 0.491\\
Toy Story (1995)  & Toy Story 2 (1999)    &  0.475     \\
Patriot Games (1992)        & Clear and Present Danger (1994)   & 0.472     \\
The Wrong Trousers (1993)    & A Close Shave (1995)   & 0.453  \\
\hline
\hline
\end{tabular}
\end{adjustbox}
\caption{Pairs of movies with top 5 absolute values of partial correlations in the precision matrix from GM-\ref{ref:t:concord}\ref{eqn:2sta} and GM-\ref{ref:t:concord}\ref{eqn:2staf}(FCONCORD).}
\label{tab:mv:pc}
\end{table}

Fair GMs also provide information to study the relationships among items based on user ratings. To illustrate this, the top-5 movie pairs with the highest absolute values of partial correlations are shown in Table \ref{tab:mv:pc}. If we look for the highly related movies to a specific movie in the precision matrix, we find that FCONCORD enhances the balance  by assigning higher correlations to more recent movies such as “The Wrong Trousers'' (1993) and  “A Close Shave'' (1995).
\begin{figure}[!ht]
\begin{center}
\includegraphics[scale=0.22]{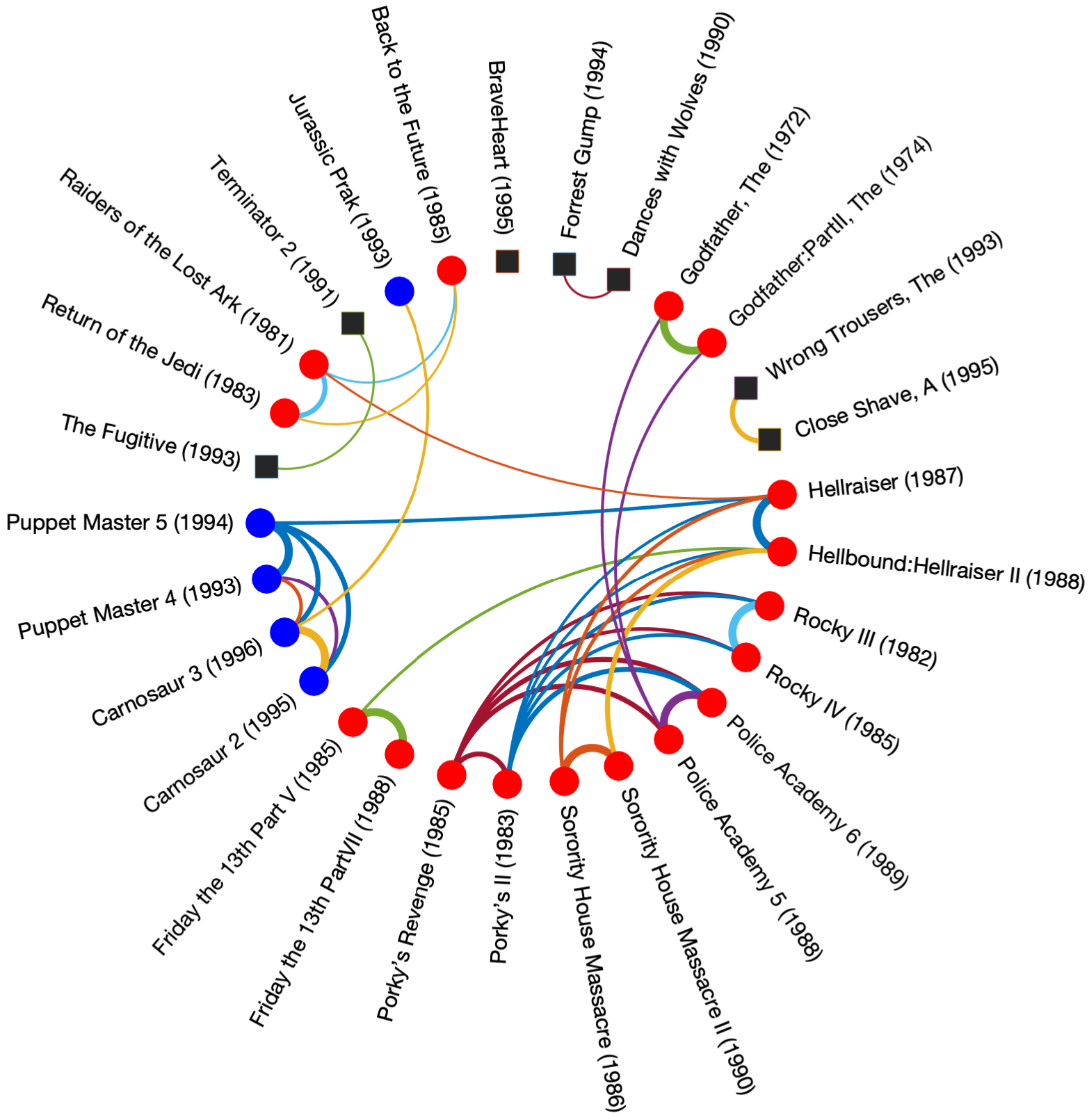}
\includegraphics[scale=0.22]{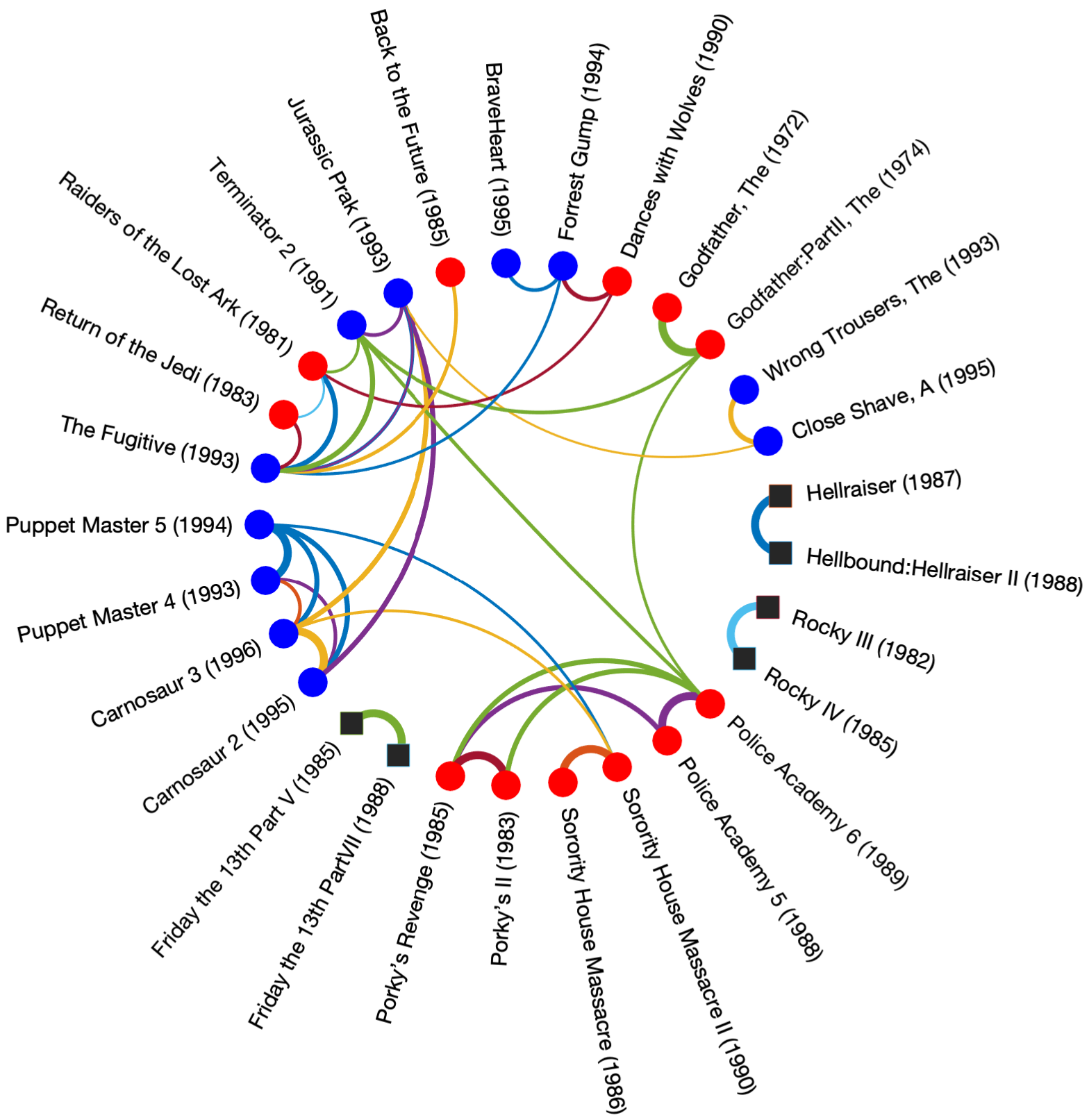}
 \end{center}
  \caption{Subgraphs of the precision matrices estimated by GM-\ref{ref:t:concord}\ref{eqn:2sta}(left) and GM-\ref{ref:t:concord}\ref{eqn:2staf}(right). Nodes represent the movies labeled by their titles. Circle markers denote movies within a single community in each subgraph, and square markers denote isolated movies. Blue nodes are new movies and red nodes old movies within each community. The width of a link is proportional to the magnitude of the corresponding partial correlations. GM-\ref{ref:t:concord}\ref{eqn:2staf}(FCONCORD) enhances the neutrality from year bias by replacing the old movies within each community by new ones.
  }
\label{fig:net:mov}
\end{figure}%

In addition, the estimated communities for two sub-graphs of movies are also shown in Figure~\ref{fig:net:mov}. From both networks, we can see that the estimated communities mainly consist of mass-marketed commercial movies, dominated by action films. Note that these movies are usually characterized by high production budgets, state-of-the-art visual effects, and famous directors and actors. Examples in this community include “The Godfather (1972) '', “Terminator 2 (1991) '', and “Return of the Jedi (1983)'', “Raiders of Lost Ark (1981)'', etc. As expected, movies within the same series are most strongly associated. Figure~\ref{fig:net:mov} (right) shows that FCONCORD enhances the neutrality from the old movies bias by replacing them with new ones such as “Jurassic Park (1993)'', “The Wrong Trousers(1993) '', and “A Close Shave (1995)''.

%% file: sections/sec_conc.tex
\section{Conclusion}\label{sec:conc}

In this work, we developed a novel approach to learning fair graphical models with community structure. Our goal is to motivate a new line of work for fair community learning in graphical models that can begin to alleviate fairness concerns in this important subtopic within unsupervised learning. We established statistical consistency of the proposed method for both a Gaussian GM and an Ising model proving that our method can recover the graphs and their fair communities with high probability. We applied the proposed framework to the tasks of estimating a Gaussian graphical model and a binary network. The proposed framework can also be applied to other types of graphical models, such as the Poisson graphical model~\citep{allen2012log} or the exponential family graphical model \citep{yang2012graphical}.

%% file: sections/sec_appendix.tex
\section*{Appendix}
The appendix is organized as follows:  
\begin{itemize}
\item Appendix~\ref{sec:app:pre} provides some preliminaries used in the proof of main theorems. 
\item Appendix~\ref{sec:app:conti} provides large sample properties of FCONCORD, i.e., the proof of Theorem~\ref{thm:three:fglasso}.
\item Appendix~\ref{sec:app:fbn} provides large sample properties of FBN, i.e., the proof of Theorem~\ref{thm:bin}. 
\item Appendix~\ref{sec:app:comm} gives the consistency of fair community labeling in graphical models. 
\item Appendix~\ref{sec:der:alg} provides the detailed derivation of the updates for Algorithm~\ref{Alg:general}.   
\end{itemize}

\subsection{Preliminaries}\label{sec:app:pre}
Let $F (\g{\Theta},\m{Q};\m{Y}) =n/2\big[-\log |\text{diag}(\g{\Theta})^2|+ \trace\big( (\m{S}+ \rho_2 \m{Q} ) \g{\Theta}^2\big) \big]$.  Let the notation ${F}_n(\g{\theta}^d, \g{\theta}^o, \m{q}^d, \m{q}^o; \m{Y})$ stands for $\frac{{F}}{n}$. We introduce a restricted version of criterion~\eqref{eqn:fgglasso} as below: 
\begin{equation}\label{eqn:local:estmin}
\minimize_{\g{\theta}^o, \g{q}^o}~~F_n(\bar{\g{\theta}}^d, \g{\theta}^o, \bar{\m{q}}^d, \m{q}^o;\m{Y})+ \rho_{1n} \|\g{\theta}^o\|_1, ~~~ \text{subj. to}~~ \g{\theta}^o_{\mc{B}^c}=0.
\end{equation} 

We define a linear operator  $\mc{A}: \mb{R}^{\binom{p}{2}} \rightarrow \mb{R}^{p \times p}$, $\m{w} \rightarrow \mc{A} \m{w}$, satisfying 
\begin{equation}\label{eqn:graph:oper}
[\mc{A}\m{w}]_{ij} =
\begin{cases}
w_{i+d_j}\; & \; \;i >j,\\
[\mc{A}\m{w}]_{ji} \; & \; \;i<j,\\
0\; &\;\; i=j,
\end{cases}
\end{equation}
where $d_j=-j+\frac{j-1}{2}(2p-j).$  An example for $\mc{A}\m{w} $ on  $\m{w}=[w_1,w_2,\cdots,w_6]^\top$ is given below
\begin{equation*}
\mc{A} \m{w}= \left[
\begin{array}{cccc}
	0 & w_1 & w_2 & w_3\\
	w_1 & 0 & w_4 & w_5\\
	w_2 & w_4 & 0& w_6\\
	w_3 & w_5 & w_6 & 0
\end{array}
\right].
\end{equation*}	
We derive the adjoint operator $\mc{A}^*$ of $\mc{A}$ by making $\mc{A}^*$ satisfy $\langle \mc{A} \m{w}, \m{z} \rangle=\langle \m{w},\mc{A}^*\m{z} \rangle$; see, \cite[
Section~4.1]{kumar2020unified} for more details.

Let $\m{Q}^o=\text{diag}(\bar{\m{q}}^d)+ \mc{A}\m{q}^o$ and $\g{\Theta}^o=\text{diag}(\bar{\g{\theta}}^d)+ \mc{A}\g{\theta}^o$. Since by our assumption $\g{\epsilon}=\m{0}$, we obtain
\begin{align}~\label{eqn:app:obj}
\begin{array}{ll}
\underset{\g{\theta}^o, \g{q}^o}{\text{minimize}} &
\begin{array}{c}
\hspace{.5cm} L_n(\bar{\g{\theta}}^d,\g{\theta}^o;\m{Y}) + \rho_{1n} \|\g{\theta}^o\|_{1} + \rho_{2n} \trace(\g{\Theta}^o \m{Q}^o \g{\Theta}^o),
\end{array}\\
\text{subj. to} &  
\begin{array}[t]{l}
\hspace{.5cm} \g{\theta}^o_{\mc{B}^c}=0,~~ \m{Q}^o \succeq \m{0}, ~~\m{A}_1 \m{Q}^o=\m{0}, ~~ \m{0} \leq \m{Q}^o\leq \m{J}_p.
\end{array}
\end{array}
\end{align}
It is easy to see that $\textnormal{rank}(\m{A}_1)=H-1$. Let $\m{N}\in \mb{R}^{p\times(p-H+1)}$ be a matrix whose rows form an orthonormal basis of the nullspace of $\m{A}_1$. We can substitute $\m{Q}^o=\m{N} \m{R}^o\m{N}^\top$ for $\m{R}^o\in \mb{R}^{(p-H+1) \times (p-H+1)}$, and then, using that $\m{N}^\top\m{N}=\m{I}_{(p-H+1)}$, Problem \eqref{eqn:app:obj} becomes
\begin{align}~\label{eqn:appnull:obj}
\begin{array}{ll}
\underset{\g{\theta}^o, \g{r}^o}{\text{minimize}} &
\begin{array}{c}
\hspace{.5cm} F_{n}(\bar{\g{\theta}}^d, \g{\theta}^o, \bar{\m{r}}^d, \m{r}^o;\m{Y})  +\rho_{1n} \|\g{\theta}^o\|_{1}
\end{array}\\
\text{subj. to} & 
\begin{array}[t]{l}
\hspace{.5cm} \g{\theta}^o_{\mc{B}^c}=0,~~\m{Q}^o \succeq \m{0},  ~~ \m{0} \leq  \m{Q}^o \leq \m{J}_p,
\end{array}
\end{array}
\end{align}
where  $F_{n}(\bar{\g{\theta}}^d, \g{\theta}^o, \bar{\m{r}}^d, \m{r}^o;\m{Y})   = L_n(\bar{\g{\theta}}^d,\g{\theta}^o;\m{Y}) + \rho_{2n} \trace( \g{\Theta}^o \m{N} \m{R}^o\m{N}^\top \g{\Theta}^o)$. 
%

Throughout, we use $\bar{\m{g}}^n$ and $\bar{\m{H}}^n$ to denote the gradient and the Hessian of $L_n(\bar{\g{\theta}}^d,\g{\theta}^o;\m{Y})$. We also define the population gradient and Hessian as follows:  For $1 \leq i<j \leq p$
$$
\bar{\m{g}}_{ij}:=
\mb{E}_{\bar{\g{\theta}}^o}\left(\frac{\partial
L(\bar{\g{\theta}}^d,\g{\theta}^o,\m{Y})}{\partial
\theta^o_{ij}}\Bigl|_{\g{\theta}^o=\bar{\g{\theta}}^o}\right),
$$
and for  $1 \leq i<j \leq p$ and  $1 \leq t<s \leq p$,
$$
\bar{\m{H}}_{ij,ts}:= \mb{E}_{\bar{\g{\theta}}^o} \Big(\frac{\partial^2 L({\bar{\g{\theta}}}^d, {\g{\theta}}^o;\m{Y})}{\partial \theta_{ij}^o \partial \theta_{ts}^o } \Big|_{\g{\theta}^o=\bar{\g{\theta}}^o} \Big).
$$

\subsection{Large Sample Properties of FCONCORD}\label{sec:app:conti}


We list some properties of the loss function.

\begin{lem}\citep{peng2009partial}\label{lem:bound:hess}
The following is true for the loss function:
\begin{enumerate}[label={\textnormal{(L\arabic*)}}]
    \item There exist constants $0 < M_1 \leq M_2  < \infty$ such that
   \begin{equation*}
        M_1(\bar{\g{\theta}}^o) \leq \Lambda_{\min}( \bar{\m{H}}) \leq \Lambda_{\max}(\bar{\m{H}}) \leq M_2(\bar{\g{\theta}}^o).
    \end{equation*}
    \item There exists a constant $M_3(\bar{\g{\theta}}^o)<\infty$ such that for all $1 \leq i < j \leq p$, $\bar{\m{H}}_{ij,ij}\leq M_3(\bar{\g{\theta}}^o)$.
    \item There exist constants $M_4(\bar{\g{\theta}}^o)$ and  $M_5(\bar{\g{\theta}}^o)<\infty$, such that for any $1 \leq i < j \leq p$
    \begin{equation*}
       \Var_{\bar{\g{\theta}}^o}(\bar{\m{g}}^n_{ij}) \leq M_4(\bar{\g{\theta}}^o), ~~~~~ \Var_{\bar{\g{\theta}}^o}(\bar{\m{H}}^n_{ij,ij}) \leq M_5(\bar{\g{\theta}}^o).
    \end{equation*}
    \item There exists a constant $0 <M_6(\bar{\g{\theta}}^o) <\infty$, such that for all $(i,j) \in \mc{B}$
   \begin{equation*}
      \bar{\m{H}}_{ij,ij} - \bar{\m{H}}_{ij,\mc{B}_{ij}} \bar{\m{H}}_{\mc{B}_{ij},\mc{B}_{ij}}^{-1} \bar{\m{H}}_{\mc{B}_{ij},ij} \geq M_4(\bar{\g{\theta}}^o), \quad \textnormal{where}\quad \mc{B}_{ij}:=\mc{B}/\{(i,j)\}.
    \end{equation*} 
   \item There exists a constant $M_7(\bar{\g{\theta}}^o)<\infty$, such that for any $(i,j) \in \mc{B}^c$
   \begin{equation}\label{eq:m7}
        \| \bar{\m{H}}_{ij,\mc{B}} \bar{\m{H}}_{\mc{B},\mc{B}}^{-1}\|\leq M_7(\bar{\g{\theta}}^o).
    \end{equation}
\end{enumerate}
\end{lem}

\begin{lem}\label{lem:bound:lip}
\citep{peng2009partial} Suppose Assumptions~\ref{assu:subgauss}--\ref{assu:beig} hold, then for any $\eta>0$, there exist constants $c_{0}$--$c_{3}$, such that for any $\m{v} \in \mb{R}^{q}$ the following events hold with probability at least $1-O(\exp(-\eta \log p))$ for sufficiently large $n$:
\begin{enumerate}[label={\textnormal{(L\arabic*)}}]
\item $\|\bar{\m{g}}^n_{\mc{B}}\| \leq c_{0}\sqrt{q\frac{\log p}{n}}$.
    \item $|\m{v}^\top \bar{\m{g}}^n_{\mc{B}}| \leq c_{1}\|\m{v}\|\sqrt{q\frac{\log p}{n}}$.
    \item $|\m{v}^\top (\bar{\m{H}}^n_{\mc{B},\mc{B}}-\bar{\m{H}}_{\mc{B},\mc{B}}) \m{v}| \leq c_{2}\|\m{v}\|^2q\sqrt{\frac{\log p}{n}}$.
    \item $\|(\bar{\m{H}}^n_{\mc{B},\mc{B}} - \bar{\m{H}}_{\mc{B},\mc{B}})\m{v}\| \leq c_{3}\|\m{v}\|q\sqrt{\frac{\log p}{n}}$.
\end{enumerate}
\end{lem}
\begin{lem}\label{lem:disc:conti}
Suppose Assumptions~\ref{assu:subgauss}--\ref{assu:comm} are satisfied.  Assume further that $\rho_{1n} =O(\sqrt{\log p/n})$, $n = O(q \log (p))$, $\rho_{2n} =O(\sqrt{\log(p-H+1)/n})$, and $\g{\epsilon}=\g{0}$. Then, there exist finite constants $C_1(\bar{\g{\theta}}^o)$ and $D_1(\bar{\m{q}}^o)$, such that for any $\eta>0$, there exists a (local) minimizer of the restricted problem \eqref{eqn:local:estmin} within the disc: 
\begin{align}\label{eqn:disc}
\nonumber
\Big\{ ( \widehat{\g{\theta}}^o, \widehat{\m{q}}^o ): &\max\left(\| \widehat{\g{\theta}}^o_{\mc{B}}- \bar{\g{\theta}}^o_{\mc{B}} \|, \|\wh{\m{q}}^o - \bar{\m{q}}^o\| \right) \\
&\leq \max \left( C_1(\bar{\g{\theta}}^o) \rho_{1n} \sqrt{q},  D_1(\bar{\m{q}}^o) \rho_{2n} \sqrt{\Psi(p,H,K)}\right) \Big\}
\end{align} 
with probability at least $1 - O(\exp(-\eta \log p))$ for sufficiently large $n$.
\end{lem}
\begin{proof}
Let $\mu_{1n} = \rho_{1n}\sqrt{q}$ with $q=|\mc{B}|$ and  $\mu_{2n} = \rho_{2n}\sqrt{\Psi(p,H,K)}$. Let $C_1>0$ and $\m{w} \in \mathbb{R}^{\binom{p}{2}}$ such that $\m{w}_{\mc{B}^c}=0$, $\|\m{w}\|_2=C_1$. Further, assume $\m{z} \in \mathbb{R}^{\binom{p-H+1}{2}}$ be an arbitrary vector with finite entries and $\|\m{z}\|=D_1$. For sufficiently large $n$, we have
\begin{equation*}
F_{n}(\bar{\g{\theta}}^d,\bar{\g{\theta}}^o+ \mu_{1n} \m{w}, \bar{\m{r}}^d, \bar{\m{r}}^o+ \mu_{1n} \m{z};\m{Y})-  F_n(\bar{\g{\theta}}^d,\bar{\g{\theta}}^o,\bar{\m{r}}^d,\bar{\m{r}}^o;\m{Y})= I_1+  I_2 + I_3.   
\end{equation*}
Here, 
\begin{align*}
I_1 &:= L_n(\bar{\g{\theta}}^{d},\bar{\g{\theta}}^{o}+\mu_{1n}\m{w};(1-\sqrt{\rho_{2n}})\m{Y})
      - L_n(\bar{\g{\theta}}^{d},\bar{\g{\theta}}^{o};\,(1-\sqrt{\rho_{2n}})\m{Y}) \\
I_2 &:=  \rho_{1n}  (\|\bar{\g{\theta}}^o+ \mu_{1n} \m{w}\|_1 -\|\bar{\g{\theta}}^o\|_1), ~~\textnormal{and}\\
I_3 &:=  \rho_{2n} \trace \left((\bar{\g{\Theta}}+  \mu_{1n} \mc{A} \m{w})^2 (\m{S}+\m{N}(\bar{\m{R}} + \mu_{2n} \mc{A} \m{z}) \m{N}^\top)-\bar{\g{\Theta}}^2
\bar{\m{Q}}\right),
\end{align*}
where we used our assumption that $\rho_{2n} \rightarrow 0$ as $n \rightarrow \infty$, and
$\m{Y} =(1-\sqrt{\rho_{2n}}) \m{Y} + \sqrt{\rho_{2n}}\m{Y}.$

Following \cite{peng2009partial}, we first provide lower bounds for $I_1$ and $I_2$.  For the term $I_1$,  it follows from Lemma~\ref{lem:bound:lip} that
\begin{align}\label{eqn:hessb}
 \m{w}_{\mc{B}}^\top \bar{\m{H}}_{\mc{B},\mc{B}} \m{w}_{\mc{B}} \geq \Lambda_{\min}( \bar{\m{H}}_{\mc{B},\mc{B}}) \|\m{w}_{\mc{B}}\|_2^2\geq M_1 C^2_1, 
\end{align}
which together with Lemma~\ref{lem:bound:lip} gives
\begin{align*}\label{eqn:bI1}
   I_1 &= \mu_{1n} \m{w}^\top_{\mc{B}} \bar{\m{g}}^n_{\mc{B}} + \frac{1}{2}\mu_{1n}^2 \m{w}^\top_{\mc{B}} \bar{\m{H}}^n_{\mc{B},\mc{B}}\m{w}_{\mc{B}}  = \mu_{1n} (\m{w}_{\mc{B}})^\top \bar{\m{g}}^n_{\mc{B}} + \frac{1}{2}\mu_{1n}^2 \m{w}_{\mc{B}}^\top \bar{\m{H}}^n_{\mc{B},\mc{B}}\m{w}_{\mc{B}} \\
    & = \mu_{1n} \m{w}_{\mc{B}}^\top \bar{\m{g}}^n_{\mc{B}} + \frac{1}{2}\mu_{1n}^2 \m{w}_{\mc{B}}^\top (\bar{\m{H}}^n_{\mc{B},\mc{B}} - \bar{\m{H}}_{\mc{B},\mc{B}})\m{w}_{\mc{B}}  + \frac{1}{2}\mu_{1n}^2 \m{w}_{\mc{B}}^\top \bar{\m{H}}_{\mc{B},\mc{B}}\m{w}_{\mc{B}}\\
    & \geq \frac{1}{2}  (1-\sqrt{\rho_{2n}})^2 \big( \mu_{1n}^2 M_1 C^2_1 - \mu_{1n} c_{1}\|\m{w}_{\mc{B}}\|_2^2\sqrt{q\frac{\log p}{n}} - \frac{1}{2} \mu_{1n}^2  c_{2}\|\m{w}_{\mc{B}}\|_2^2q \sqrt{\frac{\log p}{n}} \big). 
\end{align*} 
For sufficiently large $n$, by assumption that $\rho_{1n}\sqrt{n/\log p} \rightarrow \infty$ if $p \rightarrow \infty$ and $\sqrt{\log p / n}=o(1)$, the second term in the last line above is $O(\mu_{1n} \sqrt{q} \rho_{n}) = o(\mu_{1n}^2)$; the last term is $o(\mu_{1n}^2)$. Thus, for sufficiently large $n$, we have 
\begin{equation}\label{eqn:bbI1}
I_1  \geq  \frac{1}{2} (1-\sqrt{\rho_{2n}})^2\mu_{1n}^2 M_1 C^2_1. 
\end{equation} 
For the term $I_2$, by Cauchy-Schwartz and triangle inequality, we have
\begin{equation}\label{eqn:bI2}
|I_2| =   \rho_{1n}|  \|\bar{\g{\theta}}^o + \mu_{1n} \m{w}\|_1 - \mu_{1n} \|\m{w}\|_1 | \leq  C_1 \rho_{1n} \mu_{1n} \sqrt{q} =  C_1 \mu_{1n}^2.    
\end{equation}
Next, we provide an upper bound for $I_3$. Note that 
\begin{align*}
I_{3} &:= \rho_{2n} \trace \left( (\bar{\g{\Theta}}  + \mu_{1n} \mc{A} \m{w})^2  \m{S} \right) \\
&+\rho_{2n}\trace \left((\bar{\g{\Theta}}+  \mu_{1n} \mc{A} \m{w})^2 (\m{N}(\bar{\g{R}}+  \mu_{2n} \mc{A} \m{z}) \m{N}^\top)- \bar{\g{\Theta}}^2
\bar{\m{Q}} \right)
 \\
    & = \rho_{2n} \trace \left( \bar{\g{\Theta}}^2 \bar{\m{S}} \right) + \rho_{2n} \trace \left( (\bar{\g{\Theta}}+  \mu_{1n} \mc{A} \m{w})^2(\m{S}- \bar{\m{S}})   \right)    \\
   & +  \rho_{2n}\trace \left((\bar{\g{\Theta}}+  \mu_{1n} \mc{A} \m{w})^2 (\m{N}(\bar{\g{R}}+  \mu_{2n} \mc{A} \m{z}) \m{N}^\top)- \bar{\g{\Theta}}^2
\bar{\m{Q}} \right)      
\end{align*}
Now, we have 
\begin{align*}
&  \left| \trace \left( (\bar{\g{\Theta}}+  \mu_{1n} \mc{A} \m{w})^2(\m{S}-\bar{\m{S}})   \right)   \right| \\
&+  \left| \trace \left((\bar{\g{\Theta}}+  \mu_{1n} \mc{A} \m{w})^2 (\m{N}(\bar{\g{R}}+  \mu_{2n} \mc{A} \m{z}) \m{N}^\top)- \bar{\g{\Theta}}^2
\bar{\m{Q}} \right) \right|\\
&=  \left| \trace \left((\bar{\g{\Theta}}+  \mu_{1n} \mc{A} \m{w})^2 (\m{N}(\bar{\g{R}}+  \mu_{2n} \mc{A} \m{z}) \m{N}^\top)- \bar{\g{\Theta}}^2
\bar{\m{Q}} \right) \right| + I_{3,0}\\ 
&= \left|\trace \left((\bar{\g{\Theta}}^2 +2 \mu_{1n} \bar{\g{\Theta}} \mc{A} \m{w} ) (\m{N}(\bar{\g{R}}+  \mu_{2n} \mc{A} \m{z}) \m{N}^\top)- \bar{\g{\Theta}}^2
\bar{\m{Q}} \right) \right|+ I_{3,1} + I_{3,0}\\
&= \left|\trace \left(\bar{\g{\Theta}}^2 (\m{N}(\bar{\g{R}}+  \mu_{2n} \mc{A} \m{z}) \m{N}^\top- \bar{\m{Q}}) \right)\right|+I_{3,2} +I_{3,1} + I_{3,0}\\
&= I_{3,3}+I_{3,2} +I_{3,1}+ I_{3,0}, 
\end{align*}
where $\bar{\m{S}}$ represents the population covariance of the sample covariance matrix $\mathbf{S}$.

From Assumption~\ref{assu:beig}, we have
\begin{align}\label{eq:boun:trace}
\nonumber 
I_{3,0} & \leq O( \|\m{S}-\bar{\m{S}} \|),\\
\nonumber 
I_{3,1} &= \mu_{1n}^2 \left|\trace \left( (\mc{A} \m{w})^2 \m{N}\bar{\g{R}} \m{N}^\top \right) \right|+ \mu_{1n}^2  \mu_{2n}  \left|\trace \left((\mc{A}\m{w})^2 \m{N} \mc{A} \m{z}\m{N}^\top \right)\right|, \\
\nonumber 
&  \leq \tau_3^2 C_1^2 \mu_{1n}^2 + C_1^2 D_1  \mu_{1n}^2  \mu_{2n}, \\
\nonumber 
I_{3,2} &= 2 \mu_{1n} \left|\trace \left(\bar{\g{\Theta}} \mc{A} \m{w} \m{N}\bar{\g{R}} \m{N}^\top  \right) \right|+ 2 \mu_{1n}  \mu_{2n} \left| \trace \left(\bar{\g{\Theta}} \mc{A} \m{w} \m{N} \mc{A} \m{z}\m{N}^\top \right)\right| \\
\nonumber 
& \leq  2 \tau_2 \tau_3 C_1  \mu_{1n}  + 2 \mu_{1n}  \mu_{2n} \tau_2C_1 D_1  \\
I_{3,3} &:=   \left| \trace \left(\bar{\g{\Theta}}^2 (\m{N}(\bar{\g{R}}+  \mu_{2n} \mc{A} \m{z}) \m{N}^\top- \bar{\m{Q}}) \right)\right|
\leq  \tau_2^2 C_1^2 \mu_{2n}.
\end{align}
Hence, for sufficiently large $n$, we have 
\begin{equation}\label{eqn:bI3}
I_3  \geq  O(\tau_{1} \rho_{2n}). 
\end{equation}
Now, by combining \eqref{eqn:bbI1}-\eqref{eqn:bI3}, for sufficiently large $n$, we obtain
\begin{align*}
& F_{n}(\bar{\g{\theta}}^d,\bar{\g{\theta}}^o+ \mu_{1n} \m{w}, \bar{\m{r}}^d, \bar{\m{r}}^o+ \mu_{1n} \m{z};\m{Y})-  F_n(\bar{\g{\theta}}^d,\bar{\g{\theta}}^o,\bar{\m{r}}^d,\bar{\m{r}}^o;\m{Y}) \\
      & \geq  \frac{1}{2} (1-\rho_{2n})^2 M_1 C^2_1   \mu_{1n}^2 -    C_1 \mu_{1n}^2 +   O(\tau_{1} \rho_{2n})  \geq 0.  
\end{align*} 
Here,  the last inequality follows by setting $C_1 > 2/(M_1(1-\rho_{2n})^2)$.
  
Now, let  $\mc{S}_{\m{w}, \m{z}} = \left\{ (\m{w}, \m{z}) : \m{w}_{\mc{B}^c} = 0,\|\m{w}\|=C_1, \|\m{z}\|=D_1 \right\}$. Then, for $n$ sufficiently large, the following holds 
\begin{equation*}
    \inf_{ (\m{w}, \m{z})  \in \mc{S}_{\m{w}, \m{z}}}  ~ F_{n}(\bar{\g{\theta}}^d,\bar{\g{\theta}}^o+ \mu_{1n} \m{w}, \bar{\m{r}}^d, \bar{\m{r}}^o+ \mu_{2n} \m{z};\m{Y}) > F_n(\bar{\g{\theta}}^d,\bar{\g{\theta}}^o,\bar{\m{r}}^d,\bar{\m{r}}^o;\m{Y}),
\end{equation*} 
with probability at least $1-O(\exp(-\eta \log p))$.

Thus, any solution to the problem defined in \eqref{eqn:local:estmin} is within the disc  \eqref{eqn:disc}
with probability at least $1-O(\exp(-\eta \log p))$. Finally, since  $\bar{\m{Q}}=\m{N} \bar{\m{R}}\m{N}^\top$ and $\m{N}^\top\m{N}=\m{I}_{(p-H+1)}$, we have  that $\|\wh{\m{q}}^o- \bar{\m{q}}^0\| = \|\wh{\m{r}}^o- \bar{\m{r}}^0\|$. This completes the proof. 
\end{proof}

\begin{lem}\label{lem:gradnorm}
Assume conditions of Lemma~\ref{thm:one:fglasso} hold and  $\rho_{2n} < \delta \rho_{1n}/(\tau_2 \tau_3)$. Then, there exists a constant $C_2>0$, such that for any $\eta>0$, for sufficiently large $n$, the following event holds with probability at least $1-O(\exp(-\eta \log p))$: for any $\g{\theta}^o $ satisfying 
\begin{equation}\label{eqn:bcon}
\|\g{\theta}^o-\bar{\g{\theta}}^o\| \geq C_2\sqrt{q }\rho_{1n}, ~~~\g{\theta}^o_{\mc{B}^c}=0,
\end{equation}
we have $\| \nabla_{\g{\theta}^o}F_{n}(\bar{\g{\theta}}^d, \hat{\g{\theta}}^o_{\mc{B}}, \bar{\m{r}}^d, \hat{\m{r}}^o;\m{Y})  \| >   \sqrt{q }\rho_{1n} $.
\end{lem}
\begin{proof}
The proof follows the idea of \citep[Lemma S-4]{peng2009partial}. For $\g{\theta}^o=\hat{\g{\theta}}^o$ satisfying \eqref{eqn:bcon}, we have $\hat{\g{\theta}}^o=\bar{\g{\theta}}^o+\mu_{1n} \m{w}$, with $\m{w}_{\mc{B}^c}=0$ and $\|\m{w}\| \geq C_2$. 
We have  
\begin{align*}
 \nabla_{\g{\theta}^o} F_{n}(\bar{\g{\theta}}^d, \hat{\g{\theta}}^o_{\mc{B}}, \bar{\m{r}}^d, \hat{\m{r}}^o;\m{Y})  &= \hat{\m{g}}^n_{\mc{B}} + \rho_{2n} \m{N} \hat{\m{R}}^o\m{N}^\top \mc{A}^* \mc{A}  \hat{\g{\theta}}^o_{\mc{B}} \\
 &=  \bar{\m{g}}^n_{\mc{B}}+  \mu_{1n}\bar{\m{H}}_{\mc{B},\mc{B}}^n\m{w}_{\mc{B}} + \rho_{2n} \m{N} \hat{\m{R}}^o\m{N}^\top \mc{A}^* \mc{A}  \bar{\g{\theta}}^o_{\mc{B}}\\
 &\geq \bar{\m{g}}^n_{\mc{B}}+   \mu_{1n}\bar{\m{H}}_{\mc{B},\mc{B}}^n\m{w}_{\mc{B}} \\
 &+ \rho_{2n} \m{N} \hat{\m{R}}^o\m{N}^\top \mc{A}^* \mc{A}  \bar{\g{\theta}}^o_{\mc{B}} +  \rho_{2n} \mu_{1n} (\m{N} \hat{\m{R}}^o\m{N}^\top) \mc{A}^* \mc{A}  \m{w}_{\mc{B}},
\end{align*}    
where the inequality follows from Taylor expansion of $\nabla_{\g{\theta}^o}L_{n}(\bar{\g{\theta}}^d, \hat{\g{\theta}}^o_{\mc{B}};\m{Y})$.

Let 
\begin{equation}\label{eqn:aa}
 \hat{\m{A}}:= \m{N} \hat{\m{R}}^o\m{N}^\top \mc{A}^* \mc{A}^\top \bar{\g{\theta}}^o ~~\textnormal{and} ~~\bar{\m{A}}:= \m{N} \bar{\m{R}}^o\m{N}^\top \mc{A}^* \mc{A}^\top \bar{\g{\theta}}^o.   
\end{equation}
Then, we have
\begin{align}\label{eqn:ahat}
\|\hat{\m{A}}_{\mc{B}}\| &\leq    \|\m{N} \bar{\m{R}}^o\m{N}^\top \| \|\mc{A}^*\mc{A}  \bar{\g{\theta}}^o\| +    \|\bar{\m{A}} - \hat{\m{A}}\| \leq  2\tau_2 \tau_3 \sqrt{q},
\end{align}
where the last inequality follows since
\begin{align*}
&\|\bar{\m{A}}\|  \leq \Lambda_{\max} (\m{N} \bar{\m{R}}^o\m{N}^\top ) \|\mc{A}^*\mc{A}  \bar{\g{\theta}}^o\| \leq  \tau_2 \tau_3 \sqrt{q},\\
& \|\bar{\m{A}} -  \hat{\m{A}}\|=o(\rho_{2n}).
\end{align*}

Now, let $\mu_{1n} =\sqrt{q}\rho_{1n}$. By triangle inequality and similar proof strategies as in Lemma~\ref{lem:disc:conti}, for sufficiently large $n$, we obtain 
\begin{align*}
 \| \nabla_{\g{\theta}^o}F_{n}(\bar{\g{\theta}}^d, \hat{\g{\theta}}^o_{\mc{B}}, \bar{\m{r}}^d, \hat{\m{r}}^o;\m{Y})  \| & \geq    \mu_{1n}\|\bar{\m{H}}_{\mc{B},\mc{B}}\m{w}_{\mc{B}}\| -   c_{0} (q^{1/2}n^{-1/2}\sqrt{\log p})\\
    &- c_{3} \|\m{w}_{\mc{B}}\|_2(\mu_{1n} qn^{-1/2}\sqrt{\log p}) - 2\tau_2 \tau_3 \sqrt{q} \rho_{2n} -o(\rho_{2n})\\
& \geq  M_1 C_2 \sqrt{q} \rho_{1n}- 2\tau_2 \tau_3 \sqrt{q} \rho_{2n},
\end{align*} 
with probability at least $1-O(\exp(-\eta \log p))$. Here, the first inequality uses Lemma~\ref{lem:bound:lip} and the last inequality follows from  Lemma~\ref{lem:bound:hess} where $\|\bar{\m{H}}_{\mc{B},\mc{B}}\m{w}_{\mc{B}}\|\geq  M_1\|\m{w}_{\mc{B}}\|$. Now, taking
\begin{equation}\label{eqn:c2}
 C_2=\frac{1+2 \delta }{M_1 +\epsilon}  
\end{equation}
for some $\epsilon > 0 $, completes the proof.
\end{proof}
Next, inspired by \cite{peng2009partial,khare2015convex}, we prove estimation consistency for the nodewise FCONCORD, restricted to the true support, i.e., $\g{\theta}^o_{\mc{B}^c}=0$.
\begin{lem}\label{thm:one:fglasso}
Suppose Assumptions~\ref{assu:subgauss}--\ref{assu:comm} are satisfied. Assume  $\rho_{1n} =O(\sqrt{\log p/n})$, $n > O(q \log (p))$ as $n \rightarrow \infty$, $ \rho_{2n} =O(\sqrt{\log(p-H+1)/n})$, $  \rho_{2n} \leq \delta\rho_{1n}/((1+M_7(\bar{\g{\theta}}^o))\tau_2 \tau_3)$, and $\g{\epsilon}=\g{0}$. Then, there exist finite constants $C(\bar{\g{\theta}}^o)$ and $D(\bar{\m{q}}^o)$, such that for any $\eta>0$, the following events hold with probability at least $1 - O(\exp(-\eta\log p))$:
\begin{itemize}
\item There exists a local minimizer $(\wh{\g{\theta}}^o_{\mc{B}}, \wh{\m{q}}^o)$ of \eqref{eqn:local:estmin} such that
\begin{align*}	 							
&\max\left(\| \widehat{\g{\theta}}^o_{\mc{B}}- \bar{\g{\theta}}^o_{\mc{B}} \|, \|\wh{\m{q}}^o- \bar{\m{q}}^o\| \right) \\
& \qquad \leq \max \left( C(\bar{\g{\theta}}^o) \rho_{1n}\sqrt{q}/, D(\bar{\m{q}}^o) \rho_{2n} \sqrt{\Psi(p,H,K)}\right),
\end{align*}    
where $q$ and $\Psi(p,H,K)$ are defined in \eqref{eqn:rate:quant}.
\item   If  $\min_{(i,j) \in \mc{B}} \bar{\theta}_{ij} \geq  2 C(\bar{\g{\theta}}^o) \rho_{1n}\sqrt{q}$,  then $\wh{\g{\theta}}^o_{\mc{B}^c} =0$.
\end{itemize}
\end{lem}

\begin{proof} 
By the KKT condition, for any solution $(\hat{\g{\theta}}^o,\hat{\m{r}}^o)$ of \eqref{eqn:local:estmin}, it satisfies
\begin{align*}
 \| \nabla_{\g{\theta}^o}F_{n}(\bar{\g{\theta}}^d, \hat{\g{\theta}}^o_{\mc{B}}, \bar{\m{r}}^d, \hat{\m{r}}^o;\m{Y})  \|_\infty &\leq \   \rho_{1n}.
 \end{align*}
Thus, for $n$ sufficiently large, we have 
\begin{align*}
\| \nabla_{\g{\theta}^o}F_{n}(\bar{\g{\theta}}^d, \hat{\g{\theta}}^o_{\mc{B}}, \bar{\m{r}}^d, \hat{\m{r}}^o;\m{Y})  \| \leq
 \sqrt{q }\rho_{1n},
\end{align*}
Let $C(\bar{\g{\theta}}^o):=C_2$. Using \eqref{eqn:c2} and Lemma~\ref{lem:gradnorm}, we obtain
\begin{align*}
\|\hat{\g{\theta}}^o -\bar{\g{\theta}}^o\| \leq  C(\bar{\g{\theta}}^o)  \sqrt{q }\rho_{1n},~~~\g{\theta}^o_{\mc{B}^c}=0
\end{align*}
with probability at least $1-O(\exp(-\eta \log p))$. 

Now, if $\min_{(i,j) \in \mc{B}} \bar{\g{\theta}}^o_{ij} \geq 2 C(\bar{\g{\theta}}^o)  \sqrt{q }\rho_{1n}$, then
\begin{align*}
 & 1-O(\exp(-\eta \log p)) \\
 &\leq P_{\bar{\g{\theta}}^o} \left( \|\wh{\g{\theta}}^o_{\mc{B}}-\bar{\g{\theta}}^o_{\mc{B}}\| \leq  C(\bar{\g{\theta}}^o)  \sqrt{q }\rho_{1n},~~\min_{(i,j) \in \mc{B}} \bar{\g{\theta}}^o_{ij} \geq 2  C(\bar{\g{\theta}}^o)  \sqrt{q }\rho_{1n} \right)\\
& \leq P_{\bar{\g{\theta}}^o} \left(\text{sign}(\hat{\theta}^o_{ij})=\text{sign}(\bar{{\theta}}^o_{ij}),~~\forall (i,j) \in \mc{B}\right).
\end{align*}
\end{proof}

The following Lemma~\ref{thm:two:fglasso} shows that no wrong edge is selected with probability tending to one.
\begin{lem}\label{thm:two:fglasso}
  Suppose that the conditions of Lemma~\ref{thm:one:fglasso} and Assumption~\ref{assu:incoh} are satisfied. Suppose further that $p = O(n^{\alpha})$ for some ${\alpha} > 0$. Then for $\eta > 0$, and for $n$ sufficiently large, the solution of \eqref{eqn:local:estmin} satisfies
\begin{equation}
P\left(\| \nabla_{\g{\theta}^o}F_{n}(\bar{\g{\theta}}^d, \hat{\g{\theta}}^o_{{\mc{B}^c}}, \bar{\m{r}}^d, \hat{\m{r}}^o;\m{Y})\|_\infty < 1 \right) \geq  1- O(\exp(-\eta \log p)).
\end{equation}
\end{lem}
\begin{proof}
Let $\mc{E}_{n,k}=\{\text{sign}(\wh{\g{\theta}}^o_{ij,\mc{B}})=\text{sign}(\bar{\g{\theta}}^o_{ij,\mc{B}})\}$. Then by Lemma~\ref{thm:one:fglasso}, $P_{\bar{\g{\theta}}^o}(\mathcal{E}_{n}) \geq 1-O(\exp(-\eta \log p))$ for large $n$.
Define the sign vector $\hat{\m{t}}$ for $\hat{\g{\theta}}^o$ to satisfy the following properties,
\begin{equation}\label{eqn:sig:def}
	 \begin{cases} \hat{t}_{ij} = \textnormal{sign}(\hat{\theta}^o_{ij}), & \textnormal{if } \theta_{ij}^o\neq 0 \ ,  \\ 
                          \rvert \hat{t}_{ij} \rvert \leq 1, & \textnormal{if } \hat{\theta}_{ij}^o = 0 \ . \end{cases}
\end{equation}
for all  $1 \leq i<j \leq p$.

On $\mathcal{E}_{n,k}$, by the KKT condition and the expansion of $F_n$ at $(\bar{\g{\theta}}^d, \hat{\g{\theta}}^o,\bar{\m{r}}^o,\hat{\m{r}}^o)$, we have 
 \begin{align}\label{eqn:firs}
\hat{\m{g}}^n_{\mc{B}} + \rho_{1n} \hat{\m{t}}_{\mc{B}} +  \rho_{2n} (\m{N} \hat{\m{R}}^o\m{N}^\top) \mc{A}^* \mc{A}  \hat{\g{\theta}}^o_{\mc{B}}=0.
 \end{align}   
 where $\hat{\m{g}}^n= \nabla_{\g{\theta}^o}L_{n}(\bar{\g{\theta}}^d, \hat{\g{\theta}}^o;\m{Y})$.
Then, we can write   
$$
\hat{\m{g}}^n_{\mc{B}} - \bar{\m{g}}^n_{\mc{B}} = - \rho_{1n} \hat{\m{t}}_{\mc{B}} - \rho_{2n} \m{N} \hat{\m{R}}^o\m{N}^\top \mc{A}^* \mc{A}^\top  \hat{\g{\theta}}^o_{\mc{B}}  -\bar{\m{g}}^n_{\mc{B}}.
$$
Let $\tilde{\g{\theta}}^o$ denote a point in the line segment connecting $\hat{\g{\theta}}^o$ and $\bar{\g{\theta}}^o$. Applying the Taylor expansion, we obtain 
\begin{equation}
\label{eqn:main_structure}
\bar{\m{H}}\left(\bar{\g{\theta}}^o_{\mc{B}}- \hat{\g{\theta}}^o_{\mc{B}}\right) = -\bar{\m{g}}^n_{\mc{B}}- \rho_{1n} \hat{\m{t}}_{\mc{B}} + \m{L}^n_{\mc{B}} - \rho_{2n}\hat{\m{A}}_{\mc{B}}.
\end{equation}
where $\m{L}^n := \left(\bar{\m{H}}^n- \bar{\m{H}}\right) \left(\hat{\g{\theta}}^o- \bar{\g{\theta}}^o\right)$ and  
$\hat{\m{A}} := \m{N} \hat{\m{R}}^o\m{N}^\top \mc{A}^* \mc{A}^\top  \hat{\g{\theta}}^o$.

Now, by utilizing the fact that $\hat{\g{\theta}}_{\mc{B}^c}^o= \bar{\g{\theta}}_{\mc{B}^c}^o= 0$, we have
\begin{eqnarray}
\label{structure_I11}
 \bar{\m{H}}_{\mc{B}^c\mc{B}}(\bar{\g{\theta}}_{\mc{B}}^o - \hat{\g{\theta}}_{\mc{B}}^o) & = & - \bar{\m{g}}^n_{\mc{B}^c} - \rho_{1n}\hat{\m{t}}_{\mc{B}^c} + \m{L}^n_{\mc{B}^c} -  \rho_{2n}\hat{\m{A}}_{\mc{B}^c} , \\
\label{eqn:structure_II2}
 \bar{\m{H}}_{\mc{B}\mc{B}}(\bar{\g{\theta}}_{\mc{B}}^o - \hat{\g{\theta}}_{\mc{B}}^o) & = & - \bar{\m{g}}^n_{\mc{B}} - \rho_{1n}\hat{\m{t}}_{\mc{B}} + \m{L}^n_{\mc{B}} - \rho_{2n}\hat{\m{A}}_{\mc{B}} . 
\end{eqnarray}
Since $\bar{\m{H}}^n_{\mc{B}\mc{B}}$ is invertible by assumption, we get
Now, using results from Lemmas \ref{lem:control_Wn} and \ref{lem:control_Rn}, we obtain
\begin{eqnarray}\label{eqn:thm2_bouding}
\nonumber 
\rho_{1n}  \|\hat{\m{t}}_{\mc{B}^c}\|_\infty   &=& \big\| \bar{\m{H}}_{\mc{B}^c\mc{B}}(\bar{\m{H}}_{\mc{B}\mc{B}})^{-1}(- \bar{\m{g}}^n_{\mc{B}} - \rho_{1n}\hat{\m{t}}_{\mc{B}}  \\
&+& \m{L}^n_{\mc{B}}- \rho_{2n}\hat{\m{A}}_{\mc{B}}) - \bar{\m{g}}^n_{\mc{B}^c} - \m{L}^n_{\mc{B}^c} + \rho_{2n}\hat{\m{A}}_{\mc{B}^c} \big\|_\infty. 
\end{eqnarray}
Now, (i) by the incoherence condition outlined in Assumption~\ref{assu:incoh}, for any $(i,j) \in \mc{B}^c$, we have
\begin{equation*}
    \left|\bar{\m{H}}_{ij,\mc{B}}\bar{\m{H}}_{\mc{B},\mc{B}}^{-1} \text{sign}(\bar{\g{\theta}}^o_{\mc{B}})\right| \leq (1-\delta) < 1.
\end{equation*}
(ii)  by Lemma~\ref{lem:bound:hess}, for any $(i,j) \in \mc{B}^c$: $||\bar{\m{H}}_{ij,\mc{B}}\bar{\m{H}}_{\mc{B}\mc{B}}^{-1}||
\leq M_7 (\bar{\theta})$; (iii) by the similar steps as in the proof of Lemma~\ref{lem:gradnorm},
\begin{equation}\label{eq:boun:ahat}
\rho_{2n} \|\bar{\m{H}}_{\mc{B}^c\mc{B}}(\bar{\m{H}}_{\mc{B},\mc{B}})^{-1}\|_\infty \|\hat{\m{A}}_{\mc{B}}\|_{\infty}  \leq  2 (1+M_7 (\bar{\theta})) \tau_2 \tau_3  \rho_{2n}.
\end{equation}
Thus, following straightforwardly (with the modification that we are considering each $\mc{B}$ instead of $\mc{B}$) from the proofs of \cite[Theorem~2]{peng2009partial}, the remaining terms in \eqref{eqn:thm2_bouding} can be shown to be all $o(\rho_{1n})$, and the event
\begin{eqnarray*}
\rho_{1n}  \|\hat{\m{t}}_{\mc{B}^c}\|_\infty  &\leq&  \rho_{1n}  (1-\delta) +  4 (1+M_7) \tau_2 \tau_3  \rho_{2n} \\
& \leq & \rho_{1n}  (1-3\delta/4)
\end{eqnarray*}
holds with probability at least $1-O(\exp(-\eta \log p))$ for sufficiently large $n$ and $\rho_{2n} \leq \delta \rho_{1n}/(16(1+
M_7 (\bar{\theta})) \tau_2 \tau_3)$. Thus, it has been proved that for sufficiently large $n$, no wrong edge will be included for each true edge set $\mc{B}$. 
\end{proof}

\subsubsection{Proof of Theorem~\ref{thm:three:fglasso}}

\begin{proof}
By Lemmas~\ref{thm:one:fglasso}~and~\ref{thm:two:fglasso}, with probability tending to $1$, there exists a local minimizer of the restricted problem that is also a minimizer of the original problem. This completes the proof.
\end{proof}

\subsection{Large Sample Properties of FBN}\label{sec:app:fbn} 

The proof bears some similarities to the proof of \cite{ravikumar2010high,Guo15} for the neighborhood selection method, who in turn adapted the proof from \cite{meinshausen2006high} to binary data; however, there are also important differences, since all conditions and results are for fair clustering and joint estimation, and many of our bounds need to be more precise than those given by \cite{ravikumar2010high,Guo15}. Throughout, we set 
\begin{align}\label{eq:newf}
\nonumber
F (\g{\Theta},\m{Q};\m{Y})&= \sum_{j=1}^p \sum_{j'=1}^p -n\theta_{jj'} s_{jj'}+\frac{\rho_2n}{2} \trace\big((\m{Q}+\iota_n\m{I})\g{\Theta}^2\big)\\
&+ \sum_{i=1}^n\sum_{j=1}^p \log \big( 1+ \exp(\theta_{jj}+ \sum_{j'\ne j} \theta_{jj'}y_{ij'}) \big).     
\end{align}
Let the notation ${F}_n(\g{\theta}^d, \g{\theta}^o, \m{q}^d, \m{q}^o; \m{Y})$ stands for $\frac{{F}}{n}$. We consider a restricted version of criterion similar to \eqref{eqn:local:estmin} and use the preliminaries introduced in Section~\ref{sec:app:pre} with $F$ replaced with \eqref{eq:newf}. 

Following the literature, we prove the main theorem in two steps: first, we prove the result holds when assumptions \ref{assu:eig:bin} and \ref{assu:inc:bin} hold for $\bar{\m{H}}^n$ and $\m{T}^n$, the sample versions of of  $\bar{\m{H}}$ and $\m{T}$. Then, we show that if \ref{assu:eig:bin} and \ref{assu:inc:bin} hold for the population versions $\bar{\m{H}}$ and $\m{T}$, they also hold for $\bar{\m{H}}^n$ and $\m{T}^n$ with high probability (Lemma~\ref{prop:condition_consistency}).  
 
\begin{enumerate}[label={\textbf{(B\arabic*'})}]
\item\label{assu:eig:bin} 
There exist constants $\tau_{4}, \tau_5 \in (0,\infty)$ such that
\begin{equation*}
\Lambda_{\min}( \bar{\m{H}}^n_{\mc{B}\mc{B}}) \geq \tau_4~~~\text{and}~~~\Lambda_{\max}(\m{T}^n) \leq \tau_5. 
\end{equation*}

\item \label{assu:inc:bin} There exists a constant $\delta\in (0,1]$, such that 
\begin{equation}
  \|\bar{\m{H}}^n_{\mc{B}^c\mc{B}} \left(\bar{\m{H}}^n_{\mc{B}\mc{B}}\right)^{-1}\|_{\infty} \leq (1 - \delta).  
\end{equation}
\end{enumerate}
Here, \(\bar{\mathbf{H}}_{\mathcal{B}\mathcal{B}}^n\) denotes the principal submatrix of the sample Hessian of the empirical loss \(L_n(\bar{\boldsymbol{\theta}}^d,\bar{\boldsymbol{\theta}}^o;\mathbf{Y})\), restricted to the active (nonzero) coordinates \(\mathcal{B}\) of \(\bar{\boldsymbol{\theta}}^o\).

We first list some properties of the loss function. 
\begin{lem}\label{lem:control_Wn}
For $\delta \in (0, 1]$, we have 
\begin{equation*}
\mb{P}\left(\frac{2 - \delta}{\rho_{1n}}\lVert \bar{\m{g}}^n \rVert_{\infty} \geq \frac{\delta}{4}\right) \leq 2 \exp \left(-\frac{\rho_{1n}^2n\delta^2}{128(2 - \delta)^2} + \log{p}\right) .
\end{equation*}
where $\bar{\m{g}}^n:=\nabla_{\g{\theta}^o} L(\bar{\g{\theta}}^d, \bar{\g{\theta}}^o; \m{Y}) $. This probability goes to $0$ as long as $\rho_{1n} \geq \frac{16(2 - \delta)}{\delta} \sqrt{\frac{\log{p}}{n}}$.
\end{lem}
%
\begin{lem}\label{lem:control:eig:hess}
Suppose Assumption~\ref{assu:eig:bin} holds and  $n > Cq^2\log p$ for some positive constant $C$, then for any $\delta \in (0, 1]$,  we have 
\begin{equation*}
\Lambda_{\min}( [\nabla^2_{\g{\theta}^o} L(\bar{\g{\theta}}^d,\bar{\g{\theta}}^o+\delta \m{w}_\mc{B}; \m{Y})]_{\mc{B}\mc{B}}) \geq \frac{\tau_4}{2}.
\end{equation*}
\end{lem}
\begin{lem}\label{lem:control_Rn}
For $\delta \in (0, 1]$, if $\rho_{1n} q \leq \frac{\tau_4^2}{100 \tau_5}\frac{\delta}{2 - \delta}$,  $\lVert \bar{\m{g}}^n\rVert_{\infty} \leq \frac{\rho_{1n}}{4}$, then 
\begin{equation*}
\left\|\left( \bar{\m{H}}_{\mc{B}\mc{B}}^n - \hat{\m{H}}_{\mc{B}\mc{B}}^n\right) \left(\hat{\g{\theta}}^o - \bar{\g{\theta}}^o\right)\right\|_{\infty} \leq \frac{\delta \rho_{1n}^2}{4(2-\delta)} \ . 
\end{equation*}
\end{lem}
\begin{lem}\label{prop:condition_consistency}
If $\bar{\m{H}}^n$ and $\m{T}^n$ satisfy \ref{assu:eig:bin} and \ref{assu:inc:bin}, the following hold for any $\alpha > 0$ and some positive constant $C$:
\begin{eqnarray*}
\mb{P}\left(\Lambda_{\min}(\bar{\m{H}}_{\mc{BB}}^n) \leq \tau_4 - \alpha\right) &\leq& 2 \exp\left(-\frac{\alpha^2n}{2q^2} + 2\log q\right), \\
\mb{P}\left(\Lambda_{\max}\left( \m{T}_{\mc{BB}}^n \right) \geq \tau_5 + \alpha\right) &\leq& 2 \exp\left(- \frac{\alpha^2n}{2q^2}+ 2 \log q  \right), \\
\mb{P}\left(\lVert \bar{\m{H}}_{\mc{B}^c\mc{B}}^n \left(\bar{\m{H}}_{\mc{BB}}^n\right)^{-1}\rVert_{\infty} \geq 1 - \frac{\delta}{2}\right) &\leq& 12 \exp \left(-C \frac{n}{q^3} + 4 \log{p} \right).
\end{eqnarray*}
\end{lem}

We omit the proof of Lemmas~\ref{lem:control_Wn}-\ref{prop:condition_consistency}, which are very similar to ~\cite{ravikumar2010high}.

\begin{lem}\label{lem:disc:ising}
Suppose Assumptions~\ref{assu:beig}--\ref{assu:comm} and  \ref{assu:eig:bin}--\ref{assu:inc:bin} are satisfied by $\bar{\m{H}}^n$ and $\m{T}^n$. Assume further that  $\rho_{1n} \geq  16(2- \delta)/\delta \sqrt{\log p/n}$ and $n > Cq^2\log p$ for some positive constant $C$. Then, with probability at least $1 - 2(\exp(-C \rho_{1n}^2n))$, there exists a (local) minimizer of the restricted problem \eqref{eqn:fgglasso} within the disc: 
\begin{align*}
\qquad \qquad &  \Big\{ ( \widehat{\g{\theta}}^o, \widehat{\m{q}}^o ): \max\left(\| \widehat{\g{\theta}}^o_{\mc{B}}- \bar{\g{\theta}}^o_{\mc{B}} \|_2, \|\wh{\m{q}}^o - \bar{\m{q}}^o\| \right) \\
&  \qquad \qquad  \leq \max \left( \check{C}(\bar{\g{\theta}}^o) \rho_{1n} \sqrt{q},  \check{D}(\bar{\m{q}}^o) \rho_{2n} \sqrt{\Psi(p,H,K)}\right)\Big\}.
\end{align*} 
for some finite constants $\check{C}(\bar{\g{\theta}}^o)$ and $\check{D}(\bar{\m{q}}^o)$. Here, $q$ and $\Psi(p,H,K)$ are defined in \eqref{eqn:rate:quant}. 
\end{lem}
\begin{proof}


The proof is similar to Lemma~\ref{lem:disc:conti}.
Let $\mu_{1n} = \rho_{1n}\sqrt{q}$ with $q=|\mc{B}|$ and  $\mu_{2n} = \rho_{2n}\sqrt{\Psi(p,H,K)}$.  Let $C_1>0$ and $\m{w} \in \mathbb{R}^{{\binom{p}{2}}}$ such that $\m{w}_{\mc{B}^c}=0$, $\|\m{w}\|_2=C_1$. Further, assume $\m{z} \in \mathbb{R}^{{\binom{p-H+1}{2}}}$ be an arbitrary vector with finite entries and $\|\m{z}\|=D_1$.
Let
\begin{equation*}
F_{n}(\bar{\g{\theta}}^d,\bar{\g{\theta}}^o+ \mu_{1n} \m{w}, \bar{\m{r}}^d, \bar{\m{r}}^o+ \mu_{1n} \m{z};\m{Y})-  F_n(\bar{\g{\theta}}^d,\bar{\g{\theta}}^o,\bar{\m{r}}^d,\bar{\m{r}}^o;\m{Y})= I_1+  I_2 + I_3,   
\end{equation*}
where 
\begin{align*}
I_1 &:=  L_n(\bar{\g{\theta}}^d,\bar{\g{\theta}}^o + \mu_{1n} \m{w};\m{Y}) - L_n(\bar{\g{\theta}}^d,\bar{\g{\theta}}^o;\m{Y}),\\
I_2 &:=  \rho_{1n}  (\|\bar{\g{\theta}}^o+ \mu_{1n} \m{w}\|_1 -\|\bar{\g{\theta}}^o\|_1), ~~\textnormal{and}\\
I_3 &:=  \rho_{2n} \trace \left((\bar{\g{\Theta}}+  \mu_{1n} \mc{A} \m{w})^2 (\iota_n\m{I}+\m{N}(\bar{\m{R}} + \mu_{2n} \mc{A} \m{z}) \m{N}^\top)-\bar{\g{\Theta}}^2
\bar{\m{Q}}\right).
\end{align*}

It follows from Taylor expansions that  
$I_1= \mu_{1n} \m{w}_\mc{B}^\top \bar{\m{g}}_\mc{B}^n +  \mu_{1n}^2 \m{w}^\top [\nabla^2_{\g{\theta}^o} L(\bar{\g{\theta}}^d,\bar{\g{\theta}}^o+\delta \m{w}_\mc{B}; \m{Y})]_{\mc{B}\mc{B}} \m{w}_\mc{B},
$ for some $\delta \in (0,1]$. Now, let $\rho_{1n} \geq 16(2- \delta)/\delta \sqrt{\log p/n}$. It follows from Lemma~\ref{lem:control_Wn} that 
\begin{eqnarray}\label{eqn:bI11:is}
|\m{w}^\top_\mc{B}  \bar{\m{g}}_\mc{B}^n | &\leq& \|\bar{\m{g}}_\mc{B}^n \|_\infty\|\m{w}_\mc{B}\|_1 \leq  \mu_{1n} \frac{C_1}{4},
\end{eqnarray}
where the last inequality follows since  $\|\bar{\m{g}}_\mc{B}^n\|_\infty \leq 4$.

Further, using our assumption on the sample size, it follows from Lemma~\ref{lem:control:eig:hess} that 
\begin{eqnarray}\label{eqn:bI12:is}
\m{w}_\mc{B}^\top [\nabla^2_{\g{\theta}^o} L(\bar{\g{\theta}}^d,\bar{\g{\theta}}^o+\delta \m{w}_\mc{B}; \m{Y})]_{\mc{B}\mc{B}} \m{w}_\mc{B}  &\geq& \frac{\tau_4 C_1 ^2}{2}. 
\end{eqnarray}
For the second term, it can be easily seen that 
\begin{equation}\label{eqn:bI2:is}
|I_2|=|\rho_{1n}(\|\bar{\g{\theta}}_\mc{B}^o+\m{w}_\mc{B}\|_1 - \|\bar{\g{\theta}}^o\|_1)| \leq 
\sqrt{q}\rho_{1n}  C_1.     
\end{equation}
In addition, by the similar argument as in the proof of Lemma~\ref{lem:disc:conti}, we obtain 
\begin{align*}
I_{3} &= \rho_{2n} \iota_n \trace (\bar{\g{\Theta}}+  \mu_{1n} \mc{A} \m{w})^2\\
&+  \rho_{2n} \trace \left((\bar{\g{\Theta}}+  \mu_{1n} \mc{A} \m{w})^2 (\m{N}(\bar{\m{R}} + \mu_{2n} \mc{A} \m{z}) \m{N}^\top)-\bar{\g{\Theta}}^2
\bar{\m{Q}}\right).
\end{align*}
Following steps in \eqref{eq:boun:trace}, for sufficiently large $n$, we have  
\begin{equation}\label{eqn:bI3:is}
I_3  \geq   O(\rho_{2n} \iota_n) - 2 \tau_2 \tau_3 C_1 \rho_{2n} \mu_{1n} -  \tau_2^2 C_1^2 \rho_{2n} \mu_{2n}. 
\end{equation}
By combining~~\eqref{eqn:bI11:is}--\eqref{eqn:bI3:is}, and choosing  $\iota_n \geq O(\max(\mu_{1n},\mu_{2n}))$, we obtain
\begin{eqnarray*}
&&~~F_{n}(\bar{\g{\theta}}^d,\bar{\g{\theta}}^o+ \mu_{1n} \m{w}, \bar{\m{r}}^d, \bar{\m{r}}^o+ \mu_{1n} \m{z};\m{Y}) -  F_n(\bar{\g{\theta}}^d,\bar{\g{\theta}}^o,\bar{\m{r}}^d,\bar{\m{r}}^o;\m{Y}) \\
& \geq & C_1^2 \frac{q\log p}{n} \left( \frac{\tau_4}{2}-\frac{1}{C_1} - \frac{1}{4C_1} \right) \\
&+& O(\rho_{2n} \iota_n)  - 2 \tau_2 \tau_3 C_1 \rho_{2n} \mu_{1n} -  \tau_2^2 C_1^2 \rho_{2n} \mu_{2n} \geq 0. 
\end{eqnarray*}
The last inequality uses the condition  $C_1 > 5/\tau_4$.  The proof follows by setting $\check{C}(\bar{\g{\theta}}^o)=C_1$ and  $\check{D}(\bar{\g{\theta}}^o)=D_1$.
\end{proof}

\begin{lem}\label{prop:sample_version}
Suppose Assumptions~\ref{assu:beig}--\ref{assu:comm} hold.
If \ref{assu:eig:bin} and \ref{assu:inc:bin} are satisfied by $\bar{\m{H}}^n$ and $\m{T}^n$, $\rho_{1n} =O(\sqrt{\log p/n})$, $n > O(q^2\log p)$ as $n \rightarrow \infty$, $ \rho_{2n} =O(\sqrt{\log(p-H+1)/n})$, $\rho_{2n} \leq \delta\rho_{1n}/(4(2-\delta) \tau_3 \|\bar{\g{\theta}}^o\|_{\infty})$, $\g{\epsilon}=\g{0}$, and $\min_{(i,j) \in \mc{B}} \bar{\theta}_{ij}^o \geq 2 \check{C}(\bar{\g{\theta}}^o) \sqrt{q }\rho_{1n}$. Then, the result of Theorem \ref{thm:bin} holds.
\end{lem}
\begin{proof}
Define the sign vector $\hat{\m{t}}$ for $\hat{\g{\theta}}$ to satisfy \eqref{eqn:sig:def}. For $\hat{\g{\theta}}$ to be a solution of \eqref{eqn:fair:ising}, the sub-gradient at $\hat{\g{\theta}}$ must be 0, i.e., 
 \begin{equation}\label{subgradient}
\hat{\m{g}}^n + \rho_{1n} \hat{\m{t}} + \rho_{2n}\m{N} \hat{\m{R}}^o\m{N}^\top \mc{A}^* \mc{A}  \hat{\g{\theta}}^o = 0,
\end{equation}
where $\hat{\m{g}}^n=\nabla_{\g{\theta}^o} L(\bar{\g{\theta}}^d, \hat{\g{\theta}}^o; \m{Y})$.
Then we can write   
$$
\hat{\m{g}}^n - \bar{\m{g}}^n = - \rho_{1n} \hat{\m{t}} - \rho_{2n} \m{N} \hat{\m{R}}^o\m{N}^\top \mc{A}^* \mc{A}^\top  \hat{\g{\theta}}^o  -\bar{\m{g}}^n.
$$
Let $\tilde{\g{\theta}}$ denote a point in the line segment connecting $\hat{\g{\theta}}$ and $\bar{\g{\theta}}$.  Applying the mean value theorem gives 
\begin{equation}\label{main_structure}
 \bar{\m{H}}^n\left(\hat{\g{\theta}}^o - \bar{\g{\theta}}^o\right) = -\bar{\m{g}}^n - \rho_{1n} \hat{\m{t}} + \m{L}^n - \rho_{2n}\m{A}^n.
\end{equation}
Here, $\m{L}^n = \left(\bar{\m{H}}_{\mc{B}\mc{B}}^n- \tilde{\m{H}}_{\mc{B}\mc{B}}^n\right) \left(\hat{\g{\theta}}^o- \bar{\g{\theta}}^o\right)$ and $\m{A}^n = \m{N} \hat{\m{R}}^o\m{N}^\top \mc{A}^* \mc{A}^\top \hat{\g{\theta}}^o$.  

Let $\hat{\g{\theta}}_\mc{B}$ be the solution of  restricted problem and let $\hat{\g{\theta}}_{\mc{B}^c} = 0$, i.e. \eqref{eqn:appnull:obj}. We will show that this $\hat{\g{\theta}}$ is the optimal solution and is sign consistent with high probability. To do so, let $\rho_{1n} = \frac{16(2-\delta)}{\delta}\sqrt{\frac{ \log{q}}{n}}$. By Lemma~\ref{lem:control_Wn}, we have  $\| \bar{\m{g}}^n \|_\infty \leq \frac{\rho_{1n} \delta}{4 (2-\delta)} \leq \frac{\rho_{1n}}{4}$ with probability at least $1-4\exp(-C\rho_{1n}^2n)$.
Choosing $n \geq \frac{100^2\tau_5^2(2-\delta)^2}{\tau_4^4\delta^2}q^2\log p$, we have $\rho_{1n} q \leq  \frac{\tau_4^2}{100 \tau_5}\frac{\delta}{2 - \delta}$, thus the conditions of Lemma \ref{lem:control_Rn} hold.

Now, by rewriting \eqref{main_structure} and utilizing the fact that $\hat{\g{\theta}}_{\mc{B}^c} = \bar{\g{\theta}}_{\mc{B}^c} = 0$, we have
\begin{eqnarray}
\label{structure_I1}
 \bar{\m{H}}^n_{\mc{B}^c\mc{B}}(\hat{\g{\theta}}_{\mc{B}}^o - \bar{\g{\theta}}_{\mc{B}}^o) & = & - \bar{\m{g}}^n_{\mc{B}^c} - \rho_{1n}\hat{\m{t}}_{\mc{B}^c} + \m{L}^n_{\mc{B}^c} -  \rho_{2n}\m{A}^n_{\mc{B}^c} , \\
\label{structure_I2}
 \bar{\m{H}}^n_{\mc{B}\mc{B}}(\hat{\g{\theta}}_{\mc{B}}^o - \bar{\g{\theta}}_{\mc{B}}^o) & = & - \bar{\m{g}}^n_{\mc{B}} - \rho_{1n}\hat{\m{t}}_{\mc{B}} + \m{L}^n_{\mc{B}} - \rho_{2n}\m{A}^n_{\mc{B}} . 
\end{eqnarray}
Since $\bar{\m{H}}^n_{\mc{BB}}$ is invertible by assumption, combining \eqref{structure_I1} and \eqref{structure_I2} gives 
\begin{align}\label{eqn:t_equation}
\nonumber
&\bar{\m{H}}^n_{\mc{B}^c\mc{B}}(\bar{\m{H}}^n_{\mc{B}\mc{B}})^{-1}(-  \bar{\m{g}}^n_{\mc{B}} - \rho_{1n}\hat{\m{t}}_{\mc{B}} + \m{L}^n_{\mc{B}} - \rho_{2n}\m{A}^n_{\mc{B}}) \\
&= -\bar{\m{g}}^n_{\mc{B}^c}-\rho_{1n}\hat{\m{t}}_{\mc{B}^c} + \m{L}^n_{\mc{B}^c} -\rho_{2n}\m{A}^n_{\mc{B}^c}.
\end{align} 
Now, using results from Lemmas \ref{lem:control_Wn} and \ref{lem:control_Rn}, we obtain
\begin{eqnarray*}
\rho_{1n}  \|\hat{\m{t}}_{\mc{B}^c}\|_\infty   &=& \big\| \bar{\m{H}}^n_{\mc{B}^c\mc{B}}(\bar{\m{H}}^n_{\mc{B}\mc{B}})^{-1}(- \bar{\m{g}}^n_{\mc{B}} - \rho_{1n}\hat{\m{t}}_{\mc{B}}  \\
&+& \m{L}^n_{\mc{B}}- \rho_{2n}\m{A}^n_{\mc{B}}) - \bar{\m{g}}^n_{\mc{B}^c} - \m{L}^n_{\mc{B}^c} + \rho_{2n}\m{A}^n_{\mc{B}^c} \big\|_\infty \\
&\leq& \|\bar{\m{H}}^n_{\mc{B}^c\mc{B}}(\bar{\m{H}}^n_{\mc{B}\mc{B}})^{-1}\|_\infty(\|\bar{\m{g}}^n_{\mc{B}}\|_\infty + \rho_{1n} + \|\m{L}^n\|_\infty + \rho_{2n}\|\m{A}^n_{\mc{B}^c}\|_{\infty} )\\
&+&   \|\bar{\m{g}}^n_{\mc{B}}\|_\infty + \|\m{L}^n\|_\infty + \rho_{2n}\|\m{A}^n_{\mc{B}^c}\|_{\infty}\\
&\leq&  \rho_{1n}  (1-\frac{\delta}{2}) + (2-\delta) \|\bar{\g{\theta}}^o\|_{\infty} \tau_3 \rho_{2n}\\ 
&\leq&  \rho_{1n}  (1-\frac{\delta}{2}) + \frac{\delta}{4}\rho_{1n}.
\end{eqnarray*}
The result  follows by using  Lemma \ref{lem:disc:ising},  and our assumption that  $\min_{(i,j) \in \mc{B}} \bar{\theta}_{ij}^o  \geq 2 \check{C}(\bar{\g{\theta}}^o) \rho_{1n}\sqrt{q}$ where 
 $ \check{C}(\bar{\g{\theta}}^o)$ is defined in  Lemma \ref{lem:disc:ising}.
\end{proof}

\subsubsection{Proof of Theorem \ref{thm:bin}}
\begin{proof}
With Lemmas  \ref{prop:condition_consistency} and \ref{prop:sample_version}, the proof of Theorem \ref{thm:bin} is straightforward. Given that \ref{assu:eigp:bin:pop} and \ref{assu:incp:bin:pop} are satisfied by $\bar{\m{H}}$ and $\m{T}$ and that $\rho_{1n} =O(\sqrt{\log p/n})$ and $q \sqrt{(\log p)/n}=o(1)$ hold. Thus, the conditions of Lemma \ref{prop:sample_version} hold, and therefore the results in Theorem \ref{thm:bin} hold.
\end{proof}

\subsection{Proof of Theorem~\ref{thm:consist:commu}}\label{sec:app:comm}
\begin{proof} 
We want to bound $\min_{\m{O}\in\mb{R}^{K\times K}}\|\m{Z} \wh{\m{V}}-\m{Z}\bar{\m{V}}\m{O}\|_{F}$,  where $ \m{O}^\top \m{O}=\m{O}\m{O}^\top=\m{I}_K$. For any $\m{O}\in\mb{R}^{K\times K}$ with  $\m{O}^\top\m{O}=\m{O}\m{O}^\top=\m{I}_K$, since $\m{Z}^\top \m{Z}=\m{I}_{(p-H+1)}$,  we have
\begin{align*}
\|\m{Z} \wh{\m{V}}-\m{Z}\bar{\m{V}}\m{O}\|^2_{F}=\|\m{Z}(\wh{\m{V}}-\bar{\m{V}}\m{O})\|^2_{F}=  \|\wh{\m{V}}-\bar{\m{V}}\m{O}\|^2_{F}.
\end{align*}
Hence, 
\begin{align}\label{absch_popoloplo}
\nonumber 
&\min_{\m{O}\in\mb{R}^{K\times K}: \m{O}^\top \m{O}=\m{O}\m{O}^\top=\m{I}_K} \|\m{Z} \wh{\m{V}}-\m{Z}\bar{\m{V}}\m{O}\|^2_{F}\\
& \qquad \qquad = \min_{\m{O}\in\mb{R}^{K\times K}: \m{O}^\top \m{O}=\m{O}\m{O}^\top=\m{I}_K} \|\wh{\m{V}}-\bar{\m{V}} \m{O}\|_{F}.
\end{align}

We proceed similarly to \cite{lei2015consistency}. By Davis-Kahan’s Theorem \citep[Theorem~1]{yu2015useful}
\begin{equation*}
\|\wh{\m{V}}-\bar{\m{V}} \m{O}\|_{F} \leq  \frac{4 \sqrt{K} \min\left( \sqrt{s-r+1} \|\wh{\m{V}}\wh{\m{V}}^\top -  \bar{\m{V}}^\top \bar{\m{V}}^\top \|,   \|\wh{\m{V}}\wh{\m{V}}^\top -  \bar{\m{V}} \bar{\m{V}}^\top \|_{\textnormal{F}} \right) }{\min \left(\Lambda_{r-1}-\Lambda_{r}, \Lambda_{s} -\Lambda_{s+1} \right)},
\end{equation*}
where $s$ and $r$ denote the positions of the ordered (from large to small) eigenvalues of the matrix $\bar{\m{V}} \bar{\m{V}}^\top$.  Using Theorem~\ref{thm:one:fglasso}, we have that  
\begin{align*}
 \|\wh{\m{V}}\wh{\m{V}}^\top -  \bar{\m{V}}^\top \bar{\m{V}}^\top \|_2 &\leq  \|\wh{\m{V}}\wh{\m{V}}^\top -  \bar{\m{V}} \bar{\m{V}}^\top \|_{\textnormal{F}} =O \left(n^{-1}\Psi(p,H,K)\right). 
\end{align*}
This implies that  
\begin{equation*}
\|\wh{\m{V}}-\bar{\m{V}} \m{O}\|_{F}  \leq \kappa \Psi(p,H,K) \sqrt{\frac{K}{n}}
\end{equation*}
for some $\kappa>0$.

The rest of the proof follows as in the proof of \citep[Theorem 1]{lei2015consistency}. More specifically,  it follows from \citep[Lemma 5.3 ]{lei2015consistency} that 
\begin{align}\label{eqn:kappa}
\nonumber
\sum_{k=1}^K \frac{|\mc{S}_k|}{|\mc{C}_k|}  &\leq 4 (4+2\xi )   \|\wh{\m{V}}\wh{\m{V}}^\top -  \bar{\m{V}}^\top \bar{\m{V}}^\top \|\\
\nonumber
 & \leq 4 (4+2\xi ) \kappa \Psi(p,H,K) \sqrt{\frac{K}{n}}\\
 & = \frac{2+\xi}{\pi} \Psi(p,H,K) \sqrt{\frac{K}{n}}, 
\end{align}
for some  $\pi>0$.
\end{proof}

\subsection{ Updating Parameters and Convergence of Algorithm~\ref{Alg:general}}\label{sec:der:alg}

\subsubsection{Proof of Theorem~\ref{thm:bcd}}\label{sec:proof:thm:bcd}
\begin{proof}
The convergence of Algorithm~\ref{Alg:general} follows from \cite{wang2019global} which, using a generalized version of the ADMM algorithm, propose optimizing a general constrained nonconvex optimization problem of the form $f(x) + g(y)$ subject to $x=y$. 
More precisely, for sufficiently large $\gamma$ (the lower bound is given in \citep[Lemma 9]{wang2019global}), and starting from any  $(\g{\Theta}^{(0)}, \m{Q}^{(0)},\g{\Omega}^{(0)},\m{W}^{(0)})$, Algorithm~\ref{Alg:general} generates a sequence that is bounded, has at least one limit point, and that each limit point $(\g{\Theta}^{(*)},\m{Q}^{(*)},\g{\Omega}^{(*)},\m{W}^{(*)})$ is a stationary point of  \eqref{eqn:lagfun}.

The global convergence of Algorithm~\ref{Alg:general} uses the Kurdyka--{\L}ojasiewicz (KL) property of $\mc{L}_\gamma$. Indeed, the KL property has been shown to hold for a large class of functions including subanalytic and semi-algebraic functions such as indicator functions of semi-algebraic sets, vector (semi)-norms $\|\cdot\|_{p}$ with $p \geq 0 $ be any rational number, and matrix (semi)-norms (e.g., operator, trace, and Frobenious norm). Since the loss function $L$ is convex and other functions in \eqref{eqn:lagfun} are either subanalytic or semi-algebraic, the augmented Lagrangian $\mc{L}_\gamma$ satisfies the KL property. The remainder of proof is similar to \cite[Theorem 1]{wang2019global}.
\end{proof}

\subsubsection{Proof of Lemma \ref{lem:fcon:update}}

\begin{proof}
The proof uses the idea of \citep[Lemma~4]{khare2015convex}.  Note that for $1 \leq i \leq p$, 
\begin{eqnarray}\label{eq3}
\nonumber
\Upsilon_{2,\gamma} ( \g{\Theta}) &=& -n \log \theta_{ii} + \frac{n}{2}\big(\theta_{ii}^2 s_{ii} + 2 
\theta_{ii} \sum_{j \neq i} \theta_{ij} s_{ij}\big) \\
\nonumber
 &+&\frac{\gamma}{2}({\theta}_{ii}-{\omega}_{ii}+{w}_{ii})^2+ \text{ terms independent of } \theta_{ii},
\end{eqnarray}
where $s_{ij}=\m{y}_i^\top \m{y}_j/n$. Hence,
\begin{eqnarray*}
  \frac{\partial}{\partial \theta_{ii}}  \Upsilon_{2,\gamma} ( \g{\Theta})=0 
  &\Leftrightarrow & -\frac{1}{\theta_{ii}} + \theta_{ii} (s_{ii} + n \gamma)\\
  &+& \sum_{j \neq i} \theta_{ij} s_{ij} + n\gamma(\omega_{ii}- w_{ii})=0. \\
\end{eqnarray*}  
This implies that
\begin{eqnarray*}
 0&=&\theta_{ii}^2 \underbrace{(s_{ii}+ n\gamma )}_{=:a_{i}} 
 \\
 &+&\theta_{ii} \underbrace{\big(\sum_{j \neq i} \theta_{ij} s_{ij}+ n\gamma ( w_{ii} -\omega_{ii})\big)}_{=:b_{i}} -1,
\end{eqnarray*}
which gives \eqref{eq:fggl:diag}. 
Also, for $1 \leq i < j \leq p$, we have
\begin{eqnarray}\label{eq4}
\nonumber
\Upsilon_{2,\gamma} ( \g{\Theta}) &=& \frac{n}{2} (s_{ii} + s_{jj}) \theta_{ij}^2 + n \big( 
\sum_{j' \neq j} \theta_{ij'} s_{jj'} + \sum_{i' \neq i} \theta_{i'j} s_{ii'} \big) \theta_{ij} \\
&+& \frac{\gamma}{2}({\theta}_{ij}-{\omega}_{ij}+ {w}_{ij})^2+ \mbox{ terms independent of } \theta_{ij}.
\end{eqnarray}
Thus, 
\begin{align}
\nonumber
0 &= \underbrace{ \big(s_{ii} + s_{jj} +  n\gamma\big)}_{=:a_{ij}} \theta_{ij}  + \\
\nonumber 
& +  \underbrace{\big( \sum_{j' \neq j} \theta_{ij'} s_{jj'} + \sum_{i' \neq i} \theta_{i'j} s_{ii'} \big) + n\gamma(\omega_{ij}-w_{ij})}_{=:b_{ij}},
\end{align}
which implies \eqref{eq:fggl:offdiag}.

The proof for updating $\g{\Omega}$ follows similarly. 
\end{proof}

%% file: fairGM-main.bbl
\begin{thebibliography}{100}

\bibitem{Abbe14}
Emmanuel Abbe, Afonso~S Bandeira, and Georgina Hall.
\newblock Exact recovery in the {Stochastic Block Model}.
\newblock {\em IEEE Transactions on information theory}, 62(1):471--487, 2015.

\bibitem{agarwal2021towards}
Chirag Agarwal, Himabindu Lakkaraju, and Marinka Zitnik.
\newblock Towards a unified framework for fair and stable graph representation
  learning.
\newblock In {\em Conference on Uncertainty in Artificial Intelligence}, pages
  2114--2124. PMLR, 2021.

\bibitem{agarwal2011modeling}
Deepak Agarwal, Liang Zhang, and Rahul Mazumder.
\newblock Modeling item-item similarities for personalized recommendations on
  yahoo! front page.
\newblock {\em The Annals of applied statistics}, pages 1839--1875, 2011.

\bibitem{ABBK}
N.~Agarwal, A.~S. Bandeira, K.~Koiliaris, and A.~Kolla.
\newblock Multisection in the {Stochastic Block Model} using semidefinite
  programming.
\newblock arXiv 1507.02323, 2015.

\bibitem{allen2012log}
Genevera~I Allen and Zhandong Liu.
\newblock A log-linear graphical model for inferring genetic networks from
  high-throughput sequencing data.
\newblock In {\em 2012 IEEE International Conference on Bioinformatics and
  Biomedicine}, pages 1--6. IEEE, 2012.

\bibitem{aloise2009np}
Daniel Aloise, Amit Deshpande, Pierre Hansen, and Preyas Popat.
\newblock {NP}-hardness of {Euclidean} sum-of-squares clustering.
\newblock {\em Machine learning}, 75(2):245--248, 2009.

\bibitem{Ames2013}
Brendan~PW Ames.
\newblock Guaranteed clustering and biclustering via semidefinite programming.
\newblock {\em Mathematical Programming}, 147(1):429--465, 2014.

\bibitem{amini2018semidefinite}
Arash~A Amini, Elizaveta Levina, et~al.
\newblock On semidefinite relaxations for the block model.
\newblock {\em The Annals of Statistics}, 46(1):149--179, 2018.

\bibitem{aslan2021demographic}
Sina Aslan and Lada~A Adamic.
\newblock Demographic parity loss: Measuring group fairness in ranked lists.
\newblock In {\em Proceedings of the 14th ACM International Conference on Web
  Search and Data Mining}, pages 660--668, 2021.

\bibitem{Bandeira15}
Afonso~S Bandeira.
\newblock Random {Laplacian} matrices and convex relaxations.
\newblock {\em Foundations of Computational Mathematics}, 18:345--379, 2018.

\bibitem{banerjee2008model}
Onureena Banerjee, Laurent El~Ghaoui, and Alexandre d'Aspremont.
\newblock Model selection through sparse maximum likelihood estimation for
  multivariate {Gaussian} or binary data.
\newblock {\em The Journal of Machine Learning Research}, 9:485--516, 2008.

\bibitem{barocas-hardt-narayanan}
Solon Barocas, Moritz Hardt, and Arvind Narayanan.
\newblock {\em Fairness and Machine Learning}.
\newblock fairmlbook.org, 2019.
\newblock \url{http://www.fairmlbook.org}.

\bibitem{barzilai1988two}
Jonathan Barzilai and Jonathan~M Borwein.
\newblock Two-point step size gradient methods.
\newblock {\em IMA journal of numerical analysis}, 8(1):141--148, 1988.

\bibitem{bert16}
Quentin Berthet, Philippe Rigollet, and Piyush Srivastava.
\newblock Exact recovery in the {I}sing block model.
\newblock {\em arXiv:1612.03880}, 2016.

\bibitem{bianchi2020spectral}
Filippo~Maria Bianchi, Daniele Grattarola, and Cesare Alippi.
\newblock Spectral clustering with graph neural networks for graph pooling.
\newblock In {\em International Conference on Machine Learning}, pages
  874--883. PMLR, 2020.

\bibitem{bolukbasi2016man}
Tolga Bolukbasi, Kai-Wei Chang, James~Y Zou, Venkatesh Saligrama, and Adam~T
  Kalai.
\newblock Man is to computer programmer as woman is to homemaker? debiasing
  word embeddings.
\newblock {\em Advances in neural information processing systems}, 29, 2016.

\bibitem{Boyd11}
Stephen Boyd, Neal Parikh, Eric Chu, Borja Peleato, and Jonathan Eckstein.
\newblock Distributed optimization and statistical learning via the alternating
  direction method of multipliers.
\newblock {\em Foundations and Trends{\textregistered} in Machine Learning},
  3(1):1--122, 2011.

\bibitem{brusilovsky2016educational}
Peter Brusilovsky, Hendrik Drachsler, Nikos Manouselis, and Mohamed~Amine
  Chatti.
\newblock Educational recommender systems and technologies: Practices and
  challenges.
\newblock In {\em International Workshop on Personalization Approaches in
  Learning Environments (PALE)}, 2016.

\bibitem{burke2011recommender}
Robin Burke, Alexander Felfernig, and Mehmet~H G{\"o}ker.
\newblock Recommender systems: An overview.
\newblock {\em Ai Magazine}, 32(3):13--18, 2011.

\bibitem{cai2015robust}
T~Tony Cai, Xiaodong Li, et~al.
\newblock Robust and computationally feasible community detection in the
  presence of arbitrary outlier nodes.
\newblock {\em Annals of Statistics}, 43(3):1027--1059, 2015.

\bibitem{Cai11}
Tony Cai and Weidong Liu.
\newblock Adaptive thresholding for sparse covariance matrix estimation.
\newblock {\em Journal of the American Statistical Association},
  106(494):672--684, 2011.

\bibitem{caliskan2017semantics}
Aylin Caliskan, Joanna~J Bryson, and Arvind Narayanan.
\newblock Semantics derived automatically from language corpora contain
  human-like biases.
\newblock {\em Science}, 356(6334):183--186, 2017.

\bibitem{cardoso2020algorithms}
Jos{\'e} Vin{\'\i}cius de~Miranda Cardoso, Jiaxi Ying, and Daniel~Perez
  Palomar.
\newblock Algorithms for learning graphs in financial markets.
\newblock {\em arXiv preprint arXiv:2012.15410}, 2020.

\bibitem{caton2020fairness}
Simon Caton and Christian Haas.
\newblock Fairness in machine learning: A survey.
\newblock {\em arXiv preprint arXiv:2010.04053}, 2020.

\bibitem{celis2017ranking}
L~Elisa Celis, Damian Straszak, and Nisheeth~K Vishnoi.
\newblock Ranking with fairness constraints.
\newblock {\em arXiv preprint arXiv:1704.06840}, 2017.

\bibitem{chen2020bias}
Jiawei Chen, Hande Dong, Xiang Wang, Fuli Feng, Meng Wang, and Xiangnan He.
\newblock Bias and debias in recommender system: A survey and future
  directions.
\newblock {\em arXiv preprint arXiv:2010.03240}, 2020.

\bibitem{CSX2014}
Y~Chen, S~Sanghavi, and H~Xu.
\newblock Improved graph clustering.
\newblock {\em IEEE Transactions on Information Theory}, 60:6440--6455, 2014.

\bibitem{ChenXu14}
Yuchen Chen and Jiaming Xu.
\newblock Statistical-computational tradeoffs in planted problems and submatrix
  localization with a growing number of clusters and submatrices.
\newblock In {\em Proceedings of the 31st International Conference on Machine
  Learning}, pages 849--857, 2014.

\bibitem{chen2014clustering}
Yudong Chen, Ali Jalali, Sujay Sanghavi, and Huan Xu.
\newblock Clustering partially observed graphs via convex optimization.
\newblock {\em The Journal of Machine Learning Research}, 15(1):2213--2238,
  2014.

\bibitem{CLX2015}
Yudong Chen, Xiaodong Li, and Jiaming Xu.
\newblock Convexified modularity maximization for degree-corrected {Stochastic
  Block Model}s.
\newblock {\em The Annals of Statistics}, 46:1573--1602, 2018.

\bibitem{chien2021adaptive}
Eli Chien, Jianhao Peng, Pan Li, and Olgica Milenkovic.
\newblock Adaptive universal generalized pagerank graph neural network.
\newblock In {\em International Conference on Learning Representations}, 2021.

\bibitem{chierichetti2017fair}
Flavio Chierichetti, Ravi Kumar, Silvio Lattanzi, and Sergei Vassilvitskii.
\newblock Fair clustering through fairlets.
\newblock In {\em Advances in Neural Information Processing Systems}, pages
  5029--5037, 2017.

\bibitem{chouldechova2018frontiers}
Alexandra Chouldechova and Aaron Roth.
\newblock The frontiers of fairness in machine learning.
\newblock {\em arXiv preprint arXiv:1810.08810}, 2018.

\bibitem{cinelli2021echo}
Matteo Cinelli, Gianmarco~Delfino Morales, Alessandro Galeazzi, Walter
  Quattrociocchi, and Michele Starnini.
\newblock Echo chambers on social media: A comparative analysis.
\newblock {\em Social Media+ Society}, 7(2):20563051211008834, 2021.

\bibitem{Danaher13}
Patrick Danaher, Pei Wang, and Daniela~M Witten.
\newblock The joint graphical lasso for inverse covariance estimation across
  multiple classes.
\newblock {\em Journal of the Royal Statistical Society: Series B (Statistical
  Methodology)}, 76(2):373--397, 2014.

\bibitem{das2014clustering}
Joydeep Das, Partha Mukherjee, Subhashis Majumder, and Prosenjit Gupta.
\newblock Clustering-based recommender system using principles of voting
  theory.
\newblock In {\em 2014 International conference on contemporary computing and
  informatics (IC3I)}, pages 230--235. IEEE, 2014.

\bibitem{de2019small}
Simon De~Deyne, Danielle~J Navarro, Amy Perfors, Marc Brysbaert, and Gert
  Storms.
\newblock The “small world of words” english word association norms for
  over 12,000 cue words.
\newblock {\em Behavior research methods}, 51(3):987--1006, 2019.

\bibitem{dong2022fairness}
Yushun Dong, Jian Ma, Song Wang, Chen Chen, and Jundong Li.
\newblock Fairness in graph neural networks: A survey.
\newblock {\em ACM Computing Surveys}, 55(7):1--37, 2022.

\bibitem{donini2018empirical}
Michele Donini, Luca Oneto, Shai Ben-David, John Shawe-Taylor, and Massimiliano
  Pontil.
\newblock Empirical risk minimization under fairness constraints.
\newblock {\em arXiv preprint arXiv:1802.08626}, 2018.

\bibitem{dwork2012fairness}
Cynthia Dwork, Moritz Hardt, Toniann Pitassi, Omer Reingold, and Richard Zemel.
\newblock Fairness through awareness.
\newblock In {\em Proceedings of the 3rd innovations in theoretical computer
  science conference}, pages 214--226, 2012.

\bibitem{eisenach2020high}
Carson Eisenach, Florentina Bunea, Yang Ning, and Claudiu Dinicu.
\newblock High-dimensional inference for cluster-based graphical models.
\newblock {\em Journal of machine learning research}, 21(53), 2020.

\bibitem{epps2020artist}
Avriel Epps-Darling, Romain~Takeo Bouyer, and Henriette Cramer.
\newblock Artist gender representation in music streaming.
\newblock In {\em Proceedings of the 21st International Society for Music
  Information Retrieval Conference (Montr{\'e}al, Canada)(ISMIR 2020). ISMIR},
  pages 248--254, 2020.

\bibitem{fei2018exponential}
Yingjie Fei and Yudong Chen.
\newblock Exponential error rates of sdp for block models: Beyond
  grothendieck’s inequality.
\newblock {\em IEEE Transactions on Information Theory}, 65(1):551--571, 2018.

\bibitem{Friedman07}
Jerome Friedman, Trevor Hastie, and Robert Tibshirani.
\newblock Sparse inverse covariance estimation with the graphical lasso.
\newblock {\em Biostatistics}, 9(3):432--441, 2008.

\bibitem{friedman2010applications}
Jerome Friedman, Trevor Hastie, and Robert Tibshirani.
\newblock Applications of the lasso and grouped lasso to the estimation of
  sparse graphical models.
\newblock Technical report, 2010.

\bibitem{gan2019bayesian}
Lingrui Gan, Xinming Yang, Naveen~N Nariestty, and Feng Liang.
\newblock {Bayesian} joint estimation of multiple graphical models.
\newblock {\em Advances in Neural Information Processing Systems}, 32, 2019.

\bibitem{gheche2020multilayer}
Mireille~El Gheche and Pascal Frossard.
\newblock Multilayer clustered graph learning.
\newblock {\em arXiv preprint arXiv:2010.15456}, 2020.

\bibitem{glassman2014feature}
Elena~L Glassman, Rishabh Singh, and Robert~C Miller.
\newblock Feature engineering for clustering student solutions.
\newblock In {\em Proceedings of the first ACM conference on Learning@ scale
  conference}, pages 171--172, 2014.

\bibitem{grant2014cvx}
Michael Grant and Stephen Boyd.
\newblock Cvx: Matlab software for disciplined convex programming, version 2.1,
  2014.

\bibitem{guedon2016community}
Olivier Gu{\'e}don and Roman Vershynin.
\newblock Community detection in sparse networks via grothendieck’s
  inequality.
\newblock {\em Probability Theory and Related Fields}, 165(3-4):1025--1049,
  2016.

\bibitem{Guo15}
Jian Guo, Jie Cheng, Elizaveta Levina, George Michailidis, and Ji~Zhu.
\newblock Estimating heterogeneous graphical models for discrete data with an
  application to roll call voting.
\newblock {\em The Annals of Applied Statistics}, 9(2):821, 2015.

\bibitem{guo2010joint}
Jian Guo, Elizaveta Levina, George Michailidis, and Ji~Zhu.
\newblock Joint structure estimation for categorical {Markov} networks.
\newblock {\em Unpublished manuscript}, 3(5.2):6, 2010.

\bibitem{Guo11}
Jian Guo, Elizaveta Levina, George Michailidis, and Ji~Zhu.
\newblock Asymptotic properties of the joint neighborhood selection method for
  estimating categorical {Markov} networks.
\newblock {\em arXiv preprint math.PR/0000000}, 2011.

\bibitem{guo2011joint}
Jian Guo, Elizaveta Levina, George Michailidis, and Ji~Zhu.
\newblock Joint estimation of multiple graphical models.
\newblock {\em Biometrika}, 98(1):1--15, 2011.

\bibitem{hao2018simultaneous}
Botao Hao, Will~Wei Sun, Yufeng Liu, and Guang Cheng.
\newblock Simultaneous clustering and estimation of heterogeneous graphical
  models.
\newblock {\em Journal of Machine Learning Research}, 2018.

\bibitem{harakawa2022trend}
Ryosuke Harakawa, Tsutomu Ito, and Masahiro Iwahashi.
\newblock Trend clustering from covid-19 tweets using graphical lasso-guided
  iterative principal component analysis.
\newblock {\em Scientific Reports}, 12(1):1--13, 2022.

\bibitem{hardt2016equality}
Moritz Hardt, Eric Price, and Nathan Srebro.
\newblock Equality of opportunity in supervised learning.
\newblock {\em arXiv preprint arXiv:1610.02413}, 2016.

\bibitem{hassner1981use}
Martin Hassner and Jack Sklansky.
\newblock The use of {Markov} random fields as models of texture.
\newblock In {\em Image Modeling}, pages 185--198. Elsevier, 1981.

\bibitem{holland1983stochastic}
Paul~W Holland, Kathryn~Blackmond Laskey, and Samuel Leinhardt.
\newblock Stochastic blockmodels: First steps.
\newblock {\em Social networks}, 5(2):109--137, 1983.

\bibitem{holstein2019fairness}
Kenneth Holstein, Jennifer Wortman~Vaughan, Hal Daumé~III, Miro Dudik, and
  Hanna Wallach.
\newblock Improving fairness in machine learning systems: What do industry
  practitioners need?
\newblock In {\em Proceedings of the 2019 CHI Conference on Human Factors in
  Computing Systems}, pages 1--16, 2019.

\bibitem{hosseini2016learning}
Mohammad~Javad Hosseini and Su-In Lee.
\newblock Learning sparse {Gaussian} graphical models with overlapping blocks.
\newblock In {\em Advances in Neural Information Processing Systems},
  volume~30, pages 3801--3809, 2016.

\bibitem{Ising25}
Ernst Ising.
\newblock Beitrag zur theorie des ferromagnetismus.
\newblock {\em Zeitschrift f{\"u}r Physik A Hadrons and Nuclei},
  31(1):253--258, 1925.

\bibitem{kamishima2012enhancement}
Toshihiro Kamishima, Shotaro Akaho, Hideki Asoh, and Jun Sakuma.
\newblock Enhancement of the neutrality in recommendation.
\newblock Citeseer, 2012.

\bibitem{ElKaroui08}
Noureddine~El Karoui.
\newblock Operator norm consistent estimation of large-dimensional sparse
  covariance matrices.
\newblock {\em The Annals of Statistics}, pages 2717--2756, 2008.

\bibitem{khare2015convex}
Kshitij Khare, Sang-Yun Oh, and Bala Rajaratnam.
\newblock A convex pseudolikelihood framework for high dimensional partial
  correlation estimation with convergence guarantees.
\newblock {\em Journal of the Royal Statistical Society: Series B: Statistical
  Methodology}, pages 803--825, 2015.

\bibitem{kleindessner2019guarantees}
Matth{\"a}us Kleindessner, Samira Samadi, Pranjal Awasthi, and Jamie
  Morgenstern.
\newblock Guarantees for spectral clustering with fairness constraints.
\newblock In {\em International Conference on Machine Learning}, pages
  3458--3467. PMLR, 2019.

\bibitem{koren2009collaborative}
Yehuda Koren.
\newblock Collaborative filtering with temporal dynamics.
\newblock In {\em Proceedings of the 15th ACM SIGKDD international conference
  on Knowledge discovery and data mining}, pages 447--456, 2009.

\bibitem{kouki2015hyper}
Pigi Kouki, Shobeir Fakhraei, James Foulds, Magdalini Eirinaki, and Lise
  Getoor.
\newblock Hyper: A flexible and extensible probabilistic framework for hybrid
  recommender systems.
\newblock In {\em Proceedings of the 9th ACM Conference on Recommender
  Systems}, pages 99--106, 2015.

\bibitem{kumar2004simple}
Amit Kumar, Yogish Sabharwal, and Sandeep Sen.
\newblock A simple linear time $(1+/spl epsiv/)$-approximation algorithm for
  k-means clustering in any dimensions.
\newblock In {\em 45th Annual IEEE Symposium on Foundations of Computer
  Science}, pages 454--462. IEEE, 2004.

\bibitem{kumar2019structured}
Sandeep Kumar, Jiaxi Ying, Jos{\'e}~Vin{\'\i}cius Cardoso, de Miranda~Cardoso,
  and Daniel Palomar.
\newblock Structured graph learning via laplacian spectral constraints.
\newblock In {\em Advances in Neural Information Processing Systems}, pages
  11651--11663, 2019.

\bibitem{kumar2020unified}
Sandeep Kumar, Jiaxi Ying, Jos{\'e}~Vin{\'\i}cius de~Miranda~Cardoso, and
  Daniel~P Palomar.
\newblock A unified framework for structured graph learning via spectral
  constraints.
\newblock {\em Journal of Machine Learning Research}, 21(22):1--60, 2020.

\bibitem{laferte2000discrete}
J-M Lafert{\'e}, Patrick P{\'e}rez, and Fabrice Heitz.
\newblock Discrete {Markov} image modeling and inference on the quadtree.
\newblock {\em IEEE Transactions on image processing}, 9(3):390--404, 2000.

\bibitem{lafit2019partial}
Ginette Lafit, Francis Tuerlinckx, Inez Myin-Germeys, and Eva Ceulemans.
\newblock A partial correlation screening approach for controlling the false
  positive rate in sparse gaussian graphical models.
\newblock {\em Scientific reports}, 9(1):1--24, 2019.

\bibitem{lee06}
Su-In Lee, Varun Ganapathi, and Daphne Koller.
\newblock Efficient structure learning of {M}arkov networks using $\ell_1
  $-regularization.
\newblock In {\em Advances in neural Information processing systems}, pages
  817--824, 2007.

\bibitem{lee2015joint}
Wonyul Lee and Yufeng Liu.
\newblock Joint estimation of multiple precision matrices with common
  structures.
\newblock {\em The Journal of Machine Learning Research}, 16(1):1035--1062,
  2015.

\bibitem{lei2015consistency}
Jing Lei, Alessandro Rinaldo, et~al.
\newblock Consistency of spectral clustering in {Stochastic Block Model}s.
\newblock {\em Annals of Statistics}, 43(1):215--237, 2015.

\bibitem{li2021convex}
Xiaodong Li, Yudong Chen, and Jiaming Xu.
\newblock Convex relaxation methods for community detection.
\newblock {\em Statistical Science}, 36(1):2--15, 2021.

\bibitem{Liljeros01}
Fredrik Liljeros, Christofer~R Edling, Luis A~Nunes Amaral, H~Eugene Stanley,
  and Yvonne {\AA}berg.
\newblock The web of human sexual contacts.
\newblock {\em Nature}, 411(6840):907--908, 2001.

\bibitem{ma2016joint}
Jing Ma and George Michailidis.
\newblock Joint structural estimation of multiple graphical models.
\newblock {\em The Journal of Machine Learning Research}, 17(1):5777--5824,
  2016.

\bibitem{manning1999foundations}
Christopher Manning and Hinrich Schutze.
\newblock {\em Foundations of statistical natural language processing}.
\newblock MIT press, 1999.

\bibitem{marlin2009sparse}
Benjamin~M Marlin and Kevin~P Murphy.
\newblock Sparse {Gaussian} graphical models with unknown block structure.
\newblock In {\em Proceedings of the 26th Annual International Conference on
  Machine Learning}, pages 705--712, 2009.

\bibitem{mastrandrea2015contact}
Rossana Mastrandrea, Julie Fournet, and Alain Barrat.
\newblock Contact patterns in a high school: a comparison between data
  collected using wearable sensors, contact diaries and friendship surveys.
\newblock {\em PloS one}, 10(9):e0136497, 2015.

\bibitem{mathieu2010correlation}
Claire Mathieu and Warren Schudy.
\newblock Correlation clustering with noisy input.
\newblock In {\em Proceedings of the twenty-first annual ACM-SIAM symposium on
  Discrete Algorithms}, pages 712--728. SIAM, 2010.

\bibitem{mazumder2011flexible}
Rahul Mazumder and Deepak~K Agarwal.
\newblock A flexible, scalable and efficient algorithmic framework for primal
  graphical lasso.
\newblock {\em arXiv preprint arXiv:1110.5508}, 2011.

\bibitem{mei2017solving}
Song Mei, Theodor Misiakiewicz, Andrea Montanari, and Roberto~Imbuzeiro
  Oliveira.
\newblock Solving sdps for synchronization and maxcut problems via the
  grothendieck inequality.
\newblock In {\em Conference on learning theory}, pages 1476--1515. PMLR, 2017.

\bibitem{meinshausen2006high}
Nicolai Meinshausen, Peter B{\"u}hlmann, et~al.
\newblock High-dimensional graphs and variable selection with the lasso.
\newblock {\em Annals of statistics}, 34(3):1436--1462, 2006.

\bibitem{mnih2008probabilistic}
Andriy Mnih and Russ~R Salakhutdinov.
\newblock Probabilistic matrix factorization.
\newblock In {\em Advances in neural information processing systems}, pages
  1257--1264, 2008.

\bibitem{montanari2016semidefinite}
Andrea Montanari and Subhabrata Sen.
\newblock Semidefinite programs on sparse random graphs and their application
  to community detection.
\newblock In {\em Proceedings of the forty-eighth annual ACM symposium on
  Theory of Computing}, pages 814--827, 2016.

\bibitem{oneto2020fairness}
Luca Oneto and Silvia Chiappa.
\newblock Fairness in machine learning.
\newblock In {\em Recent Trends in Learning From Data}, pages 155--196.
  Springer, 2020.

\bibitem{o1999clustering}
Mark O’Connor and Jon Herlocker.
\newblock Clustering items for collaborative filtering.
\newblock In {\em Proceedings of the ACM SIGIR workshop on recommender
  systems}, volume 128. Citeseer, 1999.

\bibitem{parimbelli2018patient}
Enea Parimbelli, Simone Marini, Lucia Sacchi, and Riccardo Bellazzi.
\newblock Patient similarity for precision medicine: A systematic review.
\newblock {\em Journal of biomedical informatics}, 83:87--96, 2018.

\bibitem{peng2009partial}
Jie Peng, Pei Wang, Nengfeng Zhou, and Ji~Zhu.
\newblock Partial correlation estimation by joint sparse regression models.
\newblock {\em Journal of the American Statistical Association},
  104(486):735--746, 2009.

\bibitem{petreski2022word}
Davor Petreski and Ibrahim~C Hashim.
\newblock Word embeddings are biased. but whose bias are they reflecting?
\newblock {\em AI \& SOCIETY}, pages 1--8, 2022.

\bibitem{pfohl2019creating}
Stephen~R Pfohl, Nithya Sambasivan, Bryan Kim, Yun Liu, and Kush~R Varshney.
\newblock Creating fair models of atherosclerotic cardiovascular disease risk.
\newblock In {\em Proceedings of the AAAI Conference on Artificial
  Intelligence}, volume~33, pages 682--689, 2019.

\bibitem{pham2011clustering}
Manh~Cuong Pham, Yiwei Cao, Ralf Klamma, and Matthias Jarke.
\newblock A clustering approach for collaborative filtering recommendation
  using social network analysis.
\newblock {\em J. UCS}, 17(4):583--604, 2011.

\bibitem{pircalabelu2020community}
Eugen Pircalabelu and Gerda Claeskens.
\newblock Community-based group graphical {Lasso}.
\newblock {\em Journal of Machine Learning Research}, 21(64):1--32, 2020.

\bibitem{ravikumar2010high}
Pradeep Ravikumar, Martin~J Wainwright, John~D Lafferty, et~al.
\newblock High-dimensional {Ising} model selection using $\ell_1$-regularized
  logistic regression.
\newblock {\em The Annals of Statistics}, 38(3):1287--1319, 2010.

\bibitem{Robins07}
Garry Robins, Pip Pattison, Yuval Kalish, and Dean Lusher.
\newblock An introduction to exponential random graph $p*$ models for social
  networks.
\newblock {\em Social networks}, 29(2):173--191, 2007.

\bibitem{rocha2008path}
Guilherme~V Rocha, Peng Zhao, and Bin Yu.
\newblock A path following algorithm for sparse pseudo-likelihood inverse
  covariance estimation (splice).
\newblock {\em arXiv preprint arXiv:0807.3734}, 2008.

\bibitem{samadi2018price}
Samira Samadi, Uthaipon Tantipongpipat, Jamie Morgenstern, Mohit Singh, and
  Santosh Vempala.
\newblock The price of fair pca: One extra dimension.
\newblock {\em arXiv preprint arXiv:1811.00103}, 2018.

\bibitem{sarwar2001item}
Badrul Sarwar, George Karypis, Joseph Konstan, and John Riedl.
\newblock Item-based collaborative filtering recommendation algorithms.
\newblock In {\em Proceedings of the 10th international conference on World
  Wide Web}, pages 285--295, 2001.

\bibitem{schafer2007collaborative}
J~Ben Schafer, Dan Frankowski, Jon Herlocker, and Shilad Sen.
\newblock Collaborative filtering recommender systems.
\newblock In {\em The adaptive web}, pages 291--324. Springer, 2007.

\bibitem{shakespeare2020exploring}
Dougal Shakespeare, Lorenzo Porcaro, Emilia G{\'o}mez, and Carlos Castillo.
\newblock Exploring artist gender bias in music recommendation.
\newblock {\em arXiv preprint arXiv:2009.01715}, 2020.

\bibitem{song2011fast}
Qinbao Song, Jingjie Ni, and Guangtao Wang.
\newblock A fast clustering-based feature subset selection algorithm for
  high-dimensional data.
\newblock {\em IEEE transactions on knowledge and data engineering},
  25(1):1--14, 2011.

\bibitem{tan2014learning}
Kean~Ming Tan, Palma London, Karthik Mohan, Su-In Lee, Maryam Fazel, and
  Daniela Witten.
\newblock Learning graphical models with hubs.
\newblock {\em Journal of Machine Learning Research}, 15:3297--3331, 2014.

\bibitem{Tan14}
Kean~Ming Tan, Palma London, Karthik Mohan, Su-In Lee, Maryam Fazel, and
  Daniela~M Witten.
\newblock Learning graphical models with hubs.
\newblock {\em Journal of Machine Learning Research}, 15(1):3297--3331, 2014.

\bibitem{tan2013data}
Pang-Ning Tan, Michael Steinbach, and Vipin Kumar.
\newblock Data mining cluster analysis: basic concepts and algorithms.
\newblock {\em Introduction to data mining}, pages 487--533, 2013.

\bibitem{tantipongpipat2019multi}
Uthaipon Tantipongpipat, Samira Samadi, Mohit Singh, Jamie~H Morgenstern, and
  Santosh~S Vempala.
\newblock Multi-criteria dimensionality reduction with applications to
  fairness.
\newblock {\em Advances in neural information processing systems}, (32), 2019.

\bibitem{tarzanagh2018estimation}
Davoud~Ataee Tarzanagh and George Michailidis.
\newblock Estimation of graphical models through structured norm minimization.
\newblock {\em Journal of Machine Learning Research}, 18(209):1--48, 2018.

\bibitem{tsitsulin2023graph}
Anton Tsitsulin, John Palowitch, Bryan Perozzi, and Emmanuel M{\"u}ller.
\newblock Graph clustering with graph neural networks.
\newblock {\em Journal of Machine Learning Research}, 24(127):1--21, 2023.

\bibitem{ungar1998clustering}
Lyle~H Ungar and Dean~P Foster.
\newblock Clustering methods for collaborative filtering.
\newblock In {\em AAAI workshop on recommendation systems}, volume~1, pages
  114--129. Menlo Park, CA, 1998.

\bibitem{von2007tutorial}
Ulrike Von~Luxburg.
\newblock A tutorial on spectral clustering.
\newblock {\em Statistics and computing}, 17:395--416, 2007.

\bibitem{wang2015collaborative}
Hao Wang, Naiyan Wang, and Dit-Yan Yeung.
\newblock Collaborative deep learning for recommender systems.
\newblock In {\em Proceedings of the 21th ACM SIGKDD international conference
  on knowledge discovery and data mining}, pages 1235--1244, 2015.

\bibitem{wang2009learning}
Pei Wang, Dennis~L Chao, and Li~Hsu.
\newblock Learning networks from high dimensional binary data: An application
  to genomic instability data.
\newblock {\em arXiv preprint arXiv:0908.3882}, 2009.

\bibitem{wang2023survey}
Yifan Wang, Weizhi Ma, Min Zhang, Yiqun Liu, and Shaoping Ma.
\newblock A survey on the fairness of recommender systems.
\newblock {\em ACM Transactions on Information Systems}, 41(3):1--43, 2023.

\bibitem{wang2019global}
Yu~Wang, Wotao Yin, and Jinshan Zeng.
\newblock Global convergence of admm in nonconvex nonsmooth optimization.
\newblock {\em Journal of Scientific Computing}, 78(1):29--63, 2019.

\bibitem{weeks2002social}
Margaret~R Weeks, Scott Clair, Stephen~P Borgatti, Kim Radda, and Jean~J
  Schensul.
\newblock Social networks of drug users in high-risk sites: Finding the
  connections.
\newblock {\em AIDS and Behavior}, 6:193--206, 2002.

\bibitem{Xue12}
Lingzhou Xue, Shiqian Ma, and Hui Zou.
\newblock Positive-definite $\ell_1$-penalized estimation of large covariance
  matrices.
\newblock {\em Journal of the American Statistical Association},
  107(500):1480--1491, 2012.

\bibitem{yang2012graphical}
Eunho Yang, Pradeep Ravikumar, Genevera~I Allen, and Zhandong Liu.
\newblock Graphical models via generalized linear models.
\newblock In {\em NIPS}, volume~25, pages 1367--1375, 2012.

\bibitem{ye2020exact}
Min Ye.
\newblock Exact recovery and sharp thresholds of {Stochastic} {Ising} {Block}
  {Model}.
\newblock {\em arXiv preprint arXiv:2004.05944}, 2020.

\bibitem{yu2015useful}
Yi~Yu, Tengyao Wang, and Richard~J Samworth.
\newblock A useful variant of the {Davis}--{Kahan} theorem for statisticians.
\newblock {\em Biometrika}, 102(2):315--323, 2015.

\bibitem{yuan2006model}
Ming Yuan and Yi~Lin.
\newblock Model selection and estimation in regression with grouped variables.
\newblock {\em Journal of the Royal Statistical Society: Series B (Statistical
  Methodology)}, 68(1):49--67, 2006.

\bibitem{zhao2006model}
Peng Zhao and Bin Yu.
\newblock On model selection consistency of {Lasso}.
\newblock {\em The Journal of Machine Learning Research}, 7:2541--2563, 2006.

\bibitem{zhu2018fairness}
Ziwei Zhu, Xia Hu, and James Caverlee.
\newblock Fairness-aware tensor-based recommendation.
\newblock In {\em Proceedings of the 27th ACM International Conference on
  Information and Knowledge Management}, pages 1153--1162, 2018.

\end{thebibliography}
